\documentclass[english]{article}

\usepackage{geometry}
\usepackage[T1]{fontenc}
\usepackage[utf8]{inputenc}
\usepackage{bm}
\usepackage{graphicx,amsmath}
\usepackage{amssymb}
\usepackage[unicode=true,
 bookmarks=false,
 breaklinks=false,pdfborder={0 0 1},colorlinks=false]
 {hyperref}
\hypersetup{
 colorlinks,citecolor=blue,filecolor=blue,linkcolor=blue,urlcolor=blue}

\makeatletter
\usepackage{amsthm}
\usepackage{cite}  
\usepackage{comment}
\usepackage{natbib}

\usepackage{float}
\usepackage{multirow}
\usepackage{footnote}
\usepackage{dsfont}
\usepackage{color}
\usepackage{booktabs}
\definecolor{yxc}{RGB}{255,0,0}
\definecolor{yjc}{RGB}{125,0,0}
\definecolor{ytw}{RGB}{255,69,0}

\newcommand{\yxc}[1]{\textcolor{yxc}{[YXC: #1]}}

\allowdisplaybreaks

\DeclareMathOperator{\ind}{\mathds{1}}  % Indicator

\newcommand{\defn}{:=}

\newcommand{\linf}[1]{\big\|#1 \big\|_{\infty}}
\newcommand{\one}{\bm 1}

\newcommand{\Ind}{\bm I}

\newcommand{\real}{\mathbb{R}}

\newcommand{\cS}{\mathcal{S}}
\newcommand{\cA}{\mathcal{A}}

\newcommand{\Qhat}{\widehat{\bm Q}}
\newcommand{\Qstar}{ {\bm Q}^\star}

\newcommand{\Vhat}{\widehat{\bm V}}

\newcommand{\tmix}{t_{\mathsf{mix}}}
\newcommand{\tcov}{t_{\mathsf{frame}}}
\newcommand{\tth}{t_{\mathsf{th}}}
\newcommand{\ttheps}{t_{\mathsf{th},\xi}}

\newcommand{\bQ}{\bm Q}
\newcommand{\tepoch}{t_{\mathsf{epoch}}}
\newcommand{\tcover}{t_{\mathsf{frame}}}
\newcommand{\cirS}{\mathcal{S}}
\newcommand{\A}{\mathcal{A}}
\newcommand{\bLam}{\bm \Lambda}
\newcommand{\bDel}{\bm \Delta}

\newcommand{\vrq}{\textsc{Vr-q-run-epoch}}

\usepackage[linesnumbered,ruled,vlined]{algorithm2e}
\usepackage{algorithmic}

\newcommand{\tcovertime}{t_{\mathsf{cover}}}
\newcommand{\tcoverall}{t_{\mathsf{cover,all}}}

\newcommand{\muepo}{\mu_{\mathsf{frame}}}

\newcommand{\Vbar}{\overline{\bm V}}
\newcommand{\Qbar}{\overline{\bm Q}}
\newcommand{\Ptil}{\widetilde{\bm P}}
\newcommand{\Delvr}{\bm{\Theta}}

\newcommand{\mumin}{\mu_{\mathsf{min}}}
\newcommand{\Delhat}{\widehat{\bm \Theta}}

\title{Sample Complexity of Asynchronous Q-Learning: \\ Sharper Analysis and Variance Reduction\footnotetext{This work has been presented in part in Neural Information Processing Systems (NeurIPS) 2020 \citep{li2020sampleNIPS}.}}
 
\author{Gen Li\thanks{Department of Electrical and Computer Engineering, Princeton University, Princeton, NJ 08544, USA.} \\
Princeton   \\
	\and
	Yuting Wei\thanks{Department of Statistics and Data Science, The Wharton School, University of Pennsylvania, Philadelphia, PA 19104, USA.}\\
	UPenn
	\and
	Yuejie Chi\thanks{Department of Electrical and Computer Engineering, Carnegie Mellon University, Pittsburgh, PA 15213, USA.}\\
	CMU\\
	\and
	Yuantao Gu\thanks{Department of Electronic Engineering, Tsinghua University, Beijing 100084, China.}  \\
	Tsinghua   \\
	\and
	Yuxin Chen\footnotemark[1] \\
 Princeton  \\
	}
	
\date{June, 2020;~~ Revised: August 2021}

\makeatother

\begin{document}

\theoremstyle{plain} \newtheorem{lemma}{\textbf{Lemma}}\newtheorem{proposition}{\textbf{Proposition}}\newtheorem{theorem}{\textbf{Theorem}}\newtheorem{assumption}{\textbf{Assumption}}
\newtheorem{corollary}{\textbf{Corollary}}\newtheorem{example}{\textbf{Example}}
\theoremstyle{remark}\newtheorem{remark}{\textbf{Remark}}

\maketitle

\begin{abstract}
	Asynchronous Q-learning aims to learn the optimal action-value function (or Q-function) of a Markov decision process (MDP), based on a single trajectory of Markovian samples induced by a behavior policy.  Focusing on a $\gamma$-discounted MDP with state space $\mathcal{S}$ and action space $\mathcal{A}$, we demonstrate that the $\ell_{\infty}$-based sample complexity of classical asynchronous Q-learning --- namely, the number of samples needed to yield an entrywise $\varepsilon$-accurate estimate of the Q-function --- is at most on the order of
\begin{equation*}
	 \frac{1}{\mumin(1-\gamma)^5\varepsilon^2}+ \frac{\tmix}{\mumin(1-\gamma)}   
\end{equation*}
	up to some logarithmic factor, provided that a proper constant learning rate is adopted. Here, $\tmix$ and $\mumin$ denote respectively the mixing time and the minimum state-action occupancy probability of the sample trajectory. The first term of this bound matches the sample complexity in the synchronous case with independent samples drawn from the stationary distribution of  the trajectory. The second term reflects the cost taken for the empirical distribution of the Markovian trajectory to reach a steady state, which is incurred at the very beginning and becomes amortized as the algorithm runs. Encouragingly, the above bound improves upon the state-of-the-art result \cite{qu2020finite} by a factor of at least $|\mathcal{S}||\mathcal{A}|$
for all scenarios,  and by a factor of at least $\tmix|\mathcal{S}||\mathcal{A}|$ for any sufficiently small accuracy level $\varepsilon$.   
Further, we demonstrate that the scaling on the effective horizon $\frac{1}{1-\gamma}$ can be improved by means of variance reduction.

%Our result confirms that if the mixing time is not too large, then the convergence of asynchronous Q-learning resembles the synchronous case with independent samples. 
% --- that depends on the mixing time but not the accuracy level --- 

\end{abstract}


\tableofcontents

%%%%%%%%%%%%%%%%%%%%%%%%%%%%%%%%%%%%%%%%%%%%%%%%%%%%%%%%%%%%%%%%%%%%%%%%%%%%%%%%%%%%%%%

\section{Introduction}
\label{sec:intro}

Model-free algorithms such as Q-learning \citep{watkins1992q} play a central role in recent breakthroughs of reinforcement learning (RL) \citep{mnih2015human}. In contrast to model-based algorithms that decouple model estimation and planning, model-free algorithms attempt to directly interact with the environment  --- in the form of a policy that selects actions based on perceived states of the environment --- from the collected data samples, without modeling the environment explicitly. 
Therefore, model-free algorithms are able to process data in an online fashion and are often memory-efficient. 
Understanding and improving the sample efficiency of model-free algorithms lie at the core of recent research activity \citep{dulac2019challenges}, whose importance is particularly evident for the class of RL applications in which data collection is costly and time-consuming (such as clinical trials, online advertisements, and so on).

The current paper concentrates on Q-learning, an off-policy model-free algorithm that seeks to learn the optimal action-value function by observing what happens under a behavior policy.  The off-policy feature makes it appealing in various RL applications where it is infeasible to change the policy under evaluation on the fly.  There are two basic update models in Q-learning. The first one is termed a {\em synchronous} setting, which hypothesizes on the existence of a simulator (also called a generative model); at each time, the simulator generates an independent sample for every state-action pair, and the estimates are updated simultaneously across all state-action pairs. The second model concerns an {\em asynchronous} setting, where only a single sample trajectory following a behavior policy is accessible; at each time, the algorithm updates its estimate of a single state-action pair using one state transition from the trajectory. 
Obviously, understanding the asynchronous setting is considerably more challenging than the synchronous model, due to the Markovian (and hence non-i.i.d.) nature of its sampling process.

Focusing on an infinite-horizon Markov decision process (MDP) with state space $\cS$ and action space $\cA$, this work investigates asynchronous Q-learning on a single {\em Markovian trajectory} induced by a behavior policy. We ask a fundamental question:
\begin{itemize}
	\item[]\emph{How many samples are needed for asynchronous Q-learning to learn the optimal Q-function?}
\end{itemize}
\noindent Despite a considerable number of prior works analyzing this algorithm (ranging from the classical works \citet{tsitsiklis1994asynchronous,jaakkola1994convergence} to the very recent paper \cite{qu2020finite}), it remains unclear whether existing sample complexity analysis of asynchronous Q-learning is tight. As we shall elucidate momentarily, there exists a large gap --- at least as large as $|\cS||\cA|$ --- between the state-of-the-art sample complexity bound for asynchronous Q-learning  \citep{qu2020finite} and the one derived for the synchronous counterpart \citep{wainwright2019stochastic}. This raises a natural desire to examine whether there is any bottleneck intrinsic to the asynchronous setting that significantly limits its performance.

% In this work, we investigate the sample efficiency of asynchronous Q-learning on a single Markovian trajectory, focusing on a $\gamma$-discounted Markov decision process (MDP) with state space $\cS$ and action space $\cA$.  Despite a large number of prior work exploring this setting (ranging from the classical work \cite{tsitsiklis1994asynchronous,jaakkola1994convergence} to the very recent paper \cite{qu2020finite}), the sample complexity of asynchronous Q-learning remains unsettled. As we shall discuss momentarily, there exists a large gap --- at least as large as $|\cS||\cA|$ --- between the state-of-the-art sample complexity bound for asynchronous Q-learning  \cite{qu2020finite} and the one derived for the synchronous counterpart \cite{wainwright2019stochastic}. 

\newcommand{\topsepremove}{\aboverulesep = 0mm \belowrulesep = 0mm} \topsepremove

\begin{table}[t]
\centering
\begin{tabular}{c|c|c}
	\toprule
Algorithm & Sample complexity  &  Learning rate   $\vphantom{\frac{1^{7}}{1^{7^3}}}$   \\ 
% \hline \hline
\toprule
	 Asynchronous Q-learning  & \multirow{2}{*}{$ \frac{ (\tcovertime) ^{\frac{1}{1-\gamma}}}{(1-\gamma)^4\varepsilon^2}$} & \multirow{2}{*}{linear: $\frac{1}{t}$  $\vphantom{\frac{1^{7^{7^{7^{7}}}}}{1^{7^{7^{7^{7}}}}}}$} \tabularnewline
\cite{even2003learning} &  & \tabularnewline 
\hline
	Asynchronous Q-learning  & \multirow{2}{*}{$\big( \frac{ \tcovertime^{1+3\omega } }{(1-\gamma)^4\varepsilon^2}\big)^{\frac{1}{\omega}} + \big(  \frac{ \tcovertime }{1-\gamma }\big)^{\frac{1}{1-\omega}}$} & \multirow{2}{*}{polynomial: $\frac{1}{t^{\omega}}$, $\omega \in (\frac{1}{2},1)$   $\vphantom{\frac{1^{7^{7^{7^7}}}}{1^{7^{7^{7^{7}}}}}}$} \tabularnewline
\cite{even2003learning} &  & \tabularnewline 
\hline
	 Asynchronous Q-learning & \multirow{2}{*}{$ \frac{\tcovertime^3|\cS||\cA|}{(1-\gamma)^5\varepsilon^2}$} & \multirow{2}{*}{constant: $\frac{(1-\gamma)^4\varepsilon^2}{|\cS||\cA|\tcovertime^2}$  $\vphantom{\frac{1^{7^{7^{7}}}}{1^{7^{7^{7^{7}}}}}}$} \tabularnewline
\cite{beck2012error} &  & \tabularnewline 
\hline
	Asynchronous Q-learning &  \multirow{2}{*}{$\frac{\tmix}{\mumin^2 (1-\gamma)^5\varepsilon^2} $} & \multirow{2}{*}{rescaled linear: $\frac{\frac{1}{\mumin(1-\gamma)}}{t+ \max\{ \frac{1}{\mumin(1-\gamma)},  \tmix  \}  }$  $\vphantom{\frac{1^{7^{7^{7^{7}}}}}{1^{7^{7^{7^{7^{7^7}}}}}}}$} \tabularnewline
\cite{qu2020finite} &  & \tabularnewline 
\hline
	 Speedy Q-learning &  \multirow{2}{*}{$\frac{\tcovertime}{(1-\gamma)^4\varepsilon^2} $} & \multirow{2}{*}{rescaled linear: $\frac{1}{t+1}$  $\vphantom{\frac{1^{7^{7^{7^{7}}}}}{1^{7^{7^{7^{7}}}}}}$} \tabularnewline
\citep{azar2011reinforcement} &  & \tabularnewline 
\hline
	 Asynchronous Q-learning & \multirow{2}{*}{$ \frac{1}{\mumin(1-\gamma)^5\varepsilon^2}+ \frac{\tmix}{\mumin(1-\gamma)}  $} & \multirow{2}{*}{constant: $\min \big\{ \frac{(1-\gamma)^4\varepsilon^2}{\gamma^2}, \frac{1}{\tmix} \big\}$  $\vphantom{\frac{1^{7^{7}}}{1^{7^{7^{7^{7}}}}}}$} \tabularnewline
\textbf{This work (Theorem~\ref{thm:main-asyn-Q-learning})} &  & \tabularnewline 
\hline
	 Asynchronous Q-learning & \multirow{2}{*}{$ \frac{\tcovertime}{(1-\gamma)^5\varepsilon^2}  $} & \multirow{2}{*}{constant: $\min \big\{ \frac{(1-\gamma)^4\varepsilon^2}{\gamma^2}, 1 \big\}$  $\vphantom{\frac{1^{7^{7}}}{1^{7^{7^{7^{7}}}}}}$} \tabularnewline
\textbf{This work (Theorem~\ref{thm:main-asyn-Q-learning-cover-time})} &  & \tabularnewline 
\hline
	 Asynchronous Q-learning & \multirow{2}{*}{$ \frac{1}{\mumin(1-\gamma)^5\varepsilon^2}+ \frac{\tmix}{\mumin(1-\gamma)}$} & \multirow{2}{*}{piecewise constant rescaled linear:~\eqref{eq:learning-rate-implement}  $\vphantom{\frac{1^{7^{7}}}{1^{7^{7^{7^{7}}}}}}$} \tabularnewline
\textbf{This work (Theorem~\ref{thm:new-learning-rates-asyn-Q})} &  & \tabularnewline 
\hline
	 Variance-reduced Q-learning & \multirow{2}{*}{$ \frac{1}{\mumin(1-\gamma)^3\varepsilon^2}+ \frac{\tmix}{\mumin(1-\gamma)}$} & \multirow{2}{*}{constant: $\min \big\{ \frac{(1-\gamma)^2}{\gamma^2}, \frac{1}{\tmix} \big\}$  $\vphantom{\frac{1^{7^{7}}}{1^{7^{7^{7^{7}}}}}}$} \tabularnewline
\textbf{This work (Theorem~\ref{thm:main-asyn-VR-Q-learning})} &  & \tabularnewline 
	\toprule
\end{tabular}
\caption{Sample complexity of asynchronous Q-learning and its variants to compute an $\varepsilon$-optimal
Q-function in the $\ell_{\infty}$ norm, where we hide all logarithmic
factors. With regards to the Markovian trajectory induced by the behavior
policy, we denote by $\tcovertime$, $\tmix$, and $\mumin$ the cover
time, mixing time, and minimum state-action occupancy probability
of the associated stationary distribution, respectively.} \label{table:comparison_asyncQ}
\end{table}

\subsection{Main contributions}

This paper develops a refined analysis framework that sharpens our understanding about the sample efficiency of classical asynchronous Q-learning on a single sample trajectory.  
Setting the stage, consider an infinite-horizon MDP with state space $\cS$, action space $\cA$, and a discount factor $\gamma \in (0,1)$.  What we have access to is a sample trajectory of the MDP induced by a stationary behavior policy. In contrast to the synchronous setting with i.i.d.~samples, we single out two parameters intrinsic to the Markovian sample trajectory: \emph{(i)}  the mixing time $\tmix$, which characterizes how fast the trajectory disentangles itself from the initial state; \emph{(ii)} the smallest state-action occupancy probability $\mumin$ of the stationary distribution of the trajectory, which captures how frequent each state-action pair has been at least visited. 

With these parameters in place, our findings unveil that: the sample complexity required for asynchronous Q-learning to yield an $\varepsilon$-optimal Q-function estimate --- in a strong $\ell_{\infty}$ sense --- is at most\footnote{Let $\mathcal{X}:= \big( |\mathcal{S}|,|\mathcal{A}|, \frac{1}{1-\gamma}, \frac{1}{\varepsilon} \big)$.  The notation $f(\mathcal{X}) = O(g(\mathcal{X}))$  means there exists a universal constant $C_1>0$ such that $f\leq C_1 g$. The notation $\widetilde{O}(\cdot)$ is defined analogously except that it hides any logarithmic factor.} 
\begin{equation}
	\widetilde{O}\Big(  \frac{1}{\mumin(1-\gamma)^5\varepsilon^2}+ \frac{\tmix}{\mumin(1-\gamma)}  \Big). 
	\label{eq:sample-size-Q-contributions}
\end{equation}
The first component of \eqref{eq:sample-size-Q-contributions} is consistent with the sample complexity derived for the setting with independent samples drawn from the stationary distribution of the trajectory \citep{wainwright2019stochastic}. In comparison, the second term of \eqref{eq:sample-size-Q-contributions} --- which is unaffected by the accuracy level $\varepsilon$ --- is intrinsic to the Markovian nature of the trajectory; in essence, this term reflects the cost taken for the empirical distribution of the sample trajectory to converge to a steady state, and becomes amortized as the algorithm runs.  
%Interestingly, the influence of the mixing time --- as reflected by the second term of \eqref{eq:sample-size-Q-contributions} --- is nearly unaffected by the accuracy level $\varepsilon$. In fact, as long as the behavior Markov chain mixes reasonably fast in the sense that $\tmix \leq \frac{1}{(1-\gamma)^4\varepsilon^2}$, the sample complexity reduces to $\widetilde{O}\big(  \frac{1}{\mumin(1-\gamma)^5\varepsilon^2} \big)$ and no longer varies with $\tmix$; in this case, 
In other words, the behavior of asynchronous Q-learning would resemble what happens in the setting with independent samples, as long as the algorithm has been run for reasonably long. In addition,  our analysis framework readily yields another sample complexity bound
\begin{equation}
	\widetilde{O}\Big(  \frac{\tcovertime}{(1-\gamma)^5\varepsilon^2}  \Big),
	\label{eq:sample-size-Q-contributions-cover-time}
\end{equation}
where $\tcovertime$ stands for the cover time --- namely, the time taken for the trajectory to visit all state-action pairs at least once. This facilitates comparisons with several prior results based on the cover time. 

% which might of interest as well. 

Furthermore, we leverage the idea of variance reduction to improve the scaling with the discount complexity $\frac{1}{1-\gamma}$. We demonstrate that a variance-reduced variant of asynchronous Q-learning attains $\varepsilon$-accuracy using at most 
\begin{equation}
	\widetilde{O}\Big(  \frac{1}{\mumin(1-\gamma)^3 \min \{1, \varepsilon^2\}}+ \frac{\tmix}{\mumin(1-\gamma)}  \Big)
	\label{eq:sample-size-VRQ-contributions}
\end{equation}
samples, matching the complexity of its synchronous counterpart if $\varepsilon \leq \min\big\{ 1, \frac{1}{(1-\gamma)\sqrt{\tmix}} \big\}$ \citep{wainwright2019variance}. Moreover, by taking the action space to be a singleton set, the aforementioned results immediately lead to $\ell_{\infty}$-based sample complexity guarantees for temporal difference (TD) learning \citep{sutton1988learning} on Markovian samples.

\paragraph{Comparisons with past results.} A large fraction of the classical literature focused on asymptotic convergence analysis of asynchronous Q-learning (e.g.~\cite{tsitsiklis1994asynchronous,jaakkola1994convergence,szepesvari1998asymptotic}); these results, however, did not lead to non-asymptotic sample complexity bounds. The state-of-the-art sample complexity analysis was due to the recent work \cite{qu2020finite}, which derived a sample complexity bound $\widetilde{O}\big( \frac{\tmix}{\mumin^2(1-\gamma)^5\varepsilon^2}\big)$. Given the obvious lower bound $1/\mumin \geq |\cS||\cA|$, our result \eqref{eq:sample-size-Q-contributions} improves upon that of \cite{qu2020finite} by a factor at least on the order of $  |\cS||\cA| \min\big\{ \tmix, \frac{1}{(1-\gamma)^4\varepsilon^2} \big\} $. 
In particular, for sufficiently small accuracy level $\varepsilon$, our improvement exceeds a factor of at least $$\tmix |\cS||\cA|  .$$ 
In addition, we note that several prior works \citep{even2003learning,beck2012error} developed sample complexity bounds in terms of the cover time $\tcovertime$ of the sample trajectory; our result strengthens these bounds by a factor of at least $$\tcovertime^2|\cS||\cA|\geq |\cS|^3|\cA|^3.$$  The interested reader is referred to Table~\ref{table:comparison_asyncQ} for more precise comparisons, and to Section~\ref{sec:related-work} for a discussion of further related works.

\subsection{Paper organization,  notation, and basic concept}
\label{sec:notation}

The remainder of the paper is organized as follows. Section~\ref{sec:models} formulates the problem and introduces some basic quantities and assumptions. 
Section~\ref{sec:async-q-learning} presents the asynchronous Q-learning algorithm along with its theoretical guarantees, 
whereas Section~\ref{sec:VR-Q} accommodates the extension: asynchronous variance-reduced Q-learning.
A more detailed account of related works is given in Section~\ref{sec:related-work}.
The analyses of our main theorems are described in Sections~\ref{Sec:Analysis-asynchronous}-\ref{Sec:Asynchronous-VR}.
We conclude this paper with a summary of our results and a list of future directions in Section~\ref{sec:discussion}.
Several preliminary facts about Markov chains and the proofs of technical lemmas are postponed to the appendix.

Next, we introduce a set of notation that will be used throughout the paper. 
Denote by $\Delta(\cS)$ (resp.~$\Delta(\cA)$) the probability simplex over the set $\cS$ (resp.~$\cA$).
For any vector $\bm{z}=[z_i]_{1\leq i\leq n}\in \mathbb{R}^n$, we overload the notation $\sqrt{\cdot}$ and $|\cdot|$ to denote entry-wise operations, such that $\sqrt{\bm{z}}:=[\sqrt{z_i}]_{1\leq i\leq n}$ and $|\bm{z}|:=[|z_i|]_{1\leq i\leq n}$. For any vectors $\bm{z}=[a_i]_{1\leq i\leq n}$ and $\bm{w}=[w_i]_{1\leq i\leq n}$, the notation $\bm{z} \geq \bm{w}$ (resp.~$\bm{z} \leq \bm{w}$) means $z_i \geq w_i$ (resp.~$z_i\leq w_i$) for all $1\leq i\leq n$.
%, and we let $\bm{a} \circ \bm{b} := [a_ib_i]_{1\leq i\leq n}$ represent the Hadamard product. 
Additionally, we denote by $\bm{1}$ the all-one vector, $\bm{I}$ the identity matrix, 
and $\ind\{\cdot\}$ the indicator function. For any matrix $\bm{P}=[P_{ij}]$, we denote $\|\bm{P}\|_1:=  \max_i\sum_j|P_{ij}|$. 
% which equals $1$ if and only if the event $\mathcal{B}$ is true, and $0$ otherwise. 
Throughout this paper, we use $c, c_{0}, c_{1},\cdots$ to denote universal constants that do not depend either on the parameters of the MDP or the target levels $(\varepsilon,\delta)$, and their exact values may change from line to line.

Finally, let us introduce the concept of uniform ergodicity for Markov chains. 
Consider any Markov chain $(X_0, X_1, X_2, \cdots)$ with transition kernel $P$, finite state space $\mathcal{X}$ and stationary distribution $\mu$, 
	and denote by $P^t(\cdot\,|\,x)$  the distribution of $X_t$ conditioned on $X_0=x \in \mathcal{X}$. 
	This Markov chain is said to be {\em uniformly ergodic} if, for some $\rho < 1$ and $M < \infty$, one has
\begin{equation}
\sup_{x \in \mathcal{X}} d_{\mathsf{TV}} \big(\mu, P^t(\cdot\,|\,x) \big) \le M\rho^t,
\end{equation}
	where $d_{\mathsf{TV}} (\mu, \nu)$ stands for the total variation distance between two distributions $\mu$ and $\nu$ \citep{tsybakov2009introduction}:  
\begin{equation} \label{eq:TV}
	d_{\mathsf{TV}} (\mu, \nu) := \frac{1}{2} \sum_{x \in \mathcal{X}} \big|\mu(x) - \nu(x) \big|
	= \sup_{A\subseteq \mathcal{X}} \big| \mu(A) - \nu(A) \big|.
\end{equation}
%

% In addition, the notation $\widetilde{O}(\cdot)$ (resp.~$\widetilde{\Omega}(\cdot)$) is defined in the same way as ${O}(\cdot)$ (resp.~$\Omega(\cdot)$) except that it ignores logarithmic factors. 

%\section{Asynchronous Q-learning}

\section{Models and background}
\label{sec:models}

This paper studies an infinite-horizon MDP with discounted rewards, as represented by a quintuple $\mathcal{M} = (\cS,\cA, P, r,\gamma)$. Here, $\cS$ and $\cA$ denote respectively the (finite) state space and action space, whereas $\gamma\in (0,1)$ indicates the discount factor.
Particular emphasis is placed on the scenario with large state/action space and long effective horizon, namely, $|\cS|$, $|\cA|$ and the effective horizon $\frac{1}{1-\gamma}$ can all be quite large.  
We use $P:\cS\times\cA \rightarrow \Delta(\cS)$ to represent the probability transition kernel of the MDP, where for each state-action pair $(s,a)\in \cS\times \cA$, $P(s' \,|\, {s,a})$ denotes the probability of transiting to state $s'$ from state $s$ when action $a$ is executed. The reward function is represented by $r:  \cS\times\cA \rightarrow [0,1]$, such that $r(s,a)$ denotes the immediate reward from state $s$ when action $a$ is taken; for simplicity, we assume throughout that all rewards lie within $[0,1]$.
We focus on the tabular setting which, despite its basic form, has not yet been well understood.
See \cite{bertsekas2017dynamic} for an in-depth introduction of this model.

\paragraph{Q-function and Bellman operator.} 
An action selection rule is termed a {\em policy} and represented by a mapping $\pi: \cS \rightarrow \Delta(\cA)$, which maps a state to a distribution over the set of actions.
A policy is said to be {\em stationary} if it is time-invariant. 
We denote by $\{s_t,a_t,r_t\}_{t=0}^{\infty}$ a sample trajectory, where $s_t$ (resp.~$a_t$) denotes the state (resp.~the action taken) at time $t$, and $r_t= r(s_{t}, a_{t})$ denotes the reward received at time $t$. It is assumed throughout that the rewards are deterministic and depend solely upon the current state-action pair.
We denote by $V^{\pi}: \cS \rightarrow \real$ the value function of a policy $\pi$, namely, 
\begin{align*}
 \forall s\in \cS: \qquad V^{\pi}(s) \defn \mathbb{E} \left[ \sum_{t=0}^{\infty} \gamma^t r(s_t,a_t ) \,\big|\, s_0 =s \right],
\end{align*} 
which is the expected discounted cumulative reward received when \emph{(i)} the initial state is $s_0=s$, \emph{(ii)} the actions are taken based on the policy $\pi$ (namely, $a_t \sim \pi(s_t)$ for all $t\geq 0$) and the trajectory is generated based on the transition kernel (namely, $s_{t+1}\sim P(\cdot | s_t, a_t)$). It can be easily verified that $0\leq V^{\pi}(s)\leq \frac{1}{1-\gamma} $ for any $\pi$.
The action-value function (also Q-function) $Q^{\pi}: \cS \times \cA \rightarrow \real$ of a policy $\pi$ is defined by 
\begin{equation*}
 \forall (s,a)\in \cS \times \cA: \qquad Q^{\pi}(s,a) \defn \mathbb{E} \left[ \sum_{t=0}^{\infty} \gamma^t r(s_t,a_t ) \,\big|\, s_0 =s, a_0 = a \right],
\end{equation*} 
where the actions are taken according to the policy $\pi$ except the initial action (i.e.~$a_t \sim \pi(s_t)$ for all $t\geq 1$).  
As is well-known, there exists an optimal policy --- denoted by $\pi^{\star}$ --- that simultaneously maximizes $V^{\pi}(s)$ and $Q^{\pi}(s,a)$ uniformly over all state-action pairs $(s,a)\in (\cS\times\cA)$.
Here and throughout, we shall denote by $V^{\star} :=V^{\pi^{\star}}$ and $Q^{\star} := Q^{\pi^{\star}}$ the optimal value function and the optimal Q-function, respectively.

In addition, the Bellman operator $\mathcal{T}$, which is a mapping from  $\mathbb{R}^{|\cS|\times |\cA|}$ to itself, is defined such that the $(s,a)$-th entry of $\mathcal{T}(Q)$ is given by
\begin{equation}
	\mathcal{T}(Q)(s,a) := r(s,a) + \gamma \mathop{\mathbb{E}}\limits_{s^{\prime} \sim P(\cdot| s,a)}  \Big[ \max_{a^{\prime}\in \cA} Q(s^{\prime}, a^{\prime}) \Big].
\end{equation}
It is well known that the optimal Q-function $Q^{\star}$ is the unique fixed point of the Bellman operator.  

\paragraph{Sample trajectory and behavior policy.}
Imagine we have access to a sample trajectory 
 $\{s_t,a_t,r_t\}_{t=0}^{\infty}$ generated by the MDP $\mathcal{M}$ under a given stationary policy $\pi_{\mathsf{b}}$ --- called a {\em behavior policy}. The behavior policy is deployed to help one learn the ``behavior'' of the MDP under consideration, which 
often differs from the optimal policy being sought. Given the stationarity of $\pi_{\mathsf{b}}$, the sample trajectory can be viewed as a sample path of a time-homogeneous Markov chain over the set of state-action pairs $\{(s,a)\mid s\in \cS, a\in \cA\}$. 
Throughout this paper, we impose the following uniform ergodicity assumption  \citep{paulin2015concentration} (see the definition of uniform ergodicity in Section~\ref{sec:notation}).
\begin{assumption}
	\label{assumption:uniform-ergodic}
	The Markov chain induced by the stationary behavior policy $\pi_{\mathsf{b}}$ is uniformly ergodic.
\end{assumption}
%
%\begin{remark}[Uniform ergodicity]
%\label{remark:uniform-ergodic}
%For completeness, we provide the definition of uniform ergodicity here. 
%\end{remark}

There are several properties concerning the behavior policy and its resulting Markov chain that play a crucial role in learning the optimal Q-function. Specifically, 
denote by $\mu_{\pi_{\mathsf{b}}}$ the stationary distribution (over all state-action pairs) of the aforementioned behavior Markov chain, and define
\begin{align}
	\label{defn:mu-min}
	\mumin:= \min_{(s,a)\in \cS\times \cA} \mu_{\pi_{\mathsf{b}}} (s,a). 
\end{align}
Intuitively, $\mumin$ reflects an information bottleneck; that is, the smaller $\mumin$ is, the more samples are needed in order to ensure all state-action pairs are visited sufficiently many times. 
In addition, we define the associated mixing time of the chain as 
\begin{align}
	\tmix := \min \Big\{ t ~\Big|~   \max_{ (s_0,a_0) \in \mathcal{S}\times\cA} d_{\mathsf{TV}}\big( P^t( \cdot \,|\, s_0,a_0), \mu_{\pi_{\mathsf{b}}} \big) \leq \frac{1}{4} \Big\},
	\label{defn:mixing-time-MC}
\end{align}
where $P^t( \cdot| s_0,a_0)$ denotes the distribution of $(s_t,a_t)$ conditional on the initial state-action pair $(s_0,a_0)$, 
and  $d_{\mathsf{TV}} (\mu, \nu)$ is the total variation distance between $\mu$ and $\nu$ (see \eqref{eq:TV}).  In words, the mixing time $\tmix$ captures how fast the sample trajectory decorrelates from its initial state. Moreover, we define the cover time associated with this Markov chain as follows
\begin{align}
	\tcovertime := \min \Big\{ t \mid  \min_{(s_0,a_0)\in \cS\times \cA} \mathbb{P}\big( \mathcal{B}_{t} \,|\, s_0, a_0  \big) \geq \frac{1}{2} \Big\} ,
	\label{eq:defn-cover-time}
\end{align}
where $\mathcal{B}_{t}$ denotes the event such that all $(s,a)\in \cS\times \cA$ have been visited at least once between time 0 and time $t$, and $\mathbb{P}\big( \mathcal{B}_{t} \,|\, s_0, a_0  \big) $ denotes the probability of $\mathcal{B}_t$ conditional on the initial state $(s_0,a_0)$. 
\begin{remark}
	It is known that for a finite-state Markov chain, 
	having a finite mixing time $\tmix$ implies uniform ergodicity of the chain \citep[Page 4]{paulin2015concentration}.   
	Thus, our uniform ergodicity assumption is equivalent to the assumption imposed in \citet{qu2020finite} (which assumes ergodicity in addition to a finite $\tmix$). 
\end{remark}

\paragraph{Goal.} Given {\em a single} sample trajectory $\{s_t, a_t,r_t\}_{t=0}^{\infty}$ generated by the behavior policy $\pi_{\mathsf{b}}$, we aim to compute/approximate the optimal Q-function $Q^{\star}$ in an $\ell_\infty$ sense. 
This setting --- in which a state-action pair can be updated only when the Markovian trajectory reaches it~--- is commonly referred to as {\em asynchronous} Q-learning \citep{tsitsiklis1994asynchronous,qu2020finite} in tabular RL. The current paper focuses on characterizing, in a non-asymptotic manner, the sample efficiency of classical Q-learning and its variance-reduced variant.  

%One should distinguish it from the {\em synchronous} counterpart \citep{wainwright2019stochastic} in which the samples are independently generated and all state-action pairs are updated simultaneously. 

\section{Asynchronous Q-learning on a single Markovian trajectory}
\label{sec:async-q-learning}

\subsection{Algorithm}

The Q-learning algorithm \citep{watkins1992q} is arguably one of the most famous off-policy algorithms aimed at learning the optimal Q-function. Given the Markovian trajectory $\{s_t, a_t,r_t\}_{t=0}^{\infty}$ generated by the behavior policy $\pi_{\mathsf{b}}$, the asynchronous Q-learning algorithm maintains a Q-function estimate $Q_t: \cS\times \cA \rightarrow \mathbb{R}$ at each time $t$ and adopts the following iterative update rule
\begin{equation}
\label{eqn:q-learning}
\begin{split}
	Q_t(s_{t-1},a_{t-1}) &= (1- \eta_t ) Q_{t-1}(s_{t-1},a_{t-1}) + \eta_t \mathcal{T}_t (Q_{t-1}) (s_{t-1}, a_{t-1})  \\
Q_t(s ,a) & = Q_{t-1}(s ,a ), \qquad \forall (s,a)\neq (s_{t-1},a_{t-1})
\end{split}
\end{equation}
for any $t\geq 0$, whereas $\eta_t$ denotes the learning rate or the stepsize. 
Here, $\mathcal{T}_t$ denotes the empirical Bellman operator w.r.t.~the $t$-th sample, that is, 
\begin{align}
	\mathcal{T}_t (Q) (s_{t-1}, a_{t-1}) := 
		r(s_{t-1}, a_{t-1}) + \gamma \max_{a' \in \cA} Q(s_t, a').
		%, \quad & \text{if }(s,a)=(s_{t-1}, a_{t-1}) .
	\label{defn:empirical-Bellman-t}
\end{align}
It is worth emphasizing that at each time $t$, only a single entry --- the one corresponding to the sampled state-action pair $(s_{t-1},a_{t-1})$ --- is updated, with all remaining entries unaltered. 
While the estimate $Q_0$ can be initialized to arbitrary values, we shall set $Q_0(s,a) = 0$ for all $(s,a)$ unless otherwise noted.  
The corresponding value function estimate $V_t: \cS \rightarrow \mathbb{R}$  at time $t$ is thus given by
\begin{align}
	\label{defn:Vt}
	\forall s\in \cS: \qquad V_t(s) :=  \max_{a\in \cA} Q_{t}(s,a) .
\end{align}
The complete algorithm is described in Algorithm~\ref{alg:async-q}.

%\begin{algorithm}[t]
% \caption{Asynchronous Q-learning} \label{alg:async-vr-q} 
% \begin{algorithmic}[1] 
% \STATE \textbf{Input parameter:} learning rate $\eta$.
% \STATE \textbf{Initialization:} $Q_0(s,a)=0$ for all $(s,a)\in \cS\times \cA$. 
% \FOR{$t=1,2,\dots,T$}
% \STATE Draw an action $a_{t-1} \sim \pi(\cdot |s_{t-1})$ and the next state $s_{t}\sim P(\cdot|s_{t-1},a_{t-1})$ from the MDP;
% \STATE{Update the Q-function $Q_t$ according to \eqref{eqn:q-learning}.}
%  \ENDFOR
% \end{algorithmic} 
% \end{algorithm}

\begin{algorithm}[ht]
\DontPrintSemicolon
	\textbf{input parameters:} learning rates $\{\eta_t\}$, number of iterations $T$. \\
  \textbf{initialization:} $Q_0=0$. \\

   \For{$t=1,2,\cdots,T$}
	{
		Draw action $a_{t-1} \sim \pi_{\mathsf{b}}(s_{t-1})$,
		observe reward $r(s_{t-1},a_{t-1})$, and draw next state $s_{t}\sim P(\cdot\,|\,s_{t-1},a_{t-1})$. \\
		%\blue{O.} \\ 
	     	Update $Q_t$ according to \eqref{eqn:q-learning}.
	}

   %\textbf{return:} ${Q}_{T}$. 
    \caption{Asynchronous Q-learning}
 \label{alg:async-q}
\end{algorithm}

\subsection{Theoretical guarantees for asynchronous Q-learning}
We are in a position to present our main theory regarding the non-asymptotic sample complexity of  asynchronous Q-learning, for which the key parameters $\mumin$ and $\tmix$  defined respectively in \eqref{defn:mu-min} and \eqref{defn:mixing-time-MC} play a vital role. The proof of this result is provided in Section~\ref{Sec:Analysis-asynchronous}.
\begin{theorem}[Asynchronous Q-learning]
	\label{thm:main-asyn-Q-learning}
	For the asynchronous Q-learning algorithm detailed in Algorithm~\ref{alg:async-q}, 
	there exist some universal constants $c_0,c_1>0$ such that for any $0<\delta<1$ and $0<\varepsilon \leq \frac{1}{1-\gamma}$,  one has
	\begin{align*}
		\forall (s,a)\in \cS\times \cA: \qquad |Q_T (s,a) - Q^\star(s,a)| \leq \varepsilon
	\end{align*}
	with probability at least $1-\delta$, provided that the iteration number $T$ and the learning rates $\eta_t \equiv \eta $ obey
	\begin{subequations}
	\begin{align}	
		T & \geq \frac{c_{0}}{\mumin}\left\{ \frac{1}{(1-\gamma)^{5}\varepsilon^{2}}+\frac{\tmix}{1-\gamma}\right\} 
		\log\Big( \frac{|\mathcal{S}||\mathcal{A}|T}{\delta} \Big) \log \Big(\frac{1}{(1-\gamma)^2\varepsilon} \Big), \label{eqn:numbersteps}\\
		\eta & =\frac{c_{1}}{\log\big(\frac{|\mathcal{S}||\mathcal{A}|T}{\delta}\big)}\min\left\{ \frac{(1-\gamma)^{4}\varepsilon^{2}}{\gamma^2},~\frac{1}{\tmix}\right\} .
		\label{eq:learning-rate-asynQ}
	\end{align}
	\end{subequations}
\end{theorem}
\begin{remark}
	The careful reader might immediately remark that the learning rate $\eta$ studied in Theorem~\ref{thm:main-asyn-Q-learning} 
	relies on prior knowledge of $\varepsilon$, $\delta$ and $T$. 
	This is more stringent than  the learning rates in \cite{qu2020finite}, which do not require pre-determining these parameters. 
	To address this issue, we will explore a more adaptive learning rate schedule shortly in 
	 Section~\ref{section:adaptive-learning-rates}, which achieves the same sample complexity without the need of knowing these parameters {\em a priori}.
	% In addition, Theorem~\ref{thm:main-asyn-Q-learning} is stated for a given $\varepsilon$-level, and hence it does not capture further improvement of the estimation error when $T$ continues to increase. This is in contrast to \citet[Theorem 7]{qu2020finite} which reflects continued improvement for all large enough $T$; we shall address this issue again in  Section~\ref{section:adaptive-learning-rates} via the use of adaptive learning rates.  
\end{remark}
Theorem~\ref{thm:main-asyn-Q-learning} delivers a finite-sample/finite-time analysis of asynchronous Q-learning, given that a fixed learning rate is adopted and chosen appropriately. The $\ell_{\infty}$-based sample complexity required for Algorithm~\ref{alg:async-q} to attain $\varepsilon$ accuracy is at most
\begin{align}
	\widetilde{O}\Big( \frac{1}{\mumin (1-\gamma)^5 \varepsilon^2}  +  \frac{\tmix}{\mumin (1-\gamma)} \Big). 
	\label{eq:sample-size-Q-asynchronous}
\end{align}
A few implications are in order.

\paragraph{Dependency on the minimum state-action occupancy probability $\mumin$.} Our sample complexity bound \eqref{eq:sample-size-Q-asynchronous} scales linearly in $1/\mumin$, which is in general unimprovable.  Consider, for instance, the ideal scenario where state-action occupancy is nearly uniform across all state-action pairs, in which case  $1/\mumin$ is on the order of $|\cS||\cA|$. In such a ``near-uniform'' case, the sample complexity scales linearly with $|\cS||\cA|$, and this dependency matches the known minimax lower bound \cite{azar2013minimax} derived for the setting with independent samples. In comparison, \citet[Theorem 7]{qu2020finite} depends at least quadratically on $1/\mumin$, which is at least $|\cS||\cA|$ times larger than  our result \eqref{eq:sample-size-Q-asynchronous}. 

%and \citet[Theorem 4]{even2003learning} at least exceeds the order of $(1/\mumin)^{4.29}$,\footnote{Note that the cover time also exceeds the order of $1/\mumin$.}  both of which are

\paragraph{Dependency on the effective horizon $\frac{1}{1-\gamma}$.} The sample size bound \eqref{eq:sample-size-Q-asynchronous} scales as $\frac{1}{(1-\gamma)^5\varepsilon^2}$, which coincides with both \cite{wainwright2019stochastic,chen2020finite} (for the synchronous setting) and \citet{beck2012error,qu2020finite} (for the asynchronous setting) with either a rescaled linear learning rate or a constant learning rate. This turns out to be the sharpest scaling known to date for the classical form of Q-learning.  

%the mixing time does not play any role in our sample complexity, and

\paragraph{Dependency on the mixing time $\tmix$.} The second additive term of our sample complexity \eqref{eq:sample-size-Q-asynchronous} depends linearly on the mixing time $\tmix$ and is (almost) independent of the target accuracy $\varepsilon$. The influence of this mixing term is a consequence of the expense taken for the Markovian trajectory to reach a steady state, which is a one-time cost that can be amortized over later iterations if the algorithm is run for reasonably long. Put another way, if the behavior chain mixes not too slowly with respect to $\varepsilon$ (in the sense that
$\tmix\leq \frac{1}{(1-\gamma)^4\varepsilon^2} $), then the algorithm behaves as if the samples were independently drawn from the stationary distribution of the trajectory. In comparison, the influences of $\tmix$ and $\frac{1}{(1-\gamma)^5\varepsilon^2}$ in \cite{qu2020finite} (cf.~Table~\ref{table:comparison_asyncQ}) are multiplicative regardless of the value of $\varepsilon$, thus resulting in a much higher sample complexity.  For instance, if $\varepsilon=O\big(\frac{1}{(1-\gamma)^{2}\sqrt{\tmix}}\big)$, then  the sample complexity result therein is at least $$ \frac{\tmix}{\mumin}  \geq \tmix |\cS||\cA| $$ times larger than our result (modulo some log factor).

\paragraph{Schedule of learning rates.} An interesting aspect of our analysis lies in the adoption of a time-invariant learning rate, under which the $\ell_\infty$ error decays linearly --- down to some error floor whose value is dictated by the learning rate. Therefore, a desired statistical accuracy can be achieved by properly setting the learning rate based on the target accuracy level $\varepsilon$ and then determining the sample complexity accordingly.  In comparison, classical analyses typically adopted a (rescaled) linear or a polynomial learning rule \cite{even2003learning,qu2020finite}. While the work \cite{beck2012error} studied Q-learning with a constant learning rate, their bounds were conservative and fell short of revealing the optimal scaling. Furthermore, we note that adopting time-invariant learning rates is not the only option that enables the advertised sample complexity; as we shall elucidate in Section~\ref{section:adaptive-learning-rates}, one can also adopt carefully designed diminishing learning rates to achieve the same performance guarantees.

\paragraph{Mean estimation error.}
	The high-probability bound in Theorem~\ref{thm:main-asyn-Q-learning} readily translates to a mean estimation error guarantee. 
	To see this, let us first make note of the following basic crude bound (see e.g.~\cite{gosavi2006boundedness,beck2012error})
	\begin{align}
		\big| {Q}_{t}(s,a)  \big| \leq \frac{1}{1-\gamma},
		\qquad \big| {Q}_{t}(s,a) -  {Q}^{\star}(s,a) \big| \leq \frac{1}{1-\gamma}
	\end{align}
	for all $t \geq 0$  and all  $(s, a) \in \cS \times \cA$. 
	%which can be established easily through induction on $t$ according to the update rule~\eqref{eqn:q-learning}.
	By taking $\delta= {\varepsilon(1-\gamma)}$ in Theorem~\ref{thm:main-asyn-Q-learning}, we immediately reach
	\begin{equation}
		\mathbb{E}\Big[ \max_{s,a} \big| {Q}_{T}(s,a) -  {Q}^{\star}(s,a) \big| \Big] \le \varepsilon (1-\delta)+\delta \frac{1}{1-\gamma}
		\leq 2\varepsilon,
	\end{equation}
	provided that $T$ obeys~\eqref{eqn:numbersteps}. As a result, the sample complexity remains unchanged (up to some logarithmic factor) when the goal is to achieve the mean error bound
	$\mathbb{E}\big[ \max_{s,a} \big| {Q}_{T}(s,a) -  {Q}^{\star}(s,a) \big| \big] \le 2\varepsilon$. 
	% The following high-probability bounds can also guarantee mean estimation error with the same arguments.
%\end{remark}

%
% In some sense, our result can be viewed as the asynchronous counterpart of \cite{wainwright2019stochastic} which provides $\ell_\infty$ bounds of Q-learning in the synchronous setting. 

% in the boundary case where $\tmix = \frac{1}{(1-\gamma)^4\varepsilon^2}$, the sample complexity result therein is $\widetilde{O}\big( \frac{1}{(1-\gamma)^4\varepsilon^2} \big)$ larger than ours. 

\bigskip

In addition, our analysis framework immediately leads to another sample complexity guarantee stated in terms of the cover time $\tcovertime$ (cf.~\eqref{eq:defn-cover-time}), which facilitates comparisons with several past work \cite{even2003learning,beck2012error}.  
%\blue{
%}
The proof follows essentially that of Theorem~\ref{thm:main-asyn-Q-learning}, with a sketch provided in Section~\ref{sec:analysis-cover-time}.
\begin{theorem}
	\label{thm:main-asyn-Q-learning-cover-time}
	For the asynchronous Q-learning algorithm detailed in Algorithm~\ref{alg:async-q}, 
	there exist some universal constants $c_0,c_1>0$ such that for any $0<\delta<1$ and $0<\varepsilon \leq \frac{1}{1-\gamma}$,  one has
	\begin{align*}
		\forall (s,a)\in \cS\times \cA: \qquad |Q_T (s,a) - Q^\star(s,a)| \leq \varepsilon
	\end{align*}
	with probability at least $1-\delta$, provided that the iteration number $T$ and the learning rates $\eta_t \equiv \eta $ obey
	\begin{subequations}
	\begin{align}	
		T & \geq  \frac{c_{0} \tcovertime}{(1-\gamma)^{5}\varepsilon^{2}} 
		\log^2\Big( \frac{|\mathcal{S}||\mathcal{A}|T}{\delta} \Big) \log \Big(\frac{1}{(1-\gamma)^2\varepsilon} \Big), \label{eqn:numbersteps-cover-time}\\
		\eta & =\frac{c_{1}}{\log\big(\frac{|\mathcal{S}||\mathcal{A}|T}{\delta}\big)}\min\left\{ \frac{(1-\gamma)^{4}\varepsilon^{2}}{\gamma^2},~1\right\} .
		\label{eq:learning-rate-asynQ-cover-time}
	\end{align}
	\end{subequations}
\end{theorem}
\begin{remark}
	The main difference between the cover-time-based analysis and the mixing-time-based analysis lies in the number of visits to each state-action pair $(s, a)$ in every time frame.  
	Owing to the measure concentration of Markov chains, we can see that the number of visits to each $(s,a)$ concentrates around its expected value in each time frame, 
	which in turn ensures that all state-action pairs have been visited at least once as long as the time frame is sufficiently long. 
	This important property allows one to establish an intimate connection between the analysis of Theorem~\ref{thm:main-asyn-Q-learning} and that of Theorem~\ref{thm:main-asyn-Q-learning-cover-time}. 
\end{remark}
In a nutshell, this theorem tells us that the $\ell_{\infty}$-based sample complexity of classical asynchronous Q-learning is bounded above by
\begin{align}
	\label{eq:sample-complexity-cover-time}
	\widetilde{O} \Big( \frac{ \tcovertime}{(1-\gamma)^{5}\varepsilon^{2}} \Big),
\end{align}
which scales linearly with the cover time. This improves upon the prior result \cite{even2003learning} (resp.~\cite{beck2012error}) by an order of at least 
$$\tcovertime^{3.29} \geq |\cS|^{3.29}|\cA|^{3.29} \qquad  (\text{resp. } \tcovertime^2|\cS||\cA|\geq |\cS|^3|\cA|^3).$$ 
See Table~\ref{table:comparison_asyncQ} for detailed comparisons. We shall further make note of some connections between $\tcovertime$ and $\tmix/\mumin$ to help compare Theorem~\ref{thm:main-asyn-Q-learning} and Theorem~\ref{thm:main-asyn-Q-learning-cover-time}: \emph{(i)} in general, $\tcovertime=\widetilde{O}(\tmix/\mumin)$ for uniformly ergodic chains; \emph{(ii)} one can find some cases where $\tmix/\mumin=\widetilde{O}(\tcovertime)$. Consequently, while Theorem~\ref{thm:main-asyn-Q-learning} does not strictly dominate Theorem~\ref{thm:main-asyn-Q-learning-cover-time} in all instances, the aforementioned connections reveal that   Theorem~\ref{thm:main-asyn-Q-learning} is tighter for the worst-case scenarios. 
The interested reader is referred to Section~\ref{sec:connection-tcover-tmix} for details.

\subsection{A special case: TD learning} 
\label{sec:TD-connection}

In the special circumstance that the set of allowable actions $\cA$ is a singleton, the corresponding MDP reduces to a Markov reward process (MRP), where the state transition kernel $P: \mathcal{S} \rightarrow \Delta(\cS)$ describes the probability of transitioning  between different states, and $r: \mathcal{S} \rightarrow [0,1]$ denotes the reward function (so that $r(s)$ is the immediate reward in state $s$). 
 The goal is to estimate the value function $V: \mathcal{S} \rightarrow \mathbb{R}$ from the trajectory $\{s_t,r_t\}_{t=0}^{\infty}$, which arises commonly in the task of policy evaluation for a given deterministic policy.

The Q-learning procedure in this special setting reduces to the well-known TD learning algorithm, which maintains an estimate $V_t: \mathcal{S} \rightarrow \mathbb{R}$ at each time $t$ and proceeds according to the following iterative update\footnote{When $\cA=\{a\}$ is a singleton,  the Q-learning update rule \eqref{eqn:q-learning} reduces to the TD update rule \eqref{eqn:td-learning} by relating $Q(s,a) = V(s)$. }
\begin{equation}
\label{eqn:td-learning}
\begin{split}
	V_t(s_{t-1}) &= (1- \eta_t ) V_{t-1}(s_{t-1} ) + \eta_t \left( r(s_{t-1}) + \gamma   V_{t-1}(s_t) \right),  \\
V_t(s) & = V_{t-1}(s ), \qquad \forall s \neq s_{t-1} .
\end{split}
\end{equation}
As usual, $\eta_t$ denotes the learning rate at time $t$, and $V_0$ is taken to be $0$.  
Consequently, our analysis for asynchronous Q-learning with a Markovian trajectory immediately leads to non-asymptotic $\ell_\infty$ guarantees for TD learning, stated below as a corollary of Theorem~\ref{thm:main-asyn-Q-learning}. A similar result can be stated in terms of the cover time as a corollary to Theorem~\ref{thm:main-asyn-Q-learning-cover-time}, which we omit for brevity.
\begin{corollary}[Asynchronous TD learning] \label{corollary:asynTD}
	Consider the TD learning algorithm~\eqref{eqn:td-learning}.
	There exist some universal constants $c_0,c_1>0$ such that for any $0<\delta<1$ and $0<\varepsilon \leq \frac{1}{1-\gamma}$,  one has
	\begin{align*}
		\forall s\in \cS: \qquad |V_T (s) - V (s)| \leq \varepsilon
	\end{align*}
	with probability at least $1-\delta$, provided that the iteration number $T$ and the learning rates $\eta_t \equiv \eta $ obey
	\begin{subequations}
	\begin{align}
	\label{eqn:numbersteps-TD}
		T & \geq \frac{c_{0}}{\mumin}\left\{ \frac{1}{(1-\gamma)^{5}\varepsilon^{2}}+\frac{\tmix}{1-\gamma}\right\} 
		\log\Big( \frac{|\mathcal{S} | T}{\delta} \Big) \log \Big(\frac{1}{(1-\gamma)^2\varepsilon} \Big),\\
		\eta & =\frac{c_{1}}{\log\big(\frac{|\mathcal{S}| T}{\delta}\big)}\min\left\{ \frac{(1-\gamma)^{4}\varepsilon^{2}}{\gamma^2},~\frac{1}{\tmix}\right\} .
	\end{align}
	\end{subequations}
\end{corollary}

The above result reveals that the $\ell_{\infty}$-sample complexity for TD learning  is at most
\begin{align}
	\widetilde{O}\Big( \frac{1}{\mumin (1-\gamma)^5 \varepsilon^2}  +  \frac{\tmix}{\mumin (1-\gamma)} \Big), 
	\label{eq:sample-size-TD-asynchronous}
\end{align}
provided that an appropriate constant learning rate is adopted. We note that prior finite-sample analysis on asynchronous TD learning typically focused on (weighted) $\ell_2$ estimation errors with linear function approximation \citep{bhandari2018finite,srikant2019finite}, and it is hence difficult to make fair comparisons.  The recent paper \cite{khamaru2020temporal} developed $\ell_{\infty}$ guarantees for TD learning, focusing on the synchronous settings with i.i.d.~samples rather than Markovian samples.  

% In addition, we point that in prior results \cite[Theorem 3]{bhandari2018finite},  

% We provide an immediate corollary that provides the first non-asymptotic $\ell_\infty$ guarantee to asynchronous TD learning over Markovian data by taking the action space as a singleton. 

\subsection{Adaptive and implementable learning rates}
\label{section:adaptive-learning-rates}

As alluded to previously, the learning rates recommended in \eqref{eq:learning-rate-asynQ} depend on the mixing time $\tmix$, 
a parameter that might be either {\em a priori} unknown or difficult to estimate. Fortunately, it is feasible to adopt a more adaptive learning rate schedule,
which does not rely on prior knowledge of $\tmix$ while still being capable of achieving the performance advertised in Theorem~\ref{thm:main-asyn-Q-learning}.

\paragraph{Learning rates.} In order to describe our new learning rate schedule, we need to keep track of the following quantities for all $(s,a)\in \cS\times \cA$: 
\begin{itemize}
	\item $K_t(s, a)$: the number of times that the sample trajectory visits $(s,a)$ during the first $t$ iterations.
\end{itemize}
In addition, we maintain an estimate $\widehat{\mu}_{\mathsf{min},t}$ of $\mumin$, computed recursively as follows
%\blue{%
\begin{equation} \label{eq:mu_est}
\widehat{\mu}_{\mathsf{min},t} = 
\left\{
\begin{array}{ll}
\frac{1}{|\mathcal{S}||\mathcal{A}|}, &\text{if }\min_{s, a} K_t(s, a) = 0; \\
\widehat{\mu}_{\mathsf{min},t-1}, &\text{if } \frac13 < \frac{\min_{s, a} K_t(s, a)/t}{\widehat{\mu}_{\mathsf{min},t-1}} < 3; \\
\min_{s, a} K_t(s, a)/t, \quad &\textrm{otherwise}.
\end{array}
\right.
\end{equation}
%}
%

With the above quantities in place, we propose the following learning rate schedule:  
\begin{equation} 
	\label{eq:learning-rate-implement}
	\eta_t = \min\Big\{1,  c_{\eta}\exp\Big(\Big\lfloor\log\frac{\log t}{\widehat{\mu}_{\mathsf{min},t}(1-\gamma)\gamma^2t}\Big\rfloor\Big)  \Big\} ,
	%\qquad\qquad (\text{an epoch-based choice})
\end{equation}
where $c_{\eta}>0$ is some universal constant independent of any MDP parameter\footnote{More precisely, $c_{\eta}>0$ can be 
any universal constant obeying $c_{\eta} \ge 74c_0c_1$ and $c_{\eta}>11$, with $c_0$ and $c_1$ being the universal constants stated in Theorem~\ref{thm:main-asyn-Q-learning}.}
and $\lfloor x \rfloor $ denotes the nearest integer less than or equal to $x$. 
If $\widehat{\mu}_{\mathsf{min},t}$ forms a reliable estimate of $\mumin$, then one can view \eqref{eq:learning-rate-implement} as a sort of ``piecewise constant approximation'' of the rescaled linear stepsizes $\frac{c_{\eta}\log t}{\mu_{\mathsf{min}}(1-\gamma)\gamma^2t}$; 
in fact, this can be viewed as a sort of ``doubling trick'' --- reducing the learning rate by a constant factor every once a while --- to approximate rescaled linear learning rates. 
Theorem~\ref{thm:main-asyn-Q-learning} can then be readily applied to analyze the performance for each constant segment of this learning rate schedule \eqref{eq:learning-rate-implement}. 
Noteworthily, such learning rates are fully data-driven and do no rely on any prior knowledge about the Markov chain (like $\tmix$ and $\mu_{\mathsf{min}}$) or the target accuracy level $\varepsilon$.

\paragraph{Performance guarantees.}
 Encouragingly, our theoretical framework can be readily extended without difficulty to accommodate this adaptive learning rate choice.
Specifically, for the  Q-function estimates
\begin{equation} \label{eq:Q_est}
\widehat{ Q}_t = 
\left\{
\begin{array}{ll}
	 Q_t, & \text{if }\eta_{t+1} \ne \eta_t, \\
\widehat{ Q}_{t-1}, \qquad &\text{otherwise},
\end{array}
\right.
\end{equation}
where $Q_t$ is provided by the Q-learning iterations (cf.~\eqref{eqn:q-learning}). 
We can then establish the following theoretical guarantees, whose proof is deferred to Section~\ref{sec:proof-thm:new-learning-rates-asyn-Q}.

\begin{theorem}
\label{thm:new-learning-rates-asyn-Q}
Consider  asynchronous Q-learning with learning rates \eqref{eq:learning-rate-implement} and the output \eqref{eq:Q_est}. 
There exists some universal constant $C>0$ such that: 
for any $0 < \delta < 1$ and $0 < \varepsilon \le \frac{1}{1-\gamma}$, one has
\begin{align}
	\forall (s,a)\in \cS\times \cA:\qquad	
	\big|\widehat{ Q}_T(s,a) -  Q^{\star} (s,a) \big| \le \varepsilon
\end{align}
with probability at least $1 - \delta$, provided that 
\begin{equation} 
	\label{eq:sample-new}
	T \ge \frac{C}{\gamma^2}\max \Big\{\frac{1}{\mu_{\mathsf{min}} (1-\gamma)^5 \varepsilon^2},  \frac{t_{\mathsf{mix}}}{\mu_{\mathsf{min}}(1-\gamma)} \Big\}
	\log\Big(\frac{|\mathcal{S}||\mathcal{A}|T}{\delta} \Big)\log \Big(\frac{T}{(1-\gamma)^2 \varepsilon}\Big) .
\end{equation}
%
%for some constant $C$.
%obeying $C > \frac{18c_{\eta}}{\gamma^2c_1}$. Here, $c_{\eta}$ is adopated in \eqref{eq:learning-rate-implement} and $c_1$ is the universal constant in ???
\end{theorem}
%
%Similarly, our theory for  variance-reduced Q-learning can also be extended to a stepsize that does not depend on $t_{\mathsf{mix}}$. More specifically, this requires two changes: (1) the epoch length needs to keep increasing (i.e.~at the end of every epoch, run $t_{\mathsf{epoch}} \leftarrow 2t_{\mathsf{epoch}}$); (2) set $\eta_t = \frac{c\log t_{\mathsf{epoch}}}{\widehat{\mu}_{\mathsf{min},t} (1-\gamma) t_{\mathsf{epoch}}}$.  This can be analyzed via a similar argument.
%
%In comparison to the prior work 

\begin{remark}
	%The above learning rate \eqref{eq:learning-rate-implement} can be approximately viewed as a sort ofs rescaled linear learning rate. 
	The interested reader might wonder whether our sample complexity guarantees continue to hold under the linear learning rate $\eta_t = \frac{1}{K_t(s_t,a_t)}$ --- a learning rate schedule that has been previously studied in \citet{tsitsiklis1994asynchronous,even2003learning}.  
	Nevertheless, as discussed in \citet[Section 3.3.1]{wainwright2019stochastic}, this linear learning rate can lead to 
	a sample complexity that scales exponentially in the effective horizon $\frac{1}{1-\gamma}$, which is clearly outperformed by a properly  rescaled linear learning rate. 
\end{remark}

\section{Extension: asynchronous variance-reduced Q-learning}
\label{sec:VR-Q}

As pointed out in prior literature, the classical form of Q-learning \eqref{eqn:q-learning} often suffers from sub-optimal dependence on the 
effective horizon $\frac{1}{1-\gamma}$. For instance, in the synchronous setting, the minimax lower bound is proportional to $\frac{1}{(1-\gamma)^3}$ (see,~\cite{azar2013minimax}), while the sharpest known upper bound for vanilla Q-learning scales as $\frac{1}{(1-\gamma)^5}$;  
see detailed discussions in \cite{wainwright2019stochastic}. To remedy this issue, recent work proposed to leverage the idea of variance reduction to develop accelerated RL algorithms in the synchronous setting \citep{sidford2018near,wainwright2019variance}, as   
inspired by the seminal SVRG algorithm \citep{johnson2013accelerating} that originates from the stochastic optimization literature. In this section, we adapt this idea to asynchronous Q-learning and characterize its sample efficiency.

\subsection{Algorithm}

In order to accelerate the convergence, it is instrumental to reduce the variability of the empirical Bellman operator $\mathcal{T}_t$ employed in the update rule \eqref{eqn:q-learning} of classical Q-learning. This can be achieved via the following means. Simply put, assuming we have access to \emph{(i)} a reference $Q$-function estimate, denoted by $\overline{Q}$, and \emph{(ii)} an estimate of $\mathcal{T}(\overline{Q})$, denoted by $\widetilde{\mathcal{T}}(\overline{Q})$, 
%, which is estimated using a set of data samples (to be explained momentarily).  
the variance-reduced Q-learning update rule is given by
\begin{equation}
\label{eqn:vr-q-learning}
\begin{split}
	Q_t(s_{t-1},a_{t-1}) &= (1- \eta_t ) Q_{t-1}(s_{t-1},a_{t-1}) + 
	\eta_t \Big(  \mathcal{T}_t (Q_{t-1}) - \mathcal{T}_t (\overline{Q})  
	+\widetilde{\mathcal{T}}(\overline{Q}) \Big)  (s_{t-1},a_{t-1}),  \\
	Q_t(s ,a) & = Q_{t-1}(s ,a ), \qquad \forall (s,a)\neq (s_{t-1},a_{t-1}), 
\end{split}
\end{equation}
where $\mathcal{T}_t$ denotes the empirical Bellman operator at time $t$ (cf.~\eqref{defn:empirical-Bellman-t}). The empirical estimate $\widetilde{\mathcal{T}}(\overline{Q})$ can be computed using a set of samples; more specifically, by drawing $N$ consecutive sample transitions $\{ (s_{i},a_{i}, s_{i+1}) \}_{0\leq i< N}$ from the observed trajectory,  we compute
\begin{equation} \label{eq:surrogate_TQ}
\widetilde{\mathcal{T}}(\overline{Q}) (s,a) = r(s,a) + \frac{\gamma \sum_{i=0}^{N-1}   \ind\{(s_i,a_i)=(s,a) \}\max_{a'} \overline{Q}(s_{i+1},a') }{\sum_{i=0}^{N-1} \ind\{(s_i,a_i)=(s,a)\} } .
\end{equation}
Compared with the classical form~\eqref{eqn:q-learning}, the original update term $\mathcal{T}_t(Q_{t-1})$ has been replaced by $ \mathcal{T}_t (Q_{t-1}) - \mathcal{T}_t (\overline{Q})  
	+\widetilde{\mathcal{T}}(\overline{Q}) $, in the hope of achieving reduced variance  as long as $\overline{Q}$ (which serves as a proxy to $Q^{\star}$) is chosen properly.
%
%w.r.t.~the $t$th sample, that is, 
%%
%\begin{align}
%	\mathcal{T}_t (Q) (s_{t-1}, a_{t-1}) := 
%		r(s_{t-1}, a_{t-1}) + \gamma \max_{a' \in \cA} Q(s_t, a').
%		%, \quad & \text{if }(s,a)=(s_{t-1}, a_{t-1}) .
%	\label{defn:empirical-Bellman-t}
%\end{align}
%
%

We now take a moment to elucidate the rationale behind the variance-reduced update rule \eqref{eqn:vr-q-learning}. 
In the vanilla Q-learning update rule \eqref{eqn:q-learning}, the variability in each iteration (conditional on the past) comes primarily from the stochastic term $ \mathcal{T}_t (Q_{t-1})$. 
In order to accelerate convergence, it is advisable to reduce the variability of this term. 
Suppose now that we have access to a reference point $\overline{Q}$ that is close to $Q_{t-1}$. By replacing  $\mathcal{T}_t (Q_{t-1})$ with 
\[
	\Big\{ \mathcal{T}_t (Q_{t-1}) - \mathcal{T}_t (\overline{Q}) \Big\}
	+\widetilde{\mathcal{T}}(\overline{Q}) , 
\]
we see that the variability of the first term $\mathcal{T}_t (Q_{t-1}) - \mathcal{T}_t (\overline{Q})$ can be small if $Q_{t-1}\approx \overline{Q}$, 
while the uncertainty of the second term $\widetilde{\mathcal{T}}(\overline{Q})$ can also be well controlled via the use of batch data. 
Motivated by this simple idea, the variance-reduced Q-learning rule attempts to operate in an epoch-based manner, 
computing $\widetilde{\mathcal{T}} (\overline{Q})$ once every epoch (so as not to increase the overall sampling burden) and leveraging it to help reduce variability. 
\begin{figure}[H]
\begin{center}
	\includegraphics[width=0.7\textwidth]{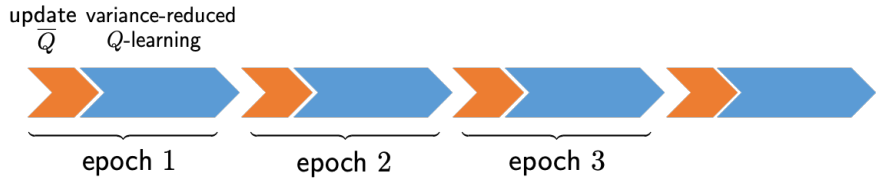}
\end{center}
\caption{A pictorial illustration of variance-reduced Q-learning.}
\end{figure}

%\smallskip
For convenience of presentation, we introduce the following notation
\begin{align} 
\label{eq:asyn_vrq}
Q = \vrq (\, \overline{Q}, N, \tepoch \,)
\end{align}
 to represent the above-mentioned update rule, 
which starts with a reference point $\overline{Q}$ and operates upon a total number of $N+\tepoch$ consecutive sample transitions. The first $N$ samples are employed to 
construct $\widetilde{\mathcal{T}}(\overline{Q})$ via \eqref{eq:surrogate_TQ}, with the remaining samples employed in $\tepoch$ iterative updates \eqref{eqn:vr-q-learning}; see Algorithm~\ref{alg:async-vr-q-block}.  
To achieve the desired acceleration, the proxy $\overline{Q}$ needs to be periodically updated so as to better approximate the truth $Q^{\star}$ and hence 
 reduce the bias. 
%Given that computing a reasonable $\widetilde{\mathcal{T}}$ (and hence $\widetilde{\mathcal{T}}(\overline{Q})$) could be expensive, 
It is thus natural to run the algorithm in a multi-epoch manner. Specifically, we divide the samples into contiguous subsets called epochs, each containing $\tepoch$ iterations and using $N+\tepoch$ samples. We then proceed as follows
\begin{align}\label{eqn:vr-q-learning-epoch}
	Q_{m}^{\mathsf{epoch}}  = \vrq(\, Q_{m-1}^{\mathsf{epoch}}, N , \tepoch \,), \quad m= 1, \ldots, M,
\end{align}
where $M$ is the total number of epochs, and $Q_m^{\mathsf{epoch}}$ denotes the output of the $m$-th epoch. 
% 
% and $Q_0$ is initialized to 0. 
The whole procedure is summarized in Algorithm~\ref{alg:async-vr-q}. Clearly, the total number of samples used in this algorithm is given by $M(N+\tepoch)$. We remark that the idea of performing variance reduction in RL is certainly not new, and has been explored in a number of recent works \citep{du2017stochastic, wainwright2019variance,khamaru2020temporal,sidford2018near,sidford2018variance,xu2020reanalysis}.

%whereas the total computational cost is on the same order as the sample size.

%We are ready to define the proposed asynchronous variance-reduced Q-learning algorithm, inspired by \cite{wainwright2019variance,khamaru2020temporal}, which is composed of two stages: \yjc{needs to be refined..}
%
%\begin{itemize}
%\item \textit{Step 1: coarse estimate.} Run \eqref{eqn:vr-q-learning-epoch} with initialization $Q_0=0$ for $M_1$ epochs. Call the output $\widehat{Q}$.
%
%\item \textit{Step 2: refinement.} Run \eqref{eqn:vr-q-learning-epoch} with initialization $\widehat{Q}$ for $M_2$ epochs.
%
%\end{itemize}
 
 %a reference Q-function $\overline{Q}$, epoch length $\tepoch$, recentering length $N$;

\subsection{Theoretical guarantees for variance-reduced Q-learning}

This subsection develops a non-asymptotic sample complexity bound for asynchronous variance-reduced Q-learning on a single trajectory. Before presenting our theoretical guarantees, there are several algorithmic parameters that we shall specify; for given target levels $(\varepsilon, \delta)$, choose
	\begin{subequations}\label{eqn:alg-parameters}
	\begin{align}
		\eta_t \equiv \eta & =
		\frac{c_{0}}{\log\big(\frac{|\mathcal{S}||\mathcal{A}|\tepoch}{\delta}\big)}\min\left\{ \frac{(1-\gamma)^{2} }{\gamma^2},~\frac{1}{\tmix}\right\} ,  \label{eq:learning-rate-VRQ}\\
		N & \geq  
		\frac{c_1}{\mumin} 
		\Big(\frac{1}{(1-\gamma)^{3} \min\{1, \varepsilon^2\}} + \tmix  \Big)
		\log\Big(\frac{|\mathcal{S}||\mathcal{A}|\tepoch}{\delta}\Big),
		\\
		\tepoch & \geq \frac{c_{2}}{\mumin}
		\Big(\frac{1}{(1-\gamma)^{3}}+\frac{\tmix}{1-\gamma}\Big) 
		\log \Big(\frac{1}{(1-\gamma)^2\varepsilon} \Big)
		\log\Big( \frac{|\mathcal{S}||\mathcal{A}|\tepoch}{\delta} \Big) , 
	\end{align}
	\end{subequations}
where $c_0>0$ is some sufficiently small constant,  $c_1,c_2>0$ are some sufficiently large constants, and we recall the definitions of $\mumin$ and $\tmix$ in \eqref{defn:mu-min} and \eqref{defn:mixing-time-MC}, respectively. Note that the learning rate \eqref{eq:learning-rate-VRQ}  chosen here could be larger than the choice \eqref{eq:learning-rate-asynQ} for the classical form by a factor of $O\big(\frac{1}{(1-\gamma)^2}\big)$ (which happens if $\tmix$ is not too large), allowing the algorithm to progress more aggressively.

\begin{theorem}[Asynchronous variance-reduced Q-learning]
	\label{thm:main-asyn-VR-Q-learning}
	%Consider the asynchronous variance-reduced Q-learning algorithm in \eqref{eqn:vr-q-learning-epoch}.

	Let ${Q}_{M}^{\mathsf{epoch}}$ be the output of Algorithm~\ref{alg:async-vr-q} with parameters chosen according to \eqref{eqn:alg-parameters}. 
	There exists some constant $c_3>0$ such that for any $0<\delta<1$ and $0<\varepsilon \leq \frac{1}{1-\gamma}$,   one has
	\begin{align*}
		\forall (s,a)\in \cS\times \cA: \qquad |{Q}_{M}^{\mathsf{epoch}} (s,a) - Q^\star(s,a)| \leq \varepsilon
	\end{align*}
	with probability at least $1-\delta$, provided that the total number of epochs exceeds 
	% \yjc{to be updated with more parameters.}
	%\begin{subequations}
	%
	\begin{align}
	\label{eqn:numbersteps-vr}
		M &\geq  c_3 \log \frac{1}{\varepsilon (1-\gamma)^2}  . 
	\end{align}
	%
	% for some sufficiently large constant $c_3>0$. 
	%
	%\yutingcomment{double check! seems like we didn't use $\varepsilon < 1$} \yjc{aglee; the optimality holds when $\varepsilon <1$.}
\end{theorem}
\noindent The proof of this result is postponed to Section~\ref{Sec:Asynchronous-VR}.

	In view of Theorem~\ref{thm:main-asyn-VR-Q-learning}, the $\ell_{\infty}$-based sample complexity for variance-reduced Q-learning to yield $\varepsilon$ accuracy --- which is characterized by $M(N+\tepoch)$ --- can be as low as 
\begin{align}
	\label{eq:sample-size-VRQ}
	 \widetilde{O}  \Big(\frac{1}{\mumin(1-\gamma)^{3}\min\{1,\varepsilon^{2}\}} +\frac{\tmix}{\mumin(1-\gamma)}\Big). 
\end{align}
	Except for the second term that depends on the mixing time, the first term matches the result of \cite{wainwright2019variance} derived for the synchronous settings with independent samples.  In the range $\varepsilon \in (0, \min\{1,\frac{1}{(1-\gamma)\sqrt{\tmix}}\}]$, the sample complexity reduce to $\widetilde{O}\big(\frac{1}{\mumin(1-\gamma)^{3}\varepsilon^{2}}\big)$; the scaling $\frac{1}{(1-\gamma)^3}$  matches the minimax lower bound derived in \cite{azar2013minimax} for the synchronous setting.

Once again, we can immediately deduce guarantees for asynchronous variance-reduced TD learning by reducing the action space to a singleton set (akin to Section~\ref{sec:TD-connection}), which extends the analysis \cite{khamaru2020temporal} to Markovian noise. 
In addition, similar to Section~\ref{section:adaptive-learning-rates}, we can also employ adaptive learning rates in variance-reduced Q-learning --- which do not require prior knowledge of  $\tmix$ and $\mumin$ --- without compromising the sample complexity. 
For the sake of brevity, we omit these extensions in the current paper. 
%here as it is not the main focus of the current paper. 

\begin{comment}
\begin{algorithm}[t]
\caption{Asynchronous variance-reduced Q-learning} \label{alg:async-vr-q} 
\begin{algorithmic} [1]
	\STATE \textbf{Input parameters:} number of epochs $M$, epoch length $\tepoch$, recentering length $N$, learning rates $\{\eta_t\}$.
\STATE \textbf{Initialization:} $Q_0=\overline{Q}=0$.
\FOR{each epoch $m=1, \cdots, M$}
\STATE Use the next $N$ samples from the sample trajectory to calculate $\widetilde{T}(\overline{Q})$ according to \eqref{eq:surrogate_TQ};
%\STATE Denote by $s_0$ the current state of MDP;
\FOR{$t=1,2,\cdots,\tepoch$} 
	\STATE Draw an action $a_{t-1} \sim \pi_{\mathsf{b}}(\cdot| s_{t-1})$ and the next state $s_{t}\sim P(\cdot|s_{t-1},a_{t-1})$ from the MDP;
\STATE{Update $Q_t$ according to \eqref{eqn:vr-q-learning}.}

\ENDFOR 
\STATE Set $\overline{Q} \leftarrow Q_{\tepoch}$, ${Q}_0 \leftarrow Q_{\tepoch}$, and reset $s_0$ to be the current state.
\ENDFOR
\end{algorithmic} 
\end{algorithm}
\end{comment}

\begin{algorithm}[t]
\DontPrintSemicolon
  \textbf{input parameters:} number of epochs $M$, epoch length $\tepoch$, recentering length $N$, learning rate $\eta$. \\
  \textbf{initialization:} set $Q_{0}^{\mathsf{epoch}}\leftarrow 0$. \\

    \For{$\mathrm{each~epoch~}m=1, \cdots, M$ }
    { 
	     % %\tcc{Compute the reference point.}
	     % Draw $N$ new consecutive samples and compute $\widetilde{\mathcal{T}}(\overline{Q})$ according to \eqref{eq:surrogate_TQ}.   \\
	     % %\tcc{Variance-reduced Q-learning updates.}
	     % Set $s_0 \leftarrow$  current state. \\
	     % \For{$t=1,2,\cdots,\tepoch$}
	     % {
	     %	     Draw action $a_{t-1} \sim \pi_{\mathsf{b}}(s_{t-1})$ and next state $s_{t}\sim P(\cdot|s_{t-1},a_{t-1})$. \\
	     %	Update $Q_t$ according to \eqref{eqn:vr-q-learning}.
	     % }
	     \tcc{Call Algorithm~\ref{alg:async-vr-q-block}.}
	     $Q_{m}^{\mathsf{epoch}}$ = \vrq(\,$Q_{m-1}^{\mathsf{epoch}}, N, \tepoch$) . \\
	     % Set $\overline{Q} \leftarrow Q_{m}^{\mathsf{epoch}}$. 
    }

    % \textbf{return:} ${Q}_{\mathsf{vrq}} \leftarrow Q_{\tepoch}$. 
    \caption{Asynchronous variance-reduced Q-learning}
 \label{alg:async-vr-q}
\end{algorithm}

\begin{algorithm}[t]
\DontPrintSemicolon
%  \textbf{input parameters:} number of epochs $M$, epoch length $\tepoch$, recentering length $N$, learning rate $\eta$. \\
%  \textbf{initialization:} $Q_0=\overline{Q}=0$. \\

     	     Draw $N$ new consecutive samples from the sample trajectory; compute $\widetilde{\mathcal{T}}(\overline{Q})$ according to \eqref{eq:surrogate_TQ}.   \\
	     %\tcc{Variance-reduced Q-learning updates.}
	     Set $s_0 \leftarrow$  current state, and $Q_0 \leftarrow \overline{Q}$. \\
	     \For{$t=1,2,\cdots,\tepoch$}
	     {
		     Draw action $a_{t-1} \sim \pi_{\mathsf{b}}(s_{t-1})$, observe reward $r(s_{t-1},a_{t-1})$, and draw next state $s_{t}\sim P(\cdot\,|\,s_{t-1},a_{t-1})$. \\
	     	Update $Q_t$ according to \eqref{eqn:vr-q-learning}.
	     }
	     % Set $\overline{Q} \leftarrow Q_{\tepoch}$ and ${Q}_0 \leftarrow Q_{\tepoch}$. 

    \textbf{return:} ${Q} \leftarrow Q_{\tepoch}$. 
	\caption{ $\mathsf{function}$ $Q = \vrq(\,\overline{Q}, N, \tepoch)$}
 \label{alg:async-vr-q-block}
\end{algorithm}

\section{Related works} 
\label{sec:related-work}

In this section, we review several recent lines of works and compare our results with them.

%Classical analyses of model-free reinforcement learning algorithms have largely focused on asymptotic performance (e.g.~\cite{tsitsiklis1997analysis,szepesvari1998asymptotic,tsitsiklis1994asynchronous,jaakkola1994convergence}). More recently, papers have shifted attention towards understanding the performance in the non-asymptotic and finite-time settings. 
%
%The finite-time analysis of temporal difference learning is provided in \cite{bhandari2018finite,lakshminarayanan2018linear,srikant2019finite,xu2019two,mou2020linear,khamaru2020temporal,kaledin2020finite}. Most of these analyses focus on the $\ell_2$ error.
%
%A highly incomplete list includes \cite{bhandari2018finite,lakshminarayanan2018linear,srikant2019finite,xu2019two,mou2020linear,khamaru2020temporal,kaledin2020finite,kearns1999finite,bradtke1996linear,strehl2006pac,wainwright2019variance,gupta2019finite,even2003learning,xu2019finite,cai2019neural,azar2017minimax,jin2018q,shah2018q,sidford2018near,qu2020finite}, a large fraction of which is concerned with model-free algorithms.

\paragraph{The Q-learning algorithm and its variants.} The Q-learning algorithm, originally proposed in \cite{watkins1989learning}, has been analyzed in the asymptotic regime  by \cite{tsitsiklis1994asynchronous,szepesvari1998asymptotic,jaakkola1994convergence,borkar2000ode} since more than two decades ago.  Additionally,  
finite-time performance of Q-learning and its variants have been analyzed by \cite{even2003learning,kearns1999finite,beck2012error,chen2020finite,wainwright2019stochastic,qu2020finite,li2021tightening,xiong2020finitetime,li2021q} in the tabular setting, by \cite{chen2019performance,xu2019finite,yang2019theoretical,bhandari2018finite,du2019provably,du2020agnostic,cai2019neural,yang2019sample,weng2020momentum,weng2020provably} in the context of function approximations, and by \cite{shah2018q} with nonparametric regression. In addition, \cite{sidford2018near,wainwright2019variance,strehl2006pac,ghavamzadeh2011speedy,azar2011reinforcement,devraj2020q} studied modified Q-learning algorithms that might potentially improve sample complexities and accelerate convergence. Another line of work studied Q-learning with sophisticated exploration strategies such as UCB exploration (e.g.~\cite{jin2018q,Wang2020Qlearning,bai2019provably}), which is beyond the scope of the current work.

\paragraph{Finite-sample $\ell_{\infty}$ guarantees for Q-learning.}
We now expand on non-asymptotic $\ell_{\infty}$ guarantees available in prior literature, which are the most relevant to the current work. 
An interesting aspect that we shall highlight is the importance of learning rates. For instance, when a linear learning rate (i.e.~$\eta_t=1/t$) is adopted, the sample complexity results derived in past works \citep{even2003learning,szepesvari1998asymptotic} exhibit an exponential blow-up in $\frac{1}{1-\gamma}$, which is clearly undesirable.   In the synchronous setting, 
 \cite{even2003learning,beck2012error,wainwright2019stochastic,chen2020finite} studied the finite-sample complexity of Q-learning under various learning rate rules; the best sample complexity known to date is $\widetilde{O}\big(\frac{|\cS||\cA|}{(1-\gamma)^5\varepsilon^2}\big)$, achieved via either a rescaled linear learning rate \citep{wainwright2019stochastic,chen2020finite} or a constant learning rate \citep{chen2020finite}. When it comes to asynchronous Q-learning (in its classical form),
our work provides the first analysis that achieves linear scaling with $1/\mumin$ or  $\tcovertime$; see Table~\ref{table:comparison_asyncQ} for detailed comparisons. 
% Notably, our result matches with the sample complexity of synchronous Q-learning with respect to $(1-\gamma)$ and $\varepsilon$, and further captures the effect of the Markov chain induced by the behavior policy. 
% 
Going beyond classical Q-learning, the speedy Q-learning algorithm, which adds a momentum term in the update by using previous Q-function estimates, provably achieves a sample complexity of $\widetilde{O}\big(\frac{\tcovertime}{(1-\gamma)^4\varepsilon^2}\big)$  \citep{azar2011reinforcement} in the asynchronous setting, whose update rule takes twice the storage of classical Q-learning.
% It is known that speedy Q-learning achieves an improved sample complexity which scales as $\frac{1}{(1-\gamma)^4}$ on the effective horizon. 
However, the proof idea adopted in the speedy Q-learning paper relies heavily on the specific update rules of speedy Q-learning, which cannot be readily used here to help improve the sample complexity of asynchronous Q-learning in terms of its dependency on $\frac{1}{1-\gamma}$. 
 In comparison, our analysis of the variance-reduced Q-learning algorithm achieves a sample complexity of $\widetilde{O}\big(\frac{1}{\mumin(1-\gamma)^3\varepsilon^2} + \frac{\tmix}{\mumin (1-\gamma)}\big)$ when $\varepsilon<1$.

%\begin{itemize}
%\item {\em Synchronous Q-learning.} The $\ell_\infty$ guarantee of synchronous Q-learning is studied in \cite{even2003learning,beck2012error,wainwright2019stochastic,chen2020finite}.
%
%\item {\em Asynchronous Q-learning.} The $\ell_\infty$ guarantee of asynchronous Q-learning is studied in \cite{even2003learning,beck2012error,qu2020finite}.
%
%\end{itemize}

\paragraph{Finite-sample guarantees for model-free algorithms.} Convergence properties of several model-free RL algorithms have been studied recently in the presence of Markovian data, including but not limited to TD learning and its variants \citep{bhandari2018finite,xu2019two,srikant2019finite,gupta2019finite,doan2019finite,kaledin2020finite,dalal2018finite_two,dalal2018finite,lee2019target,lin2020distributed,mou2020linear,xu2020reanalysis}, Q-learning \citep{chen2019performance,xu2019finite}, and SARSA \citep{zou2019finite}. However, these recent papers typically focused on the (weighted) $\ell_2$ error rather than the $\ell_\infty$ risk, where the latter is often more relevant in the context of RL. 
In addition, \cite{khamaru2020temporal} investigated the $\ell_\infty$ bounds of (variance-reduced) TD learning, although they did not account for Markovian noise.

 \paragraph{Finite-sample guarantees for model-based algorithms.} 
Another contrasting approach for learning the optimal Q-function is the class of model-based algorithms, which has been shown to enjoy minimax-optimal sample complexity in the synchronous setting. More precisely, it is known that by planning over an empirical MDP constructed from $\widetilde{O}\big(\frac{|\cS||\cA|}{(1-\gamma)^3\varepsilon^2}\big)$ samples, we are guaranteed to find not only an $\varepsilon$-optimal Q-function but also an $\varepsilon$-optimal policy \citep{azar2013minimax,agarwal2019optimality,li2020breaking}. It is worth emphasizing that the minimax optimality of model-based approach has been shown to  hold for the entire $\varepsilon$-range; in comparison, the sample optimality of the model-free approach has only been shown for a smaller range of accuracy level $\varepsilon$ in the synchronous setting.  We also remark that existing sample complexity analysis for model-based approaches might be generalizable to Markovian data.

\section{Analysis of asynchronous Q-learning}
\label{Sec:Analysis-asynchronous}

%\paragraph{Matrix notation.}

This section is devoted to establishing Theorem~\ref{thm:main-asyn-Q-learning}. 
Before proceeding, we find it convenient to introduce some matrix
notation. Let $\bm{\Lambda}_{t}\in\mathbb{R}^{|\mathcal{S}||\mathcal{A}|\times|\mathcal{S}||\mathcal{A}|}$
be a diagonal matrix obeying
\begin{equation}
\bm{\Lambda}_{t}\big((s,a),(s,a)\big):=
\begin{cases}
\eta,\quad & \text{if }(s,a)=(s_{t-1},a_{t-1}),\\
0, & \text{otherwise},
\end{cases}
\label{eq:defn-Lambda-t}
\end{equation}
where $\eta>0$ is the learning rate. In addition, we use the vector
$\bm{Q}_{t}\in\mathbb{R}^{|\mathcal{S}||\mathcal{A}|}$ (resp.~$\bm{V}_{t}\in\mathbb{R}^{|\mathcal{S}|}$)
to represent our estimate $Q_{t}$ (resp.~$V_{t}$) in the
$t$-th iteration, so that the $(s,a)$-th (resp.~$s$th) entry of $\bm{Q}_{t}$ (resp.~$\bm{V}_t$) is given by $Q_t(s,a)$ (resp.~$V_t(s)$). 
Similarly,  let the vectors $\bm{Q}^{\star}\in\mathbb{R}^{|\mathcal{S}||\mathcal{A}|}$
and $\bm{V}^{\star}\in\mathbb{R}^{|\mathcal{S}|}$ represent the optimal Q-function
$Q^{\star}$ and the optimal value function $V^{\star}$, respectively. 
We also let the vector $\bm{r}\in\mathbb{R}^{|\mathcal{S}||\mathcal{A}|}$
stand for the reward function $r$, so that the $(s,a)$-th entry of $\bm{r}$ is given by $r(s,a)$. In addition, we define the matrix
$\bm{P}_{t}\in\{0,1\}^{|\mathcal{S}||\mathcal{A}|\times|\mathcal{S}|}$ such that
\begin{equation}
\bm{P}_{t}\big((s,a),s'\big):=\begin{cases}
1, & \text{if }(s,a,s')=(s_{t-1},a_{t-1},s_{t}),\\
0, & \text{otherwise}.
\end{cases}\label{eq:defn-Pt}
\end{equation}
Clearly, this set of notation allows us to express the Q-learning
update rule \eqref{eqn:q-learning} in the following matrix form
\begin{equation}
\bm{Q}_{t}=\big(\bm{I}-\bm{\Lambda}_{t}\big)\bm{Q}_{t-1}+\bm{\Lambda}_{t}\big(\bm{r}+\gamma\bm{P}_{t}\bm{V}_{t-1}\big).\label{eq:Q-learning-matrix-notation}
\end{equation}

\subsection{Error decay in the presence of constant learning rates}

The main step of the analysis is to establish the following result concerning the dynamics of asynchronous Q-learning. In order to state it formally, we find it convenient to introduce several auxiliary quantities
\begin{subequations}
\begin{align}
	\tcover &:=   \frac{443\tmix}{\mumin}\log\Big(\frac{4|\mathcal{S}||\mathcal{A}|T}{\delta}\Big), \label{defn:tepoch} \\
	\tth &:= \max\Bigg\{ \frac{2 \log\frac{1}{(1-\gamma)^2\varepsilon}}{ \eta \mumin },\: \tcover \Bigg\} , \label{eqn:tth} \\
		\muepo &\defn \frac{1}{2}\mumin\tcover, \label{eqn:muepo} \\  
	\rho &\defn (1-\gamma)\big(1-(1-\eta)^{\muepo} \big). \label{eqn:rho}  
\end{align}
\end{subequations}
%
%Defining the estimation error as $\bDel_t:=\bm{Q}_t - \bm{Q}^{\star}$, we have the following result. 
With these quantities in mind, we have the following result. 
\begin{theorem}
\label{thm:intermediate}
Consider the asynchronous Q-learning algorithm~in Algorithm~\ref{alg:async-q} with $\eta_t\equiv \eta$. 
For any $\delta \in (0,1)$ and any $\varepsilon \in (0, \frac{1}{1-\gamma}]$,  there exists a universal constant $c > 0$ such that with probability at least $1 - 6 \delta$, 
the following relation holds uniformly for all $t \leq T$ (defined in \eqref{eqn:numbersteps})
\begin{align}
\label{eqn:beethoven}
	\| \bm{Q}_t - \bm{Q}^{\star} \|_{\infty} \leq (1- \rho )^{k} \frac{\| \bm{Q}_0 - \bm{Q}^{\star} \|_{\infty}}{1-\gamma} + \frac{c \gamma}{1-\gamma} 
	 \|\bm{V}^{\star}\|_{\infty} \sqrt{\eta\log\Big(\frac{|\mathcal{S}||\mathcal{A}|T}{\delta}\Big)} 
	 +\varepsilon, 
	% \qquad k = \Big\lfloor \frac{t - \tth}{\tcover} \Big\rfloor.
\end{align}
provided that $0<\eta\log\big(\frac{|\mathcal{S}||\mathcal{A}|T}{\delta}\big)<1$. 
Here, we define $k \defn \max \big\{0, ~\big\lfloor \frac{t - \tth}{\tcover} \big\rfloor \big\}$.
\end{theorem}
In words, Theorem~\ref{thm:intermediate} asserts that the $\ell_{\infty}$ estimation error decays linearly --- in a blockwise manner --- to some error floor that scales with $\sqrt{\eta}$. This result suggests how to set the learning rate based on the target accuracy level, which in turn allows us to pin down the sample complexity under consideration. In what follows, we shall first establish Theorem~\ref{thm:intermediate}, and then return to prove Theorem~\ref{thm:main-asyn-Q-learning} using this result.

Before embarking on the proof of Theorem~\ref{thm:intermediate}, we would like to point out a few key technical ingredients: 
(i) an epoch-based analysis that focuses on macroscopic dynamics as opposed to per-iteration dynamics, 
(ii) measure concentration of Markov chains (see Section~\ref{sec:concentration-MC}) that helps reveal the similarity between epoch-based dynamics and the synchronous counterpart, and (iii) careful analysis of recursive relations. 
These key ingredients taken collectively lead to a sample complexity bound  that improves upon prior analysis in \citet{qu2020finite}.

\subsection{Proof of Theorem~\ref{thm:intermediate}}

We are now positioned to outline the proof of Theorem~\ref{thm:intermediate}. 
We remind the reader that for any two vectors $\bm{z}=[z_i]$ and $\bm{w}=[w_i]$, 
the notation $\bm{z}\leq\bm{w}$ (resp.~$\bm{z}\geq\bm{w}$) denotes entrywise comparison (cf.~Section~\ref{sec:intro}), meaning that $z_i\leq w_i$ (resp.~$z_i\geq w_i$) holds for all $i$. As a result, for any non-negative matrix $\bm{A}$, one has $\bm{A}\bm{z}\leq \bm{A}\bm{w}$ as long as $\bm{z}\leq \bm{w}$.  

\subsubsection{Key decomposition and a recursive formula}
The starting point of our proof is the following elementary decomposition
\begin{align}
\bm{\Delta}_{t} := \bm{Q}_{t}-\bm{Q}^{\star} & =\big(\bm{I}-\bm{\Lambda}_{t}\big)\bm{Q}_{t-1}+\bm{\Lambda}_{t}\big(\bm{r}+\gamma\bm{P}_{t}\bm{V}_{t-1}\big)-\bm{Q}^{\star}\nonumber \\
 & =\big(\bm{I}-\bm{\Lambda}_{t}\big)\big(\bm{Q}_{t-1}-\bm{Q}^{\star}\big)+\bm{\Lambda}_{t}\big(\bm{r}+\gamma\bm{P}_{t}\bm{V}_{t-1}-\bm{Q}^{\star}\big)\nonumber \\
 & =\big(\bm{I}-\bm{\Lambda}_{t}\big)\big(\bm{Q}_{t-1}-\bm{Q}^{\star}\big)+\gamma\bm{\Lambda}_{t}\big(\bm{P}_{t}\bm{V}_{t-1}-\bm{P}\bm{V}^{\star}\big)\nonumber \\
 & =\big(\bm{I}-\bm{\Lambda}_{t}\big)\bm{\Delta}_{t-1}+\gamma\bm{\Lambda}_{t}\big(\bm{P}_{t}-\bm{P}\big)\bm{V}^{\star}+\gamma\bm{\Lambda}_{t}\bm{P}_{t}\big(\bm{V}_{t-1}-\bm{V}^{\star}\big)\label{eq:Delta-t-identity2}
\end{align}
for any $t>0$, where the first line results from the update rule \eqref{eq:Q-learning-matrix-notation}, and the penultimate line follows from the Bellman equation $\bm{Q}^{\star}=\bm{r}+\gamma\bm{P}\bm{V}^{\star}$ (see~\cite{bertsekas2017dynamic}). 
%By defining
%
% \begin{align}
% 	\bm{\Delta}_{t} & :=\bm{Q}_{t}-\bm{Q}^{\star} ,
%	\label{eq:defn-Deltat-xit}
% \end{align}
%
% we can rewrite (\ref{eq:Delta-t-identity2}) as follows
%
% \begin{align}
% \bm{\Delta}_{t} & =\big(\bm{I}-\bm{\Lambda}_{t}\big)\bm{\Delta}_{t-1}+\gamma\bm{\Lambda}_{t}\big(\bm{P}_{t}-\bm{P}\big)\bm{V}^{\star}+\gamma\bm{\Lambda}_{t}\bm{P}_{t}\big(\bm{V}_{t-1}-\bm{V}^{\star}\big).
% \end{align}
%
Applying this relation recursively gives
\begin{equation}
	\bm{\Delta}_{t}=
	\underset{=:\bm{\beta}_{1,t}}{\underbrace{\gamma\sum_{i=1}^{t}\prod_{j=i+1}^{t}\big(\bm{I}-\bm{\Lambda}_{j}\big)\bm{\Lambda}_{i}\big(\bm{P}_{i}-\bm{P}\big)\bm{V}^{\star}}}
	+ \underset{=:\bm{\beta}_{2,t}}{\underbrace{\gamma\sum_{i=1}^{t}\prod_{j=i+1}^{t}\big(\bm{I}-\bm{\Lambda}_{j}\big)\bm{\Lambda}_{i}\bm{P}_{i}\big(\bm{V}_{i-1}-\bm{V}^{\star}\big)}}+
	\underset{=:\bm{\beta}_{3,t}}{\underbrace{\prod_{j=1}^{t}\big(\bm{I}-\bm{\Lambda}_{j}\big)\bm{\Delta}_{0}}}.
	\label{eq:defn-beta1-beta3}
\end{equation}
Applying the triangle inequality, we obtain
\begin{align}
|\bm{\Delta}_{t}|
	\leq |\bm{\beta}_{1,t} |+ | \bm{\beta}_{2,t} |+ |\bm{\beta}_{3,t}|,
	\label{eq:Delta-t-decomposition}
\end{align}
where we recall the notation $|\bm{z}|:=[|z_{i}|]_{1\leq i\leq n}$
for any vector $\bm{z}=[z_{i}]_{1\leq i\leq n}$. 
In what follows, we shall look at these terms separately. 
\begin{itemize}
\item
First of all, given that $\bm{I}-\bm{\Lambda}_{j}$ and $\bm{\Lambda}_{j}$ are
both non-negative diagonal matrices and that
\[
\big\|\bm{P}_{i}\big(\bm{V}_{i-1}-\bm{V}^{\star}\big)\big\|_{\infty}\leq\|\bm{P}_{i}\|_{1}\|\bm{V}_{i-1}-\bm{V}^{\star}\|_{\infty}=\|\bm{V}_{i-1}-\bm{V}^{\star}\|_{\infty}\leq\|\bm{Q}_{i-1}-\bm{Q}^{\star}\|_{\infty}=\|\bm{\Delta}_{i-1}\|_{\infty},
\]
we can easily see that
\begin{equation}
	\big|\bm{\beta}_{2,t}\big|\leq\gamma \sum_{i=1}^{t} \|\bm{\Delta}_{i-1}\|_{\infty} \prod_{j=i+1}^{t}\big(\bm{I}-\bm{\Lambda}_{j}\big)\bm{\Lambda}_{i}\bm{1} .
	\label{eq:defn-beta4}
\end{equation}

\item
Next, the term $\bm{\beta}_{1,t}$ can be controlled by exploiting some sort of statistical independence across different transitions and applying the Bernstein inequality. This is summarized in the following lemma, with the proof deferred to Section~\ref{sec:proof-lemma-control-beta1}.

\begin{lemma}
\label{lemma:control-beta1}
Consider any fixed vector $\bm{V}^{\star} \in \mathbb{R}^{|\cS|}$. There exists some universal
constant $c >0$ such that for any $0<\delta<1$, one has
\begin{align}
\label{eqn:beta1}
	\forall 1\leq t\leq T: \quad
	%\big|\bm{\beta}_{1,t}\big| 
	\Bigg| \gamma\sum_{i=1}^{t}\prod_{j=i+1}^{t}\big(\bm{I}-\bm{\Lambda}_{j}\big)\bm{\Lambda}_{i}\big(\bm{P}_{i}-\bm{P}\big)\bm{V}^{\star} \Bigg|
	% \leq c \gamma \left(\sqrt{\eta\mathsf{Var}_{\bm{P}}\big[\bm{V}^{\star}\big]\log\Big(\frac{|\mathcal{S}||\mathcal{A}|T}{\delta}\Big)}+\eta\|\bm{V}^{\star}\|_{\infty}\log\Big(\frac{|\mathcal{S}||\mathcal{A}|T}{\delta}\Big)\bm{1}\right) 
	\leq \tau_1 \|\bm{V}^{\star}\|_{\infty} \bm{1}
\end{align}
with probability at least $1-\delta$, provided that $0<\eta\log\big(\frac{|\mathcal{S}||\mathcal{A}|T}{\delta}\big)<1$. Here, we define
\begin{align}
\label{eqn:tau}
	\tau_1 := c \gamma   \sqrt{\eta\log\Big(\frac{|\mathcal{S}||\mathcal{A}|T}{\delta}\Big)} .
\end{align}

\end{lemma}

\item
Additionally, we develop an upper bound on the term $\bm{\beta}_{3,t}$, which follows directly from the concentration of the empirical distribution of the Markov chain (see Lemma~\ref{lemma:Bernstein-state-occupancy}). The proof is deferred to Section~\ref{sec:proof-lemma-control-beta3}. 
\begin{lemma}
\label{lemma:control-beta3}
	For any $\delta>0$, recall the definition of $\tcover$ in \eqref{defn:tepoch}. 
	Suppose that $T>\tcover$ and $0<\eta<1$. Then with probability exceeding $1-\delta$ one has
\begin{align}
\label{eqn:beta3}
	%\left|\bm{\beta}_{3,t} \right| 
	\Bigg | \prod_{j=1}^{t}\big(\bm{I}-\bm{\Lambda}_{j}\big)\bm{\Delta}_{0} \Bigg| 
	\leq (1-\eta)^{\frac{1}{2}t\mumin} \big|\bm{\Delta}_{0}\big| \leq  (1-\eta)^{\frac{1}{2}t\mumin} \|\bm{\Delta}_0 \|_{\infty} \bm{1} 
\end{align}
	uniformly over all $t$ obeying $T\geq  t \geq  \tcover$ and all vector $\bm{\Delta}_{0} \in \mathbb{R}^{|\cS||\cA|}$. 
\end{lemma}
Moreover, in the case where $t < \tcover$, we make note of the straightforward bound
\begin{align}
\Bigg | \prod_{j=1}^{t}\big(\bm{I}-\bm{\Lambda}_{j}\big)\bm{\Delta}_{0} \Bigg| \leq \|\bm{\Delta}_0 \|_{\infty} \bm{1},
\end{align}
given that $\bm{I}-\bm{\Lambda}_{j}$ is a diagonal non-negative matrix whose entries are bounded by $1-\eta <1$. 

\end{itemize}
Substituting the preceding bounds into \eqref{eq:Delta-t-decomposition}, we arrive at 
\begin{align}
\label{eqn:recursion}
\hspace{-0.02in}	|\bDel_t| \leq 
	\begin{cases}
	\gamma \sum_{i=1}^t \linf{\bDel_{i-1}} \prod_{j=i+1}^{t} (\bm{I}-\bLam_j)\bLam_i \one
		+ \tau_1 \|\bm{V}^{\star}\|_{\infty} \one + \linf{\bDel_0} \one,  & t< \tcover \\ 
		\gamma \sum_{i=1}^t \linf{\bDel_{i-1}} \prod_{j=i+1}^{t} (\bm{I}-\bLam_j)\bLam_i \one
		+ \tau_1 \|\bm{V}^{\star}\|_{\infty} \one + (1-\eta)^{ \frac{1}{2} t \mumin } \linf{\bDel_0} \one,   &   \tcover \leq t \leq T
	\end{cases}
\end{align}
with probability at least $1-2\delta$, where $\tcover$ is defined in \eqref{defn:tepoch}.  The rest of the proof is thus dedicated to bounding $|\bDel_t|$ based on the above recursive formula  \eqref{eqn:recursion}.

\subsubsection{Recursive analysis}
\label{sec:recursion}

\paragraph{A crude bound.} We start by observing the following recursive relation
\begin{align}
\label{eqn:recursion-crude}
	|\bDel_t| \leq 
		\gamma \sum_{i=1}^t \linf{\bDel_{i-1}} \prod_{j=i+1}^{t} (\bm{I}-\bLam_j)\bLam_i \one
		+ \tau_1 \| \bm{V}^{\star}\|_{\infty} \one + \|\bDel_0\|_{\infty} \one,  \qquad 1\leq t\leq T,
\end{align}
which is a direct consequence of \eqref{eqn:recursion}. In the sequel, we invoke mathematical induction to establish, for all $1\leq t\leq T$,  the following crude upper bound
\begin{align}
\label{eqn:crude}
	\linf{\bDel_t} \leq \frac{ \tau_1 \| \bm{V}^{\star}\|_{\infty} + \|\bDel_0\|_{\infty}}{1-\gamma} ,
\end{align}
which implies the stability of the asynchronous Q-learning updates.

Towards this, we first observe that \eqref{eqn:crude} holds trivially for the base case (namely, $t=0$). Now suppose that the inequality \eqref{eqn:crude} holds for all iterations up to $t-1$. In view of \eqref{eqn:recursion-crude} and the induction hypotheses, 
\begin{align}
\label{eqn:middle-ground}
	|\bDel_t| \leq   \frac{ \gamma\big( \tau_1 \| \bm{V}^{\star}\|_{\infty} + \linf{\bDel_0} \big)}{1-\gamma} \sum_{i=1}^t \prod_{j=i+1}^{t} (\bm{I} -\bLam_j)\bLam_i \one 
	 + \tau_1 \| \bm{V}^{\star}\|_{\infty} \one + \|\bDel_0\|_{\infty} \one, 
\end{align}
where we invoke the fact that the vector $\prod_{j=i+1}^{t} (\bm{I}-\bLam_j)\bLam_i\one$ is non-negative. Next, 
define the diagonal matrix $\bm{M}_i \defn \prod_{j=i+1}^{t} (\bm{I}-\bLam_j)\bLam_i$, and denote by $N_{i}^j(s,a)$ the number of visits to the
state-action pair $(s,a)$ between the $i$-th and the $j$-th iterations (including $i$ and $j$). Then the diagonal entries of $\bm{M}_i$ satisfy
\begin{align*}
	\bm{M}_i ((s,a), (s,a)) = 
	\begin{cases}
		\eta (1-\eta)^{N_{i+1}^t(s,a)}, \quad & \text{if }(s,a)   = (s_{i-1},a_{i-1}),\\
		0, & \text{if }(s,a) \neq (s_{i-1},a_{i-1}).
	\end{cases}
\end{align*}
Letting $\bm{e}_{(s,a)}\in \mathbb{R}^{|\cS||\cA|}$ be a standard basis vector whose only nonzero entry is the $(s,a)$-th entry, we can easily verify that 
\begin{subequations}
\begin{align}
	%\bm{0} & \leq 
	\prod_{j=i+1}^{t} (\bm{I}-\bLam_j)\bLam_i \one  = \bm M_i \one = \bm M_i \bm{e}_{(s_{i-1},a_{i-1})} 
	= \eta (1-\eta)^{N_{i+1}^t(s_{i-1},a_{i-1})} \bm{e}_{(s_{i-1},a_{i-1})}  \label{eqn:cat2} 
\end{align}
and
\begin{align}
	 \sum_{i=1}^t \prod_{j=i+1}^{t} ( \bm{I} -\bLam_j)\bLam_i \one 
	& = \sum_{i=1}^t  \eta (1-\eta)^{N_{i+1}^t(s_{i-1},a_{i-1})} \bm{e}_{(s_{i-1},a_{i-1})}  \nonumber\\
	& = \sum_{(s,a)\in \cS\times \cA} \Bigg\{ \sum_{i=1}^t  \eta (1-\eta)^{N_{i+1}^t(s,a)} \ind\big\{ (s_{i-1},a_{i-1})=(s,a) \big\} \Bigg\} \bm{e}_{(s,a)} \nonumber\\
	& \leq  \sum_{(s,a)\in \cS\times \cA} \sum_{j=0}^\infty \eta (1-\eta)^j  \bm{e}_{(s,a)}   
	 =  \sum_{j=0}^\infty \eta (1-\eta)^j  \one  = \one. 
	\label{eqn:cat}
\end{align}
\end{subequations}
Combining the above relations with the inequality~\eqref{eqn:middle-ground}, one deduces that
\begin{align*}
	\linf{\bDel_t} \leq  \frac{\gamma ( \tau_1 \| \bm{V}^{\star}\|_{\infty}  + \linf{\bDel_0})}{1-\gamma} + \tau_1 \| \bm{V}^{\star}\|_{\infty}  +  \linf{\bDel_0} 
	% \leq \frac{\gamma (\tau_1  + \linf{\bDel_0})}{1-\gamma} + \tau_1  + \linf{\bDel_0} = \frac{\gamma (\tau_1  + \linf{\bDel_0})}{1-\gamma} + \tau_1  + \linf{\bDel_0} 
	=  \frac{ \tau_1 \| \bm{V}^{\star}\|_{\infty} + \linf{\bDel_0}}{1-\gamma},
\end{align*}
thus establishing \eqref{eqn:crude} for the $t$-th iteration. This induction analysis thus validates \eqref{eqn:crude} for all $1\leq t\leq T$.

\paragraph{Refined analysis.} 
Now, we strengthen the bound~\eqref{eqn:crude} by means of a recursive argument. To begin with, it is easily seen that the term $(1-\eta)^{ \frac{1}{2} t \mumin } \|\bDel_0\|_{\infty}$ is bounded above by $(1-\gamma)\varepsilon$ for any $t>\tth$, where we remind the reader of the definition of $\tth$ in \eqref{eqn:tth} and the fact that $\|\bDel_0\|_{\infty} =\|\bm{Q}^{\star}\|_{\infty}\leq \frac{1}{1-\gamma}$. 
It is assumed that $T>\tth$. To facilitate our argument, we introduce a collection of auxiliary quantities $u_{t}$ as follows
\begin{subequations} 
	\label{eq:defn-ut-vt}
\begin{align}
	&u_0 = \frac{\|\bDel_0\|_{\infty}}{1-\gamma}, \\
	&u_t = \| \bm v_t \|_{\infty}, \quad 
	\bm v_t = \begin{cases}
	 \gamma \sum_{i=1}^t \prod_{j=i+1}^{t} (\bm{I} -\bLam_j)\bLam_i \one u_{i-1} + \| \bDel_0 \|_{\infty} \one, \; &\text{for }1\leq t \leq \tth, \\
 	 % \gamma \sum_{i=1}^t \prod_{j=i+1}^{t} (\bm{I} -\bLam_j)\bLam_i \one u_{i-1} + (1-\eta)^{ \frac{1}{2} t\mumin } \| \bDel_0 \|_{\infty} \one, \; &\text{for }\tcover < t \leq \tth,  \\
  \gamma \sum_{i=1}^t \prod_{j=i+1}^{t} (\bm{I}  -\bLam_j)\bLam_i \one u_{i-1}, 
	\; &\text{for } t > \tth. 
	\end{cases}
\end{align}
\end{subequations}
%
%\begin{align}
%	\label{eq:defn-ut-vt}
%\begin{array}{lll}
%	&u_0 = \frac{\|\bDel_0\|_{\infty}}{1-\gamma},&\\
%	&u_t = \| \bm v_t \|_{\infty}, \qquad 
%	\bm v_t = \gamma \sum_{i=1}^t \prod_{j=i+1}^{t} (\bm{I} -\bLam_j)\bLam_i \one u_{i-1} + \| \bDel_0 \|_{\infty} \one,
%	\quad &\text{for }1\leq t \leq \tcover, \\
%	%\label{eqn:tverysmall}\\
%	&u_t = \| \bm v_t \|_{\infty}, \qquad
%	\bm v_t = \gamma \sum_{i=1}^t \prod_{j=i+1}^{t} (\bm{I} -\bLam_j)\bLam_i \one u_{i-1} + (1-\eta)^{ \frac{1}{2} t\mumin } \| \bDel_0 \|_{\infty} \one, 
%	\quad &\text{for }\tcover < t \leq \tth,  \\
%	%\label{eqn:tsmall}\\
%	&u_t = \|\bm v_t \|_{\infty}, \qquad \bm v_t = \gamma \sum_{i=1}^t \prod_{j=i+1}^{t} (\bm{I}  -\bLam_j)\bLam_i \one u_{i-1}, 
%	\quad &\text{for } t > \tth. 
%	%\label{eqn:tlarge}
%\end{array}
%\end{align}
% \end{subequations}
% Here $t_{0} $ is taken to be $ \frac{\tcover \log \frac{1}{\varepsilon(1-\gamma)^2}}{\eta}$. 
These auxiliary quantities are useful as they provide upper bounds on  $\| \bDel_t \|_{\infty}$, as asserted by the following lemma. The proof is deferred to Section \ref{Sec:middle-term}. 
\begin{lemma}
\label{lemma:middle-term}
	Recall the definition \eqref{eqn:tau} of $\tau_1$ in Lemma~\ref{lemma:control-beta1}. With probability at least $1-2\delta$, the quantities $\{u_t\}$ defined in \eqref{eq:defn-ut-vt} satisfy
\begin{align}
\label{eqn:bear}
	\| \bDel_t \|_{\infty} \leq \frac{ \tau_1 \| \bm{V}^{\star}\|_{\infty} }{1-\gamma} + u_t + \varepsilon. 
\end{align}
\end{lemma}

The preceding result motivates us to turn attention to bounding the quantities $\{u_t\}$. 
Towards this end, we resort to a frame-based analysis by dividing the iterations $[1,t]$ into contiguous frames each comprising $\tcover$ (cf.~\eqref{defn:tepoch}) iterations. Further, we define another auxiliary sequence: 
\begin{align}
	w_k \defn (1- \rho )^{k} \frac{\|\bDel_0\|_{\infty}}{1-\gamma} = (1- \rho )^{k} \frac{\|\bm{Q}_0 -\bm{Q}^{\star}\|_{\infty}}{1-\gamma},
	% k = \lfloor \frac{t - t_0}{\tcover} \rfloor
\end{align}
where we remind the reader of the definition of $\rho$ in \eqref{eqn:rho}. 
The connection between $\{w_k\}$ and $\{u_t\}$ is made precise as follows, whose proof is postponed to Section~\ref{Sec:daban}. 
\begin{lemma}
\label{lemma:daban}
For any $\delta \in (0,\frac{1}{2})$, with probability at least $1-2\delta$, one has 
\begin{align}
\label{eqn:lele}
	u_t \leq w_k, \qquad  
	\text{with} \quad
	k = \max \left\{0, ~\Big\lfloor \frac{t - \tth}{\tcover} \Big\rfloor \right\}.
\end{align}
\end{lemma}

Combining Lemmas~\ref{lemma:middle-term}-\ref{lemma:daban}, we arrive at
\[
\|\bm{Q}_{t}-\bm{Q}^{\star}\|_{\infty}=\|\bm{\Delta}_{t}\|_{\infty}\leq\frac{\tau_{1}\|\bm{V}^{\star}\|_{\infty}}{1-\gamma}+w_{k}+\varepsilon\leq\frac{(1-\rho)^{k}\|\bm{Q}_{0}-\bm{Q}^{\star}\|_{\infty}}{1-\gamma}+\frac{\tau_{1}\|\bm{V}^{\star}\|_{\infty}}{1-\gamma}+\varepsilon ,
\]
which finishes the proof of Theorem~\ref{thm:intermediate}.

\subsection{Proof of Theorem~\ref{thm:main-asyn-Q-learning}}

Now we return to complete the proof of Theorem~\ref{thm:main-asyn-Q-learning}. 
To control $\| \bDel_t \|_{\infty}$ to the desired level, we first claim that the first term of \eqref{eqn:beethoven} 
obeys 
\begin{align}
	\label{eqn:mozart}
	(1- \rho )^{k} \frac{\|\bDel_0\|_{\infty}}{1-\gamma} \leq \varepsilon
\end{align}
whenever
\begin{align}
	t \geq \tth + \tcover  + \frac{4}{(1-\gamma)\eta \mumin} \log\left(\frac{\|\bDel_0\|_{\infty}}{\varepsilon(1-\gamma)}\right) ,
	\label{eq:t-tth-579}
\end{align}
provided that $\eta < 1 / \muepo$. 
Furthermore, by taking the learning rate as
\begin{align}
	\label{eq:defn-eta-proof}
	\eta = \min \left\{\frac{(1-\gamma)^4\varepsilon^2}{c^2\gamma^2 \log \frac{|\mathcal{S}||\mathcal{A}|T}{\delta}},~
	\frac{1}{\muepo} \right\},
\end{align}
%
%with $c_{1}$ a universal constant given as in Lemma~\ref{lemma:control-beta1}, 
one can easily verify that the second term of \eqref{eqn:beethoven} satisfies
%
% \begin{align*}
% 	\gamma \frac{\sqrt{\eta\log\Big(\frac{|\mathcal{S}||\mathcal{A}|T}{\delta}\Big)}}{(1-\gamma)^2} 
% 	+
% 	\eta\log\Big(\frac{|\mathcal{S}||\mathcal{A}|T}{\delta}\Big)
% 	\leq c_1 \varepsilon
% \end{align*}
\begin{align}
\label{eqn:Mahler}
 \frac{ c \gamma}{1-\gamma} \|\bm{V}^{\star}\|_{\infty} \sqrt{\eta\log\Big(\frac{|\mathcal{S}||\mathcal{A}|T}{\delta}\Big)}
	\leq \varepsilon,
\end{align}
where the last step follows since $\|\bm{V}^{\star}\|_{\infty} \leq \frac{1}{1-\gamma}.$
Putting the above bounds together ensures $\| \bDel_t \|_{\infty} \leq 3\varepsilon$. 
%provided that \eqref{eqn:mozart} and \eqref{eq:defn-eta-proof} are satisfied.   
By replacing $\varepsilon$ with ${\varepsilon}/{3}$, we can readily conclude the proof, as long as the claim \eqref{eqn:mozart} can be justified.

\begin{proof}[Proof of the inequality~\eqref{eqn:mozart}.] Observe that 
\begin{align*}
	(1- \rho )^{k} \frac{\|\bDel_0\|_{\infty}}{1-\gamma} \leq \exp(- \rho k) \frac{\|\bDel_0\|_{\infty}}{1-\gamma} \leq \varepsilon
\end{align*}
holds true whenever $k \geq \frac{\log\big(\frac{\|\bDel_0\|_{\infty}}{\varepsilon(1-\gamma)}\big)}{\rho}$, which would hold as long as (according to the definition \eqref{eqn:lele} of $k$)
\begin{align}
\label{eqn:clarinet}
	t \geq \tth + \tcover + \frac{\tcover}{\rho} \log\left(\frac{\|\bDel_0\|_{\infty}}{\varepsilon(1-\gamma)}\right)  . 
\end{align}
In addition, if $\eta < 1/\muepo$, then one has $(1-\eta)^{\muepo} \leq 1 - \eta\muepo/2$, thus guaranteeing that  
\begin{align*}
	\rho = (1-\gamma)\big(1-(1-\eta)^{\muepo}) \geq (1-\gamma)\Big(1 - 1 + \frac{\eta \muepo}{2} \Big) 
	= \frac{1}{2}(1-\gamma)\eta \muepo.
\end{align*}
This taken collectively with \eqref{eqn:clarinet} demonstrates that $(1- \rho )^{k} \frac{\|\bDel_0\|_{\infty}}{1-\gamma} \leq \varepsilon$ holds as long as
\begin{align}
\label{eqn:clarinet2}
	t \geq \tth + \tcover + \frac{2\tcover}{(1-\gamma)\eta \muepo} \log\left(\frac{\|\bDel_0\|_{\infty}}{\varepsilon(1-\gamma)}\right) = \tth + \tcover + \frac{4}{(1-\gamma)\eta \mumin} \log\left(\frac{\|\bDel_0\|_{\infty}}{\varepsilon(1-\gamma)}\right), 
\end{align}
where we have made use of the definition of $\muepo$ (cf.~\eqref{eqn:muepo}).
%	As a consequence, the condition  \eqref{eqn:clarinet2} would hold as long as \yxc{???}
%%
%\begin{align}
%	t \geq \tth + \tcover + \frac{4}{(1-\gamma)\eta \mumin} \log\left(\frac{1}{\varepsilon(1-\gamma)^2}\right) 
%	\geq \tth + \tcover + \frac{2\tcover}{(1-\gamma)\eta \muepo} \log\left(\frac{\|\bDel_0\|_{\infty}}{\varepsilon(1-\gamma)}\right),
%\end{align}
%%
%where we have made use of the simple bound $\|\bm{\Delta}_0\|_{\infty}=\|\bm{Q}^{\star}\|_{\infty}\leq \frac{1}{1-\gamma}$ with $\bm{Q}_0 =\bm{0}$. 
%
\end{proof}

% Comparing terms in \eqref{eqn:mozart} and \eqref{eqn:clarinet}, it only remains for us to show that 
% \begin{align*}
% 	\frac{\tcover}{\log(\frac{1}{1-\rho})} \leq \frac{8}{(1-\gamma)\eta \mumin}.
% \end{align*}
% To this end, first recall the elementary relation that $\log(\frac{1}{1-\rho}) = \log(1 + \frac{\rho}{1-\rho}) \geq \frac{\rho}{2(1-\rho)}$ for $\rho < 1$. 
% At the same time, 

% Combining these arguments together ensures
% \begin{align}
% 	\frac{\tcover}{\log(\frac{1}{1-\rho})} \leq \frac{2\tcover(1-\rho)}{\rho} \leq \frac{4 \tcover}{(1-\gamma)\eta \muepo}
% 	= \frac{8}{(1-\gamma)\eta \mumin}.
% % \end{align}
% which finishes the last piece in proving inequality~\eqref{eqn:mozart}.
% Here the last equality uses the definition of $\muepo$ which is provided in expression~\eqref{eqn:rho}.

\section{Cover-time-based analysis of asynchronous Q-learning}
\label{sec:analysis-cover-time}

In this section, we sketch the proof of Theorem~\ref{thm:main-asyn-Q-learning-cover-time}. Before continuing, we recall the definition of $\tcovertime$ in \eqref{eq:defn-cover-time}, and further introduce a quantity
\begin{equation}
	\tcoverall := \tcovertime \log \frac{T}{\delta}.
	\label{defn:tcover-all}
\end{equation}
There are two useful facts regarding $\tcoverall$ that play an important role in the analysis.
\begin{lemma}
\label{lem:connection-cover-time}
	Define the event $$\mathcal{K}_{l}:=\Big\{ \exists(s,a)\in\mathcal{S\times\mathcal{A}}\text{ s.t.~it is not visited within iterations }\big(l\tcoverall,(l+1)\tcoverall\big]\, \Big\},$$ and set $L:=\lfloor\frac{T}{\tcoverall}\rfloor$. Then one has $\mathbb{P}\left\{ \bigcup\nolimits_{l=0}^{L}\mathcal{K}_{l}\right\}   \leq\delta.$	
\end{lemma}
\begin{proof} See Section~\ref{sec:proof-lemma-connection-cover-time}. \end{proof}
 In other words, Lemma~\ref{lem:connection-cover-time} tells us that with high probability, all state-action pairs are visited at least once in every time frame $(l\tcoverall,(l+1)\tcoverall]$ with $0\leq l\leq \lfloor {T}/{\tcoverall} \rfloor$. The next result is a consequence of Lemma~\ref{lem:connection-cover-time} as well as  the analysis of Lemma~\ref{lemma:control-beta3}; the proof can be found in Section~\ref{sec:proof-lemma-control-beta3}. 
\begin{lemma}
\label{lemma:control-beta3-cover-time}
	For any $\delta>0$, recall the definition of $\tcoverall$ in \eqref{defn:tcover-all}. 
	Suppose that $T>\tcoverall$ and $0<\eta<1$. Then with probability exceeding $1-\delta$ one has
\begin{align}
\label{eqn:beta3-cover-time}
	%\left|\bm{\beta}_{3,t} \right| 
	\Bigg | \prod_{j=1}^{t}\big(\bm{I}-\bm{\Lambda}_{j}\big)\bm{\Delta}_{0} \Bigg| 
	% \leq (1-\eta)^{\frac{t}{2\tcoverall}} \big|\bm{\Delta}_{0}\big| 
	\leq  (1-\eta)^{\frac{t}{2\tcoverall}} \|\bm{\Delta}_0 \|_{\infty} \bm{1} 
\end{align}
	uniformly over all $t$ obeying $T\geq  t \geq  \tcoverall$ and all vector $\bm{\Delta}_{0} \in \mathbb{R}^{|\cS||\cA|}$. 
\end{lemma}

With the above two lemmas in mind, we are now positioned to prove Theorem~\ref{thm:main-asyn-Q-learning-cover-time}. 
Repeating the analysis of \eqref{eqn:recursion} (except that Lemma~\ref{lemma:control-beta3} is replaced by Lemma~\ref{lemma:control-beta3-cover-time}) yields 
\begin{align*}%\label{eqn:recursion-cover-time}
\hspace{-0.02in}	|\bDel_t| \leq 
	\begin{cases}
	\gamma \sum_{i=1}^t \linf{\bDel_{i-1}} \prod_{j=i+1}^{t} (\bm{I}-\bLam_j)\bLam_i \one
		+ \tau_1 \|\bm{V}^{\star}\|_{\infty} \one + \linf{\bDel_0} \one,  & t< \tcoverall \\ 
		\gamma \sum_{i=1}^t \linf{\bDel_{i-1}} \prod_{j=i+1}^{t} (\bm{I}-\bLam_j)\bLam_i \one
		+ \tau_1 \|\bm{V}^{\star}\|_{\infty} \one + (1-\eta)^{ \frac{t}{2\tcoverall} } \linf{\bDel_0} \one,   &   \tcoverall \leq t \leq T
	\end{cases}
\end{align*}
with probability at least $1-2\delta$. This observation resembles \eqref{eqn:recursion}, except that $\tcover$ (resp.~$\mumin$) is replaced by $\tcoverall$ (resp.~$\frac{1}{\tcoverall}$). As a consequence, we can immediately use the recursive analysis carried out in Section~\ref{sec:recursion} to establish a convergence guarantee based on the cover time. More specifically, define
\begin{align}
	\widetilde{\rho} &:= (1-\gamma)  \Big( 1 - (1-\eta)^{ \frac{\tcoverall}{2\tcoverall}} \Big) = (1-\gamma)  \Big( 1 - (1-\eta)^{\frac{1}{2} } \Big) .
	\label{eq:defn-rho-tilde}
\end{align}
Replacing $\rho$ by $\widetilde{\rho}$ in Theorem~\ref{thm:intermediate} reveals that with probability at least $1-6\delta$, 
\begin{align}
\label{eqn:beethoven-covertime}
	\| \bm{Q}_t - \bm{Q}^{\star} \|_{\infty} \leq (1- \widetilde{\rho} )^{ k } \frac{\| \bm{Q}_0 - \bm{Q}^{\star} \|_{\infty}}{1-\gamma} 
	+ \frac{c \gamma}{1-\gamma} \|\bm{V}^{\star}\|_{\infty} \sqrt{\eta\log\Big(\frac{|\mathcal{S}||\mathcal{A}|T}{\delta}\Big)} 
	 +\varepsilon
	% \qquad k = \Big\lfloor \frac{t - \tth}{\tcover} \Big\rfloor.
\end{align}
holds for all $t\leq T$, where $k \defn \max \big\{0, ~\big\lfloor \frac{t - t_{\mathsf{th,cover}}}{\tcoverall} \big\rfloor \big\}$ and we abuse notation to define
\begin{align*}
	t_{\mathsf{th,cover}} := 2\tcoverall\log\frac{1}{(1-\gamma)^{2}\varepsilon}.
\end{align*}

Repeating the proof of the inequality \eqref{eqn:mozart} yields
\begin{align*}
	(1- \widetilde{\rho} )^{k} \frac{\|\bDel_0\|_{\infty}}{1-\gamma} \leq \varepsilon,
\end{align*}
whenever $t \geq t_{\mathsf{th,cover}} + \tcoverall  + \frac{2\tcoverall}{(1-\gamma)\eta } \log\big(\frac{1}{\varepsilon(1-\gamma)^2}\big)$, 
with the proviso that $\eta < 1/2$. In addition, setting $\eta = \frac{(1-\gamma)^{4}}{c^{2}\gamma^{2}\varepsilon^{2}\log\big(\frac{|\mathcal{S}||\mathcal{A}|T}{\delta}\big)}$ guarantees  that 
\begin{equation*}
	\frac{c \gamma}{1-\gamma} \|\bm{V}^{\star}\|_{\infty} \sqrt{\eta\log\Big(\frac{|\mathcal{S}||\mathcal{A}|T}{\delta}\Big)} 
	\leq \frac{c \gamma}{(1-\gamma)^2} \sqrt{\eta\log\Big(\frac{|\mathcal{S}||\mathcal{A}|T}{\delta}\Big)}  
	\leq \varepsilon.
\end{equation*}
In conclusion, we have 	$\| \bm{Q}_t - \bm{Q}^{\star} \|_{\infty}\leq 3\varepsilon$ as long as 
\begin{equation*}
	t \geq \frac{c'\tcoverall}{(1-\gamma)^5\varepsilon^2 } \log\Big(\frac{|\mathcal{S}||\mathcal{A}|T}{\delta}\Big) \log \Big(\frac{1}{\varepsilon(1-\gamma)^2}\Big),
\end{equation*}
for some sufficiently large constant $c'>0$. This together with the definition \eqref{defn:tcover-all} completes the proof.

\section{Analysis under adaptive learning rates (proof of Theorem~\ref{thm:new-learning-rates-asyn-Q})}
\label{sec:proof-thm:new-learning-rates-asyn-Q}

%As usual, we use the vector $\widehat{\bm Q}_t\in \mathbb{R}^{|\cS||\cA|}$ to represent our estimate $\widehat{Q}_t$. 

\paragraph{Useful preliminary facts about $\eta_t$.} To begin with, we make note of several useful properties about $\eta_t$. 
\begin{itemize}
\item Invoking the concentration result in Lemma~\ref{lemma:Bernstein-state-occupancy}, one can easily show that with probability at least $1 - \delta$, 
%
%\blue{
\begin{align}
	\frac12\mu_{\mathsf{min}} < \min_{s, a} \frac{K_t(s, a)}{t} < \frac32\mu_{\mathsf{min}}
	\label{eq:concentration-Kt}
\end{align} 
holds simultaneously for all $t$ obeying $T\geq t \ge  \frac{ 443\tmix \log(\frac{4|\mathcal{S}||\mathcal{A}|t}{\delta})}{\mu_{\mathsf{min}}}$.
%}
In addition, this concentration result taken collectively with the update rule \eqref{eq:mu_est} of $\widehat{\mu}_{\mathsf{min},t}$ --- in particular, the second case of \eqref{eq:mu_est} --- implies that $\widehat{\mu}_{\mathsf{min},t}$ ``stabilizes'' as $t$ grows; to be precise, there exists some quantity $c'\in [1/6,9/2]$ such that
\begin{align}
	\widehat{\mu}_{\mathsf{min},t} \equiv c'\mu_{\mathsf{min}}
	%,\quad \forall t \gtrsim  \frac{t_{\rm mix}\log(\frac{|\mathcal{S}||\mathcal{A}|t}{\delta})}{\mu_{\mathsf{min}}},
\end{align}
holds simultaneously for all $t$ obeying $T\geq t \ge  \frac{ 443\tmix \log(\frac{4|\mathcal{S}||\mathcal{A}|t}{\delta})}{\mu_{\mathsf{min}}}$.

\item 
For any $t$ obeying $t \ge \frac{ 6c_{\eta}\tmix \log(\frac{2|\mathcal{S}||\mathcal{A}|t}{\delta})}{\mu_{\mathsf{min}}(1-\gamma)\gamma^2}$ (so that $\frac{\log t}{\widehat{\mu}_{\mathsf{min},t}(1-\gamma)\gamma^2t} \le \frac{1}{c_{\eta}}$ and $t \ge \frac{ 443\tmix \log(\frac{2|\mathcal{S}||\mathcal{A}|t}{\delta})}{\mu_{\mathsf{min}}}$ for $c_{\eta} \ge 11$), the learning rate \eqref{eq:learning-rate-implement} simplifies to
\begin{align}
	\eta_t = c_{\eta}\exp\Big(\Big\lfloor\log\frac{\log t}{c'\mu_{\mathsf{min}}(1-\gamma)\gamma^2t}\Big\rfloor\Big).
	\label{eq:learning-rate-simplified}
\end{align}
Clearly, there exists a sequence of endpoints $t_{1} < t_{2} < t_{3} < \ldots$ with $t_1 \le \frac{ 6ec_{\eta}\tmix \log(\frac{2|\mathcal{S}||\mathcal{A}|t_1}{\delta})}{\mu_{\mathsf{min}}(1-\gamma)\gamma^2}$ such that: 
\begin{align} 
	&2t_k < t_{k+1} < 3t_k \qquad\qquad \text{and}  \label{eq:t_k} \\
	\eta_t = \eta_{(k)} &:= \frac{\alpha_k\log t_{k+1}}{\mu_{\mathsf{min}}(1-\gamma)\gamma^2t_{k+1}},\quad \forall t_k < t \le t_{k+1} \label{eq:eta_t_k}
\end{align}
for some positive constant $\alpha_k \in \big[\frac{2c_{\eta}}{9e}, 6c_{\eta} \big]$;
in words,  \eqref{eq:eta_t_k} provides a concrete expression/bound for the piecewise constant learning rate, where the $t_k$'s form the change points.  

\end{itemize}

Combining \eqref{eq:eta_t_k} with the definition of $\widehat{Q}_t$ (cf.~\eqref{eq:mu_est}), one can easily check that for $t > t_1$,
\begin{align}
\widehat{Q}_t = Q_{t_k},\qquad \forall t_k < t \le t_{k+1},
\end{align}
meaning that $\widehat{Q}_t$ remains fixed within each time segment $(t_k, t_{k+1}]$. 
With this property in mind, we only need to analyze $Q_{t_k}$ in the sequel, which can be easily accomplished by invoking Theorem~\ref{thm:main-asyn-Q-learning}.

	\paragraph{A crude bound.} Given that $0<\eta_t\leq 1$ and $0\leq r(s,a)\leq 1$, the update rule \eqref{eqn:q-learning} of $\bm Q_{t}$ implies that
\begin{align*}
	\|\bm Q_{t}\|_{\infty} \le \max\big\{ (1-\eta_t)\|\bm Q_{t-1}\|_{\infty} + \eta_t(1+\gamma\|\bm Q_{t-1}\|_{\infty}), ~\|\bm Q_{t-1}\|_{\infty} \big\} 
	\le \|\bm Q_{t-1}\|_{\infty} + \gamma,
\end{align*}
thus leading to the following crude bound 
\begin{align}
	\|\bm Q_{t}-\bm Q^{\star} \|_{\infty}\le t + \|\bm Q_{0}\|_{\infty} + \|\bm Q^{\star}\|_{\infty} 
	\le t + \frac{2}{1-\gamma} \leq 3t, \qquad  \text{for any } t> \frac{1}{1-\gamma} .
	\label{eq:crude-Qt-UB1}
\end{align}
\begin{remark}
	As we shall see momentarily, this crude bound allows one to control --- in a coarse manner --- the error at the beginning of each time interval $[t_{k-1}, t_k]$, which is needed when invoking Theorem~\ref{thm:main-asyn-Q-learning}.
\end{remark}

\paragraph{Refined analysis.} Let us define 
\begin{align}
	{\varepsilon}_k \defn \sqrt{\frac{c_{k,0}\log(\frac{|\mathcal{S}||\mathcal{A}|t_k}{\delta})\log t_k}{\mu_{\mathsf{min}}(1-\gamma)^5\gamma^2t_k}},
	\label{eq:defn-varepsilon-tilde}
\end{align}
where the constant $c_{k,0}$ is chosen to be $c_{k,0}= {\alpha_{k-1}}/{c_1} > 0$, with $c_1>0$  the universal constant stated in Theorem~\ref{thm:main-asyn-Q-learning}.
The property \eqref{eq:eta_t_k} of $\eta_t$ together with the definition \eqref{eq:defn-varepsilon-tilde} implies that
\begin{align*}
	\eta_t =  \frac{c_1 (1-\gamma)^4 {\varepsilon}_k^2 }{\log(\frac{|\mathcal{S}||\mathcal{A}|t_k}{\delta})}  
	= \frac{c_1}{\log(\frac{|\mathcal{S}||\mathcal{A}|t_k}{\delta})}\min
	\Big\{ (1-\gamma)^4 {\varepsilon}_k^2, \frac{1}{t_{\mathsf{mix}}} \Big\}, \quad \forall t\in (t_{k-1}, t_k],
\end{align*} 
as long as $(1-\gamma)^4 {\varepsilon}_k^2 \leq {1}/{t_{\mathsf{mix}}}$, or more explicitly, when
\begin{align}
	t_{k}\geq\frac{c_{k,0}t_{\mathsf{mix}}\log(\frac{|\mathcal{S}||\mathcal{A}|t_{k}}{\delta})\log t_{k}}{\mu_{\mathsf{min}}(1-\gamma)\gamma^{2}}. 
	\label{eq:tk-lower-bound}
\end{align}
%
% Here, $c'_0$ is some constant (cf.~Theorem~\ref{thm:main-asyn-Q-learning}). 
%\blue{
In addition, the condition \eqref{eq:t_k} and the definition \eqref{eq:defn-varepsilon-tilde} further tell us that 
\begin{align*}
	t_k - t_{k-1} & > t_{k-1} > \frac{1}{3}t_k 
	= \frac{c_{k,0}\log\big(\frac{|\mathcal{S}||\mathcal{A}|t_k}{\delta} \big)\log t_k}{3\mu_{\mathsf{min}}(1-\gamma)^5\gamma^2 {\varepsilon}_k^2}. 
\end{align*}
Invoking Theorem~\ref{thm:main-asyn-Q-learning} with an initialization $\bm Q_{t_{k-1}}$ (which clearly satisfies the crude bound \eqref{eq:crude-Qt-UB1}) ensures that 
\begin{align}
	\|\bm Q_{t_k}-\bm Q^{\star}\|_{\infty}\leq {\varepsilon}_k
	%\leq \varepsilon
\end{align}
with probability at least $1-\delta$, with the proviso that
	\begin{align}	
		\frac{1}{3}t_k & \geq \frac{c_{0}}{\mumin}\left\{ \frac{1}{(1-\gamma)^{5} {\varepsilon}_k^{2}}+\frac{\tmix}{1-\gamma}\right\} 
		\log\Big( \frac{|\mathcal{S}||\mathcal{A}| t_k }{\delta} \Big) \log \Big(\frac{t_k}{(1-\gamma)^2 {\varepsilon}_k } \Big) 
		\label{eqn:numbersteps-new}
	\end{align}
with $c_0>0$ the universal constant stated in Theorem~\ref{thm:main-asyn-Q-learning}.
Under the sample size condition \eqref{eq:tk-lower-bound}, this requirement \eqref{eqn:numbersteps-new} can be guaranteed by adjusting the constant $c_{\eta}$ in \eqref{eq:learning-rate-implement} to satisfy the following inequality:
\begin{align*}
c_{k,0} = \frac{\alpha_{k-1}}{c_1} \ge \frac{2c_{\eta}}{9ec_1} > 6c_0.
\end{align*}
%}

%\blue{
Finally, taking $t_{k_{\max}}$ to be the largest change point that does not exceed $T$, we see from \eqref{eq:t_k} that $\frac{1}{3}T\leq t_{k_{\max}}\leq  T$. 
Then one has
\begin{align}
	\|\bm Q_{T}-\bm Q^{\star}\|_{\infty} &= \|\bm Q_{t_{k_{\max}}}-\bm Q^{\star}\|_{\infty} \leq \varepsilon_{k_{\max}} = \sqrt{\frac{c_{k,0}\log(\frac{|\mathcal{S}||\mathcal{A}|t_{k_{\max}}}{\delta})\log t_{k_{\max}}}{\mu_{\mathsf{min}}(1-\gamma)^5\gamma^2t_{k_{\max}}}} 
	\notag\\
	& \le \sqrt{\frac{3c_{k,0}\log(\frac{|\mathcal{S}||\mathcal{A}|T}{\delta})\log T}{\mu_{\mathsf{min}}(1-\gamma)^5\gamma^2T}}
	%\leq \varepsilon
\end{align}
These immediately conclude the proof of the theorem 
under the sample size condition \eqref{eq:sample-new}, provided that
\begin{align*}
C > \frac{18c_{\eta}}{c_1} > \frac{3\alpha_{k-1}}{c_1} = 3c_{k,0}.
\end{align*}
%}

\section{Analysis of asynchronous variance-reduced Q-learning}
\label{Sec:Asynchronous-VR}

This section aims  to establish Theorem~\ref{thm:main-asyn-VR-Q-learning}. 
We carry out an epoch-based analysis, that is, we first quantify the progress made over each epoch, and then demonstrate how many epochs are sufficient to attain the desired accuracy. In what follows, we shall overload the notation by defining 
\begin{subequations}
\begin{align}
	\tcover &:=   \frac{443\tmix}{\mumin}\log\Big(\frac{4|\mathcal{S}||\mathcal{A}|\tepoch}{\delta}\Big), \label{defn:tepoch-vr} \\
	\tth &:= \max\Bigg\{ \frac{2 \log\frac{1}{(1-\gamma)^2\varepsilon}}{ \eta \mumin },\: \tcover \Bigg\} , \label{eqn:tth-vr} \\
	\rho &\defn (1-\gamma)\big(1-(1-\eta)^{\muepo} \big), \label{eqn:rho-vr} \\ 
	\muepo &\defn \frac{1}{2}\mumin\tcover.   \label{eqn:mumin-vr}
\end{align}
\end{subequations}

\subsection{Per-epoch analysis}

We start by analyzing the progress made over each epoch. 
Before proceeding, we denote by $\widetilde{\bm{P}}\in [0,1]^{|\cS||\cA|\times |\cS|}$ a matrix corresponding to the empirical probability transition kernel used in \eqref{eq:surrogate_TQ} from $N$ new sample transitions. Further, we use the vector $\overline{\bm{Q}} \in \mathbb{R}^{|\cS||\cA|}$ to represent the reference Q-function, and introduce the vector  $\overline{\bm{V}} \in \mathbb{R}^{|\cS|}$ to represent the corresponding value function so that $\overline{V}(s) := \max_a \overline{Q}(s,a)$ for all $s\in\cS$. 

For convenience, this subsection abuses notation to assume that an epoch starts with an estimate $\bm Q_0 = \Qbar$, and consists of the subsequent
\begin{align}
	\tepoch \defn \tcover + \tth + \frac{8\log \frac{2}{1-\gamma}}{(1-\gamma)\eta\mumin}
\end{align}
iterations of variance-reduced Q-learning updates, where $\tcover$ and $\tth$ are defined  in  \eqref{defn:tepoch-vr} and \eqref{eqn:tth-vr}, respectively. 
In the sequel, we divide all epochs into two phases,  depending on the quality of the initial estimate $\Qbar$ in each epoch.

\subsubsection{Phase 1: when $\|\overline{\bm{Q}}-\bm{Q}^{\star}\|_{\infty} > 1/{\sqrt{1-\gamma}}$}

Recalling the matrix notation of $\bm{\Lambda}_{t}$ and $\bm{P}_{t}$ in \eqref{eq:defn-Lambda-t} and \eqref{eq:defn-Pt}, respectively, we can rewrite \eqref{eqn:vr-q-learning} as follows
\begin{equation}  
\bm{Q}_{t}=\big(\bm{I}-\bm{\Lambda}_{t}\big)\bm{Q}_{t-1}+\bm{\Lambda}_{t}\left( \bm{r}+\gamma\bm{P}_{t}(\bm{V}_{t-1} - \overline{\bm{V}}) + \gamma \widetilde{\bm{P}} \overline{\bm{V}} \right).\label{eq:vr-Q-learning-matrix-notation}
\end{equation}
Following similar steps as in the expression \eqref{eq:Delta-t-identity2}, we arrive at the following error decomposition
\begin{align}
\Delvr_t := \bm{Q}_{t}-\bm{Q}^{\star} & =\big(\bm{I}-\bm{\Lambda}_{t}\big)\bm{Q}_{t-1}+ \bm{\Lambda}_{t}\left( \bm{r}+\gamma\bm{P}_{t}(\bm{V}_{t-1} - \overline{\bm{V}}) + \gamma \widetilde{\bm{P}} \overline{\bm{V}} \right) -\bm{Q}^{\star}\nonumber \\
 & =\big(\bm{I}-\bm{\Lambda}_{t}\big)\big(\bm{Q}_{t-1}-\bm{Q}^{\star}\big)+\bm{\Lambda}_{t}\left( \bm{r}+\gamma\bm{P}_{t}(\bm{V}_{t-1} - \overline{\bm{V}}) + \gamma \widetilde{\bm{P}} \overline{\bm{V}} -\bm{Q}^{\star}\right)\nonumber \\
 & =\big(\bm{I}-\bm{\Lambda}_{t}\big)\big(\bm{Q}_{t-1}-\bm{Q}^{\star}\big)+\gamma\bm{\Lambda}_{t}\left( \bm{P}_{t}(\bm{V}_{t-1} - \overline{\bm{V}})+ \widetilde{\bm{P}} \overline{\bm{V}} -\bm{P}\bm{V}^{\star} \right)\nonumber \\
 & =\big(\bm{I}-\bm{\Lambda}_{t}\big)\bm{\Theta}_{t-1}+\gamma\bm{\Lambda}_{t}\big(\widetilde{\bm{P}}-\bm{P}\big)\overline{\bm{V}}+ \gamma\bm{\Lambda}_{t}\big(\bm{P}_{t}-\bm{P}\big) (\bm{V}^{\star} - \overline{\bm{V}}) +\gamma\bm{\Lambda}_{t}\bm{P}_{t}\big(\bm{V}_{t-1}-\bm{V}^{\star}\big), \label{eq:Delta-t-identity-vr}
\end{align}
which once again leads to a recursive relation
\begin{align}
	\Delvr_t =&
	\underset{=:\bm{h}_{0,t}}{\underbrace{\gamma\sum_{i=1}^{t}\prod_{j=i+1}^{t}\big(\bm{I}-\bm{\Lambda}_{j}\big)\bm{\Lambda}_{i}\big(\widetilde{\bm{P}}-\bm{P}\big)\overline{\bm{V}}}} + \underset{=:\bm{h}_{1,t}}{\underbrace{\gamma\sum_{i=1}^{t}\prod_{j=i+1}^{t}\big(\bm{I}-\bm{\Lambda}_{j}\big)\bm{\Lambda}_{i}\big(\bm{P}_{i}-\bm{P}\big)(\bm{V}^{\star} - \Vbar)}} \nonumber \\
	&\qquad\qquad + \underset{=:\bm{h}_{2,t}}{\underbrace{\gamma\sum_{i=1}^{t}\prod_{j=i+1}^{t}\big(\bm{I}-\bm{\Lambda}_{j}\big)\bm{\Lambda}_{i}\bm{P}_{i}\big(\bm{V}_{i-1}-\bm{V}^{\star}\big)}}+\underset{=:\bm{h}_{3,t}}{\underbrace{\prod_{j=1}^{t}\big(\bm{I}-\bm{\Lambda}_{j}\big)\Delvr_{0}}}.
	\label{eq:defn-beta0-beta1-beta3}
\end{align}
This identity takes a very similar form as \eqref{eq:defn-beta1-beta3} except for the additional term $\bm{h}_{0,t}$.

Let us begin by controlling the first term, towards which we have the following lemma. The proof is postponed to Section~\ref{sec:proof-lemma-bound-h0t}. 
\begin{lemma}
	\label{lemma:bound-h0t}
	Suppose that $\widetilde{\bm{P}}$ is constructed using $N$ consecutive sample transitions. If $N>\tcover$,  then with probability greater than $1-\delta$, one has
\begin{align}
	\| \bm{h}_{0,t}\|_{\infty}
&	\leq \gamma \sqrt{\frac{4\log\big( \frac{6N|\mathcal{S}||\mathcal{A}|}{\delta} \big)}{N\mumin}}   \big\|\overline{\bm{V}} - \bm{V}^{\star}\big\|_{\infty}  +\frac{\gamma}{1-\gamma} \sqrt{\frac{4\log\big( \frac{6N|\mathcal{S}||\mathcal{A}|}{\delta} \big)}{N\mumin}}    .
	\label{eq:Hoeffding-Ptilde-lemma}
\end{align}
\end{lemma}

If $t<\tcover$, then it is straightforwardly seen that 
\begin{align*}
	\left|\bm{h}_{3,t}\right| & \leq
	\|\Delvr_{0}\|_{\infty}\bm{1}.
\end{align*}
Taking this  together with the results from Lemma~\ref{lemma:control-beta1} and Lemma~\ref{lemma:control-beta3}, we are guaranteed that 
\begin{align*}
	\big|\bm{h}_{1,t}\big| & \leq\tau_2 \| \bm{V}^{\star}-\Vbar\|_{\infty} \one\\
	\left|\bm{h}_{3,t}\right| & \leq\begin{cases}
(1-\eta)^{\frac{1}{2}t\mumin}\|\Delvr_{0}\|_{\infty}\bm{1}, \quad & \text{if } \tcover\leq t\leq\tepoch\\[2mm]
\|\Delvr_{0}\|_{\infty}\bm{1}, & \text{if } t<\tcover
\end{cases}
\end{align*}
with probability at least $1-2\delta$, 
where 
\[
	\tau_2 := c'
\gamma  \sqrt{\eta\log\big(\frac{|\mathcal{S}||\mathcal{A}|\tepoch}{\delta}\big)}
\]
for some constant $c'>0$ (similar to \eqref{eqn:tau}). 
In addition, the term $\bm{h}_{2,t}$ can be bounded in the same way as $\bm{\beta}_{2,t}$ in \eqref{eq:defn-beta4}. Therefore, repeating the same argument as for Theorem~\ref{thm:intermediate} and taking $\xi = \frac{1}{16\sqrt{1-\gamma}}$, 
we conclude that with probability at least $1-\delta$, 
\begin{align}
\label{eqn:beethoven9}
	\|\bm{\Theta}_{t}\|_{\infty}
	\leq (1-\rho)^{k}\frac{\|\bm{\Theta}_{0}\|_{\infty}}{1-\gamma}+\widetilde{\tau}+\xi
	=  (1-\rho)^{k}\frac{\|\overline{\bm{Q}}-\bm{Q}^{\star}\|_{\infty}}{1-\gamma}+\widetilde{\tau}+\xi
\end{align}
holds simultaneously for all $0< t \leq \tepoch$, 
where  $k=\max\big\{ 0,\big\lfloor\frac{t-\ttheps}{\tcover}\big\rfloor\big\}$, and
\begin{align*}
	\widetilde{\tau} &:=\frac{c\gamma}{1-\gamma}\left\{ \sqrt{\frac{\log\frac{N|\cS||\cA|}{\delta}}{(1-\gamma)^{2}N\mumin}}+\|\bm{V}^{\star}-\Vbar\|_{\infty}\left( \sqrt{\eta\log\Big(\frac{|\mathcal{S}||\mathcal{A}|\tepoch}{\delta}\Big)}+ \sqrt{\frac{4\log\big( \frac{6N|\mathcal{S}||\mathcal{A}|}{\delta} \big)}{N\mumin}}  \right) \right\} ,\\
	\ttheps &:=\max\left\{ \frac{2\log\frac{1}{(1-\gamma)^{2}\xi}}{\eta\mu_{\min}},\tcover\right\} 
\end{align*}
for some constant $c>0$. 

Let $C>0$ be some sufficient large constant. Setting $\eta_t\equiv \eta = \min\Big\{\frac{(1-\gamma)^2}{C \gamma^2 \log \frac{|\cS||\cA|\tepoch}{\delta}}, \frac{1}{\muepo}\Big\}$, and ensuring $N \geq \max\{\tcover, C \frac{\log \frac{N|\cS||\cA|}{\delta}}{(1-\gamma)^3\mumin}\}$, we can easily demonstrate that 
\begin{align*}
	\| \Delvr_t \|_{\infty} \leq (1- \rho )^{k} \frac{\|\Qbar - \Qstar\|_{\infty}}{1-\gamma}
	+
	\frac{1}{8\sqrt{1-\gamma}}
	+
	\frac{1}{4} \|\bm{V}^{\star} - \Vbar\|_{\infty}.
\end{align*}
As a consequence, if $\tepoch \geq \tcover + \ttheps + \frac{8\log \frac{2}{1-\gamma}}{(1-\gamma)\eta\mumin}$, one has $$(1- \rho )^{k} \leq \frac{1}{8}(1-\gamma),$$ which in turn implies that 
\begin{align}
\label{eqn:base-step-richter2}
	\| \Delvr_{\tepoch} \|_{\infty} \leq \frac{1}{8}\|\Qbar - \Qstar\|_{\infty} + 
	\frac{1}{8\sqrt{(1-\gamma)}}
	+ \frac{1}{4} \|\bm{V}^{\star} - \Vbar\|_{\infty} 
	\leq \frac{1}{2} \max \Big\{\frac{1}{\sqrt{1-\gamma}},~\|\Qbar - \Qstar\|_{\infty} \Big\} ,
\end{align}
where the last step invokes the simple relation $\|\bm{V}^{\star} - \Vbar\|_{\infty} \leq \|\Qbar - \Qstar\|_{\infty}.$ Thus, we conclude that
\begin{align}
\label{eqn:base-step-richter}
	\| \bm{Q}_{\tepoch} - \Qstar \|_{\infty}  
	\leq \frac{1}{2} \max \Big\{\frac{1}{\sqrt{1-\gamma}},~\|\Qbar - \Qstar\|_{\infty} \Big\} .
\end{align}
\subsubsection{Phase 2: when $\|\overline{\bm{Q}}-\bm{Q}^{\star}\|_{\infty} \leq 1/{\sqrt{1-\gamma}}$}

The analysis of Phase 2 follows by straightforwardly combining the analysis of Phase 1 and that of the synchronous counterpart in \cite{wainwright2019variance}. For the sake of brevity, we only sketch the main steps.

Following the proof idea of \citet[Section B.2]{wainwright2019variance}, we introduce an auxiliary vector $\Qhat$ which is the unique fix point to the following equation, which can be regarded as a population-level Bellman equation with proper reward perturbation, namely,
\begin{align}
	\Qhat = \bm r + \gamma \bm P(\Vhat - \Vbar) + \gamma \Ptil \Vbar.
	\label{eq:bellman-qhat}
\end{align}
Here, as usual, $\Vhat\in \mathbb{R}^{|\cS|}$ represents the value function corresponding to $\Qhat$. This can be viewed as a Bellman equation when the reward vector $\bm r$ is replaced by $\widetilde{\bm r} := \bm r + \gamma (\Ptil - \bm P )   \Vbar$.  
Repeating the arguments in the proof of \citet[Lemma 4]{wainwright2019variance} (except that we need to apply the measure concentration of $\widetilde{\bm{P}}$ in the manner performed in the proof of Lemma~\ref{lemma:bound-h0t} due to Markovian data), we reach
\begin{align}
	\label{eq:qhat-quality-phase2}
	\linf{\Qhat-\Qstar}\leq c'\sqrt{\frac{\log\frac{N|\cS||\cA|}{\delta}}{(1-\gamma)^{3}N\mumin}} \leq \varepsilon
	%\linf{\Qhat - \Qstar} \leq c' \Bigg\{ \sqrt{\frac{\log \frac{|\cS||\cA|}{\delta}}{(1-\gamma)^3 N\mumin}}
	%+ \frac{\log\frac{ |\cS||\cA|}{\delta}}{(1-\gamma)^2N\mumin} \Bigg\} + c'\frac{\linf{\Qhat - \Qstar}}{1-\gamma}
	%\sqrt{\frac{\log \frac{|\cS||\cA|}{\delta}}{N\mumin}}
\end{align}
with probability at least $1-\delta$ for some constant $c'>0$, 
provided that $N \geq (c')^2\frac{\log\frac{N|\cS||\cA|}{\delta}}{(1-\gamma)^{3}\varepsilon^2}$
and that $\|\overline{\bm{Q}}-\bm{Q}^{\star}\|_{\infty} \leq 1/{\sqrt{1-\gamma}}$. It is worth noting that $\Qhat$ only serves as a helper in the proof and is never explicitly constructed in the algorithm, as we don't have access to the probability transition matrix $\bm{P}$.

In addition, we claim that
\begin{align}
\label{eqn:bruch}
	\linf{\bm Q_{\tepoch} - \Qhat} \leq 
	%\frac{1}{8}\|\Qbar-\bm{Q}^{\star}\|_{\infty}+\frac{17}{8}\varepsilon.
	 \frac{\|\Qhat - \Qstar\|_{\infty}}{8} + \frac{\|\Qbar - \Qstar\|_{\infty}}{8} +\varepsilon .
\end{align}
Under this claim, the triangle inequality yields
\begin{align}
\|\bm{Q}_{\tepoch}-\bm{Q}^{\star}\|_{\infty} & \leq\|\bm Q_{\tepoch} - \Qhat\|_{\infty}+\|\Qhat-\bm{Q}^{\star}\|_{\infty}\leq\frac{1}{8}\|\Qbar-\bm{Q}^{\star}\|_{\infty}+\frac{9}{8}\|\Qhat-\bm{Q}^{\star}\|_{\infty}+\varepsilon \nonumber\\
 & \leq\frac{1}{8}\|\Qbar-\bm{Q}^{\star}\|_{\infty}+\frac{17}{8}\varepsilon,
	\label{eq:phase2-contraction}
\end{align}
where the last inequality follows from \eqref{eq:qhat-quality-phase2}.

\paragraph{Proof of the inequality~\eqref{eq:qhat-quality-phase2}.}
%\blue{
Suppose that
\begin{align}
\big|\widetilde{\bm r} - \bm r\big| = \gamma \big|(\Ptil - \bm P ) \Vbar\big| \le c\left\{\frac{1}{\sqrt{1-\gamma}}\bm 1 + \sqrt{\mathsf{Var}_{\bm P}(\bm V^{\star})}\right\}\sqrt{\frac{\log\frac{N|\cS||\cA|}{\delta}}{N\mumin}}, \label{eq:r-bound}
\end{align}
holds for some constant $c > 0$. 
By replacing Lemma 5 in the proof of \citet[Lemma 4]{wainwright2019variance} with this bound, we can arrive at \eqref{eq:qhat-quality-phase2} immediately.
%This can be derived in a similar way as the proof of Lemma~\ref{lemma:bound-h0t}.
In what follows, we demonstrate how to prove the bound \eqref{eq:r-bound}, which follows a similar argument as in the proof of Lemma~\ref{lemma:bound-h0t}.   
%}

Let us begin with the following triangle inequality:
\begin{align}
\big|(\Ptil - \bm P ) \Vbar\big| \le \big|(\Ptil - \bm P ) (\Vbar - \bm V^{\star})\big| + \big|(\Ptil - \bm P ) \bm V^{\star}\big|, 
	\label{eq:decompose-7293}
\end{align}
leaving us with two terms to control. 

\begin{itemize}

\item

Similar to~\eqref{eq:Hoeffding-Ptilde2}, by applying the Hoeffding inequality and taking the union bound over all $(s,a)\in \cS\times \cA$, 
we can control the first term on the right-hand side of \eqref{eq:decompose-7293} as follows:
\begin{align}
	\big\|(\widetilde{\bm{P}}-\bm{P})(\Vbar - \bm V^{\star})\big\|_{\infty} 
	\leq
	\max_{(s,a)\in \cS\times \cA} \sqrt{\frac{2\log\big( \frac{2N|\mathcal{S}||\mathcal{A}|}{\delta} \big)}{K_N(s,a)}}\big\|\Vbar - \bm V^{\star}\big\|_{\infty}
	\leq \sqrt{\frac{4\log\big( \frac{2N|\mathcal{S}||\mathcal{A}|}{\delta} \big)}{N \mumin (1-\gamma)}} \label{eq:Ptilde-Vbar-Vstar}
\end{align}
with probability at least $1-\delta$. 
Here, we have made use of the following property of this phase that $$\big\|\Vbar - \bm V^{\star}\big\|_{\infty} \le \|\overline{\bm{Q}}-\bm{Q}^{\star}\|_{\infty} \leq 1/{\sqrt{1-\gamma}}$$ and $K_N(s,a) \geq  N \mumin /2 $ for all $(s,a)$ (see Lemma~\ref{lemma:Bernstein-state-occupancy}).

\item

Next, we turn attention to the second term on the right-hand side of \eqref{eq:decompose-7293}, 
towards which we resort to the Bernstein inequality. 
Note that the $(s,a)$-th entry of $\big|(\Ptil - \bm P ) \bm V^{\star}\big|$ is given by
\begin{align}
	\Bigg|\frac{1}{K_{N}(s,a)}\sum_{i=1}^{K_{N}(s,a)} \big(\bm{P}_{t_i+1}(s,a) - \bm{P}(s, a)\big) \bm V^{\star}\Bigg|,  
\end{align}
where $K_N(s,a)$ denotes the total number of visits to $(s,a)$ during the first $N$ time instances (see also \eqref{eq:defn-Kt}).
In addition, let  
$t_i:= t_i(s,a)$ denote the time stamp when the trajectory visits ($s,a$) for the $i$-th time (see also \eqref{eq:defn-tk-sa}).
In view of our derivation for \eqref{eq:independence-Markov-chain}, the state transitions happening at times $t_1,t_2,\cdots,t_k$ (which are random) are independent for any given integer $k>0$.
It can be calculated that
\begin{subequations}
\begin{align}
\Big|\big(\bm{P}_{t_i+1}(s,a) - \bm{P}(s, a)\big) \bm V^{\star}\Big| &\le \frac{1}{1-\gamma}; \\
\mathsf{Var}\Bigg(\frac{1}{k} \sum_{i=1}^{k} \big(\bm{P}_{t_{i}+1}(s,a)-\bm{P}(s,a)\big)\bm{V}^{\star}\Bigg) &= \frac{1}{k} \mathsf{Var}_{\bm{P}(s,a)}\big(\bm{V}^{\star}\big).
\end{align}
\end{subequations}
Consequently, invoking the Bernstein inequality implies that with probability at least $1-\frac{\delta}{|\mathcal{S}||\mathcal{A}|}$, 
\[
	\left| \frac{1}{k} \sum_{i=1}^{k} \big(\bm{P}_{t_{i}+1}(s,a)-\bm{P}(s,a)\big)\bm{V}^{\star}\right|\leq\sqrt{\frac{4\log\big( \frac{2N|\mathcal{S}||\mathcal{A}|}{\delta} \big)}{k}\mathsf{Var}_{\bm{P}(s,a)}\big(\bm{V}^{\star}\big)} + \frac{4\log\big( \frac{2N|\mathcal{S}||\mathcal{A}|}{\delta} \big)}{3(1-\gamma)k}
\]
holds simultaneously for all $1\leq k\leq N$. 
Recognizing the bound $\frac{1}{2} N \mumin \le K_{N}(s,a) \le N$ and applying the union bound over all $(s,a)\in \cS\times \cA$ yield
\begin{align}
\big|(\Ptil - \bm P ) \bm V^{\star}\big| \le \sqrt{\frac{2\log\big( \frac{2N|\mathcal{S}||\mathcal{A}|}{\delta} \big)}{N \mumin}\mathsf{Var}_{\bm{P}}\big(\bm{V}^{\star}\big)} + \frac{8\log\big( \frac{2N|\mathcal{S}||\mathcal{A}|}{\delta} \big)}{3(1-\gamma)N \mumin}. \label{eq:Ptilde-Vstar}
\end{align}

\item
Finally, combining~\eqref{eq:Ptilde-Vbar-Vstar} and~\eqref{eq:Ptilde-Vstar} immediately establishes the claim \eqref{eq:r-bound}. 

\end{itemize}

%\eqref{eq:qhat-quality-phase2} immediately.

\paragraph{Proof of the inequality~\eqref{eqn:bruch}.} Recalling the 
variance-reduced update rule~\eqref{eq:vr-Q-learning-matrix-notation} and using the Bellman-type equation \eqref{eq:bellman-qhat}, we obtain
\begin{align}
\notag \Delhat_t \defn \bm{Q}_{t} - \Qhat
	&= \big(\bm{I}-\bm{\Lambda}_{t}\big)(\bm{Q}_{t-1}- \Qhat) + 
	\bm{\Lambda}_{t}\left( \bm{r}+\gamma\bm{P}_{t}(\bm{V}_{t-1} - \overline{\bm{V}}) + \gamma \widetilde{\bm{P}} \overline{\bm{V}} 
	- \bm r - \gamma \bm P(\Vhat - \Vbar) - \gamma \Ptil \Vbar \right)\\
\notag	&= \big(\bm{I}-\bm{\Lambda}_{t}\big)(\bm{Q}_{t-1}- \Qhat) + 
	\bm{\Lambda}_{t} \left(\gamma\bm{P}_{t}(\bm{V}_{t-1} - \overline{\bm{V}}) 
	- \gamma \bm P(\Vhat - \Vbar)\right)\\
	&= \big(\bm{I}-\bm{\Lambda}_{t}\big) \Delhat_{t-1} + 
	\gamma\bm{\Lambda}_{t} \left((\bm P_t - \bm P)(\Vhat-\Vbar) + \bm{P}_{t}(\bm{V}_{t-1} - \Vhat) 
	\right).
\end{align}
%
% By comparing it to \eqref{eq:Delta-t-identity2}, the role of variance reduction becomes more clear when $\widehat{}
Adopting the same expansion as before (see \eqref{eq:defn-beta1-beta3}), we arrive at
\begin{align*}
	\Delhat_t 
	%&= \big(\bm{I}-\bm{\Lambda}_{t}\big)\Delhat_{t-1} + 
	%\gamma\bm{\Lambda}_{t} \left((\bm P_t - \bm P)(\Vhat-\Vbar) + \bm{P}_{t}(\bm{V}_{t-1} - \Vhat) 
	%\right) \\
%
	&= \underset{=: \bm{\vartheta}_{1,t}}{\underbrace{\gamma\sum_{i=1}^{t}\prod_{j=i+1}^{t}\big(\bm{I}-\bm{\Lambda}_{j}\big)\bm{\Lambda}_{i}\big(\bm{P}_{i}-\bm{P}\big)(\Vhat - \Vbar)}} 
	 + 
	 \underset{=: \bm{\vartheta}_{2,t}}{\underbrace{\gamma\sum_{i=1}^{t}\prod_{j=i+1}^{t}\big(\bm{I}-\bm{\Lambda}_{j}\big)\bm{\Lambda}_{i}\bm{P}_{i}\big(\bm{V}_{i-1}-\Vhat\big)}}
	 +
	 \underset{=: \bm{\vartheta}_{3,t}}{\underbrace{\prod_{j=1}^{t}\big(\bm{I}-\bm{\Lambda}_{j}\big)\Delhat_{0}}}.
	% \label{eq:defn-beta0-beta1-beta3}
\end{align*}
Inheriting the results in Lemma~\ref{lemma:control-beta1} and Lemma~\ref{lemma:control-beta3}, we can demonstrate that, with probability at least $1-2\delta$, 
\begin{align*}
	\big|\bm{\vartheta}_{1,t}\big| &\leq 
	c \gamma \|\Vhat - \Vbar\|_{\infty} \sqrt{\eta\log\Big(\frac{|\mathcal{S}||\mathcal{A}|\tepoch}{\delta}\Big)}
	\one ; \\
	\left|\bm{\vartheta}_{3,t}\right| & \leq
\begin{cases}
(1-\eta)^{\frac{1}{2}t\mumin}\|\Delhat_{0}\|_{\infty}\bm{1}, \quad & \text{if }\tcover\leq t\leq\tepoch,\\[2mm]
\|\Delhat_{0}\|_{\infty}\bm{1}, & \text{if }t<\tcover.
\end{cases}
\end{align*}
Repeating the same argument as for Theorem~\ref{thm:intermediate}, we reach
\begin{align*}
%\label{eqn:beethoven7}
	\| \Delhat_t \|_{\infty} \leq (1- \rho )^{k} \frac{\|\Qhat - \Qbar\|_{\infty}}{1-\gamma} + \frac{c \gamma}{1-\gamma} 
	 \|\Vhat - \Vbar\|_{\infty} \sqrt{\eta\log\Big(\frac{|\mathcal{S}||\mathcal{A}|\tepoch}{\delta}\Big)}
	 +\varepsilon
\end{align*}
for some constant $c>0$, where $k = \max\{0,\big\lfloor \frac{t - \tth}{\tcover} \big\rfloor\}$ with $\tth$ defined in \eqref{eqn:tth-vr}. 

By taking $\eta=c_{5}\min\big\{\frac{(1-\gamma)^{2}}{\gamma^2\log\frac{|\mathcal{S}||\mathcal{A}|\tepoch}{\delta}},\frac{1}{\muepo}\big\}$ for some sufficiently small constant $c_5>0$ and ensuring that $$\tepoch\geq\tth+\tcover+\frac{c_{6}}{(1-\gamma)\eta\mumin}\log\frac{1}{(1-\gamma)^{2}}$$ for some large constant $c_6>0$, we obtain
\begin{align*}
	 \| \Delhat_{\tepoch} \|_{\infty} \leq \frac{\|\Qhat - \Qbar\|_{\infty}}{8} +\varepsilon 
	\leq \frac{\|\Qhat - \Qstar\|_{\infty}}{8} + \frac{\|\Qbar - \Qstar\|_{\infty}}{8} +\varepsilon  ,
	% \leq \frac{\|\Qbar - \Qstar\|_{\infty}}{8}  + \frac{9}{8}\,\varepsilon,
\end{align*}
where the last line follows by the triangle inequality.

\subsection{How many epochs are needed?}

We are now ready to pin down how many epochs are needed to achieve $\varepsilon$-accuracy.  
\begin{itemize}
	\item In Phase 1, the contraction result \eqref{eqn:base-step-richter} indicates that, if the algorithm is initialized with $\bQ_0= \bm{0}$ at the very beginning, then it takes at most 
		\[
			\log_2 \Bigg( \frac{\|\bm{Q}^{\star}\|_{\infty}}{ \max \big\{ \varepsilon, \frac{1}{\sqrt{1-\gamma}}\big\} } \Bigg) 
			\leq \log_2 \Big( \frac{1}{\sqrt{1-\gamma}} \Big) + \log_2 \Big(\frac{1}{\varepsilon(1-\gamma)}\Big)
		\]
		epochs to yield $ \|\overline{\bm{Q}}-\Qstar \|_{\infty}\leq \max\{\frac{1}{\sqrt{1-\gamma}},\varepsilon\}$ (so as to enter Phase 2).  Clearly, if the target accuracy level $\varepsilon > \frac{1}{\sqrt{1-\gamma}}$, then the algorithm terminates in this phase.

	\item   Suppose now that the target accuracy level $\varepsilon \leq  \frac{1}{\sqrt{1-\gamma}}$. 
		Once the algorithm enters Phase 2, the dynamics can be characterized by \eqref{eq:phase2-contraction}. Given that $\overline{\bm{Q}}$ is also the last iterate of the preceding epoch, the property \eqref{eq:phase2-contraction} provides a recursive relation across epochs. Standard recursive analysis thus reveals that: within at most 
		\[
			c_7 \log\Big( \frac{1}{\varepsilon\sqrt{1-\gamma}} \Big) \leq c_7 \log\Big( \frac{1}{\varepsilon (1-\gamma)} \Big)
		\]
		epochs (with $c_7>0$ some constant), we are guaranteed to attain an $\ell_{\infty}$ estimation error at most $3\varepsilon$. 
\end{itemize}
To summarize, a total number of $O\big( \log \frac{1}{\varepsilon (1-\gamma)} + \log \frac{1}{ 1-\gamma} \big)$ epochs are sufficient for our purpose. 
This concludes the proof.

% Now Step 1 provides us with an initialization that is at least $\frac{1}{2\sqrt{1-\gamma}}$-close to $\Qstar$, so it is sufficient to argue that after $\log(1/\epsilon)$ many such epoches, one can achieve an $\epsilon$-accurate estimate of $\Qstar.$  

% To summarize, in the current epoch, we start with with $\bm Q_0 = \Qbar$, compute $T$ steps and arrive at $\bm Q_{T}$ which satisfies the following
% \begin{align}
% \label{eqn:Shostakovich}
% 	\|\bm Q_{T} - \Qstar \|_{\infty} \leq 
% 	\frac{1}{2} \max \Big\{\frac{1}{\sqrt{1-\gamma}},~\|\Qbar - \Qstar\|_{\infty} \Big\}.
% \end{align}
% As a result, if we complete $\log(\frac{1}{\sqrt{1-\gamma}})/\log(2)$ many epoches, by a simple induction 
% argument, it is guaranteed that the outcome estimator satisfies 
% \begin{align}
% 	\|\Qbar^0 - \Qstar \|_{\infty} \leq \frac{1}{2\sqrt{1-\gamma}}.
% \end{align}
% In words, we are able to locate a good initialization point $\Qbar^0$ with error at most $\frac{1}{2\sqrt{1-\gamma}}$ and the total sampling budget equals to 
% \begin{align}
% 	C\left(\frac{\tmix}{1-\gamma} + \frac{1}{(1-\gamma)^3}\right) 
% 	\frac{1}{\mumin} \log^2\left(\frac{1}{1-\gamma}\right)\log\left(\frac{|\mathcal{S}||\mathcal{A}|T}{\delta}\right).
% \end{align}

% \vspace{3cm}

\section{Discussion}
\label{sec:discussion}

This work develops a sharper finite-sample analysis of the classical asynchronous Q-learning algorithm, highlighting and refining its dependency on intrinsic features of the Markovian trajectory induced by the behavior policy. Our sample complexity bound strengthens the state-of-the-art result by an order of at least $|\cS||\cA|$. A variance-reduced variant of asynchronous Q-learning is also analyzed, exhibiting improved scaling with the effective horizon $\frac{1}{1-\gamma}$. 

Our findings and the analysis framework developed herein suggest a couple of directions for future investigation. For instance, our  improved sample complexity of asynchronous Q-learning has a dependence of $\frac{1}{(1-\gamma)^5}$ on the effective horizon, which is inferior to its model-based counterpart. 
In the synchronous setting, \cite{li2021q,li2021tightening} recently demonstrated Q-learning has a dependence of $\frac{1}{(1-\gamma)^4}$, which is tight up to logarithmic factors. In light of this development, it would be important to determine the exact scaling for the asynchronous setting, which is left as future work.
In addition, it would be interesting to see whether the techniques developed herein can be exploited towards understanding model-free algorithms with more sophisticated exploration schemes \cite{dann2015sample}. Finally, asynchronous Q-learning on a single Markovian trajectory is closely related to coordinate descent with coordinates selected according to a Markov chain; one would naturally ask whether our analysis framework can yield improved convergence guarantees for general Markov-chain-based optimization algorithms \citep{sun2020markov,doan2020convergence}.

\section*{Acknowledgements}

G.~Li and Y.~Gu are supported in part by the grant NSFC-61971266.
Y.~Wei is supported in part by the NSF grant CCF-2007911, DMS-2015447 and CCF-2106778. 
Y.~Chi is supported in part by the grants ONR N00014-18-1-2142 and N00014-19-1-2404, ARO W911NF-18-1-0303, NSF CCF-1806154, CCF-2007911 and CCF-2106778.
Y.~Chen is supported in part by the grants AFOSR YIP award FA9550-19-1-0030,
ONR N00014-19-1-2120, ARO YIP award W911NF-20-1-0097, ARO W911NF-18-1-0303, NSF CCF-2106739, CCF-1907661, IIS-1900140,  and IIS-2100158. 
We thank Shicong Cen, Chen Cheng and Cong Ma for numerous discussions about reinforcement learning.

%%%%%%%%%%%%%%%%%%%%%%%%%%%%%%%%%%%%%%%%%%%%%%%%%%%%%%%%%%%%%%%%%%%%%%%%%%%%%%%%%%%%%%%%%%%%%%

\appendix

\section{Preliminaries on Markov chains}
\label{sec:preliminary}

In this section, we gather some basic facts about Markov chains. Before proceeding, we remind the readers of some notation. 
For any two probability distributions $\mu$ and $\nu$, denote by $d_{\mathsf{TV}} (\mu, \nu)$ the total variation distance between $\mu$ and $\nu$ (cf.~\eqref{eq:TV}). 
Recall the definition of uniform ergodicity in Section~\ref{sec:notation}. 
For any {\em time-homogeneous} and {\em uniformly ergodic} Markov chain $(X_0, X_1, X_2, \cdots)$ with transition kernel $P$, finite state space $\mathcal{X}$ and stationary distribution $\mu$, we let $P^t(\cdot\,|\,x)$ denote the distribution of $X_t$ conditioned on $X_0=x$. Then the mixing time $\tmix$ of this Markov chain is defined by
\begin{subequations}
\label{defn:mixing-time-all}
\begin{align}
	\tmix(\epsilon) &:= \min \Big\{ t ~\Big|~  \max_{x\in \mathcal{X}} d_{\mathsf{TV}}\big( P^t(\cdot \,|\, x), \mu \big) \leq \epsilon \Big\}; \label{defn:mixing-time-epsilon}\\
	\tmix &:= \tmix(1/4). \label{defn:mixing-time}
\end{align}
\end{subequations}

\subsection{Concentration of empirical distributions of Markov chains}
\label{sec:concentration-MC}
We first record a result concerning the concentration of measure of the empirical distribution of a uniformly ergodic Markov chain, which makes clear the role of the mixing time.
\begin{lemma}
	\label{lemma:Bernstein-state-occupancy}
	Consider the above-mentioned Markov chain. For any $0<\delta<1$, if $t \geq  \frac{443 \tmix}{\mumin}\log\frac{4|\mathcal{X}|}{\delta}$, then %with probability at least $1-2\delta$ one has
	%
	%\blue{
	\begin{align}
		\forall y \in \mathcal{X}:\quad  \mathbb{P}_{X_{1}=y}\Bigg\{ \exists x\in\mathcal{X}: \left|\sum_{i= 1}^{t}  \ind\{X_{i}=x\} - t \mu(x)\right| \geq \frac{1}{2}t \mu(x)\Bigg\} 
		\leq \delta.
	\end{align}
	%}
	%
\end{lemma}
\begin{proof}
To begin with, consider the scenario when $X_{1}\sim\mu$, namely, when $X_{1}$ follows
the stationary distribution of the chain. Then \citet[Theorem~3.4]{paulin2015concentration}
tells us that: for any given $x\in\mathcal{X}$ and any $\tau\geq 0$, 
%
%\blue{
\begin{align}
	\mathbb{P}_{X_{1}\sim\mu}\left\{ \left|\sum_{i=1}^{t}\ind\{X_{i}=x\} - t\mu(x)\right|\geq \tau\right\} 
	&\leq 2 \exp\left(-\frac{\tau^{2}\gamma_{\mathsf{ps}}}{8(t+1/\gamma_{\mathsf{ps}})\mu(x)+20\tau}\right) \nonumber\\
	&\leq 2 \exp\left(-\frac{\tau^{2}/\tmix}{16(t+2\tmix)\mu(x)+40\tau}\right),\label{eq:paulin-bound}
\end{align}
%}
%
where $\gamma_{\mathsf{ps}}$ stands for the so-called {\em pseudo spectral gap} as
defined in \citet[Section~3.1]{paulin2015concentration}. Here, the first inequality relies on the fact $\mathsf{Var}_{X_i\sim \mu}[ \ind\{X_{i}=x\} ]=\mu(x)(1-\mu(x))\leq \mu(x)$, while the last
inequality results from the fact $\gamma_{\mathsf{ps}}\geq1/(2\tmix)$ that holds for uniformly ergodic chains (cf.~\citet[Proposition~3.4]{paulin2015concentration}). 
Consequently, for any $t\geq\tmix$ and any $\tau\geq 0$, one can continue the bound (\ref{eq:paulin-bound}) to obtain
\begin{align*}
	\eqref{eq:paulin-bound} & 
	\leq 2 \exp\left(-\frac{\tau^{2}}{48t\mu(x)\tmix+40\tau\tmix}\right)
	\leq 2 \max\left\{ \exp\left(-\frac{\tau^{2}}{96t\mu(x)\tmix}\right),\exp\left(-\frac{\tau}{80\tmix}\right)\right\} 
	\leq  \frac{\delta}{ |\mathcal{X}| },
\end{align*}
provided that $$\tau\geq\max\left\{ 10\sqrt{t\mu(x)\tmix\log\frac{2|\mathcal{X}|}{\delta}},\,80\tmix\log\frac{2|\mathcal{X}|}{\delta}\right\} .$$ As a result, by taking $\tau=\frac{10}{21}t\mu(x)$ and applying the
union bound, we reach
%
%\blue{
\begin{align}
	& \mathbb{P}_{X_{1}\sim\mu}\left\{ \exists x\in\mathcal{X}:\left|\sum_{i=1}^{t}\ind\{X_{i}=x\} - t\mu(x)\right|\geq\frac{10}{21}t\mu(x)\right\} \notag\\
	& \qquad\qquad \leq\sum_{x\in\mathcal{X}}\mathbb{P}_{X_{1}\sim\mu}\left\{ \left|\sum_{i=1}^{t}\ind\{X_{i}=x\} - t\mu(x)\right|\geq\frac{10}{21}t\mu(x)\right\} \leq\delta,\label{eq:paulin-bound-1}
\end{align}
%}
%
as long as $\frac{10}{21}t\mu(x)\geq\max\big\{10\sqrt{t\mu(x)\tmix\log\frac{2|\mathcal{X}|}{\delta}},\,80\tmix\log\frac{2|\mathcal{X}|}{\delta}\big\}$ for all $x\in \mathcal{X}$,
or equivalently, when $$t\geq\frac{441\tmix}{\mumin}\log\frac{2|\mathcal{X}|}{\delta} \qquad
	\text{with }\mumin:=\min_{x\in\mathcal{X}}\mu(x).$$

	Next, we seek to extend the above result to the more general case when $X_1$ takes an arbitrary state $y\in \mathcal{X}$.  From the definition of $\tmix(\cdot)$ (cf.~\eqref{defn:mixing-time-epsilon}), we know that
	\begin{align}
		d_{\mathsf{TV}}\Big(  \sup_{y\in \mathcal{X}} P^{\tmix(\delta)}(\cdot\,|\, y) , \, \mu \Big) \leq \delta.
		\label{eq:TV-y-small-than-delta}
	\end{align}
	This taken together with the definition of $d_{\mathsf{TV}}$ (cf.~\eqref{eq:TV}) reveals that: for any event $\mathcal{B}$ belonging to the $\sigma$-algebra generated by $\{X_\tau\}_{\tau \geq \tmix(\delta)}$, 
	%that can be fully determined by $\{X_\tau\}_{\tau \geq \tmix(\delta)}$, 
	one has 
\begin{align}
 & \big|\mathbb{P}\{\mathcal{B}\mid X_{1}=y\}-\mathbb{P}\{\mathcal{B}\mid X_{1}\sim\mu\}\big| \notag\\
 & =\left|\sum_{s\in\mathcal{S}}\mathbb{P}\{\mathcal{B}\mid X_{\tmix(\delta)}=s\}\mathbb{P}\{X_{\tmix(\delta)}=s\mid X_{1}=y\}-\sum_{s\in\mathcal{S}}\mathbb{P}\{\mathcal{B}\mid X_{\tmix(\delta)}=s\}\mathbb{P}\{X_{\tmix(\delta)}=s\mid X_{1}\sim\mu\}\right| \notag\\
 & \leq\max\left\{ \sum_{s\in\mathcal{S}_{+}}\Big[\mathbb{P}\{X_{\tmix(\delta)}=s\mid X_{1}=y\}-\mathbb{P}\{X_{\tmix(\delta)}=s\mid X_{1}\sim\mu\}\Big],\right. \notag\\
 & \qquad\qquad\left.\sum_{s\in\mathcal{S}_{-}}\Big[\mathbb{P}\{X_{\tmix(\delta)}=s\mid X_{1}\sim \mu\}-\mathbb{P}\{X_{\tmix(\delta)}=s\mid X_{1} = y\}\Big]\right\} \notag\\
 & \leq\sup_{A\subseteq\mathcal{S}}\Big|\mathbb{P}\{X_{\tmix(\delta)}\in A\mid X_{1}=y\}-\mathbb{P}\{X_{\tmix(\delta)}\in A\mid X_{1}\sim\mu\}\Big| \leq \delta,
	\label{eq:TV-B-y-mu-157}
\end{align}
where we define
\begin{align*}
\mathcal{S}_{+} & \defn\left\{ s\in\mathcal{S}:\mathbb{P}\{X_{\tmix(\delta)}=s\mid X_{1}=y\}\,>\,\mathbb{P}\{X_{\tmix(\delta)}=s\mid X_{1}\sim\mu\}\right\} ;\\
\mathcal{S}_{-} & \defn\left\{ s\in\mathcal{S}:\mathbb{P}\{X_{\tmix(\delta)}=s\mid X_{1}=y\}\,<\,\mathbb{P}\{X_{\tmix(\delta)}=s\mid X_{1}\sim\mu\}\right\} .
\end{align*}
Here, the last inequality in \eqref{eq:TV-B-y-mu-157} follows from the inequality \eqref{eq:TV-y-small-than-delta} and the definition \eqref{eq:TV} of the total-variation distance. 
As a consequence, one obtains
	%
	%\blue{
	\begin{align}
		& \sup_{y\in \mathcal{X}}  \mathbb{P}_{X_{1}=y}\left\{ \exists x\in\mathcal{X}: \left|\sum_{i=\tmix(\delta)}^{t}\ind\{X_{i}=x\}-\big(t-\tmix(\delta) \big)\mu(x)\right|\geq\frac{10}{21}\big(t-\tmix(\delta)\big)\mu(x)\right\} \nonumber\\
		& \qquad \leq \mathbb{P}_{X_{1}\sim\mu}\left\{ \exists x\in\mathcal{X}: \left|\sum_{i=\tmix(\delta)}^{t}\ind\{X_{i}=x\}-\big(t-\tmix(\delta)\big)\mu(x)\right|\geq\frac{10}{21}\big(t-\tmix(\delta)\big)\mu(x)\right\} + \delta 
		 \leq 2\delta,
		 \label{eq:sup-y-X-P-789}
	\end{align}
	%}
	%
	with the proviso that $t \geq \tmix(\delta) + \frac{441\tmix}{\mumin}\log\frac{2|\mathcal{X}|}{\delta}$.

	To finish up, we recall from \citet[Section~1.1]{paulin2015concentration} that $\tmix(\delta) \leq 2 \tmix \log \frac{2}{\delta}$.  
	% which together with the above constraint on $t$ necessarily implies that $\frac{11}{21}(t- \tmix(\delta)) \geq \frac{1}{2} t$. 
	% In the meantime, 
	Consequently, if   $t \geq  \frac{443\tmix}{\mumin}\log\frac{2|\mathcal{X}|}{\delta} \geq \tmix(\delta) + \frac{441\tmix}{\mumin}\log\frac{2|\mathcal{X}|}{\delta}$, then one has 
	\begin{align*}
		\frac{1}{2}\big(t-\tmix(\delta)\big)\mu(x)-\tmix(\delta)\geq\frac{441\tmix}{2}\log\frac{2|\mathcal{X}|}{\delta}\geq 100\tmix(\delta), \\
		\Longrightarrow\quad\frac{1}{2}\big(t-\tmix(\delta)\big)\mu(x)-\tmix(\delta)\geq\frac{10}{21}\big(t-\tmix(\delta)\big)\mu(x).
	\end{align*}
	These taken together lead to
	%\blue{
	\begin{align*}
		& \sup_{y \in \mathcal{X}} \, \mathbb{P}_{X_{1}=y}\Bigg\{ \exists x\in\mathcal{X}: \left|\sum_{i=1}^{t}\ind\{X_{i}=x\} - t\mu(x)\right|\geq\frac{1}{2}t\mu(x)\Bigg\}  \nonumber\\
		& \leq 
		\sup_{y \in \mathcal{X}} \, \mathbb{P}_{X_{1}=y}\Bigg\{ \exists x\in\mathcal{X}: \left|\sum_{i=\tmix(\delta)}^{t}\ind\{X_{i}=x\}-(t-\tmix(\delta))\mu(x)\right|\geq\frac{1}{2}\big(t-\tmix(\delta) \big)\mu(x)-\tmix(\delta)\Bigg\} \notag\\
		& \leq 
		\sup_{y \in \mathcal{X}} \, \mathbb{P}_{X_{1}=y}\Bigg\{ \exists x\in\mathcal{X}: \left|\sum_{i=\tmix(\delta)}^{t}\ind\{X_{i}=x\}-(t-\tmix(\delta))\mu(x)\right|\geq\frac{10}{21}\big(t-\tmix(\delta) \big)\mu(x)\Bigg\} 
		\leq 2\delta,
	\end{align*}
	%}
	%
	where the last inequality results from  \eqref{eq:sup-y-X-P-789}. Replacing $\delta$ with $\delta/2$ thus concludes the proof. 	 
\end{proof}

\begin{comment}
\yxc{Change these later}

We begin by recording a few elementary facts about $\bm{P}^{\pi}$ and $\bm{P}_{\pi}$ (see definitions in \eqref{eqn:ppivq}). 
These are standard results and we omit the proofs for brevity. 

\begin{lemma}\label{lemma:basic-properties-I-gammaP}For any policy $\pi$, any probability
	transition matrix $\bm{P}\in \mathbb{R}^{|\mathcal{S}||\mathcal{A}|\times |\mathcal{S}|}$ and any $0<\gamma<1$, one has 

\begin{itemize}

	\item[(a)] $(\bm{I}-\gamma\bm{P}_{\pi})^{-1}=\sum_{i=0}^{\infty}(\gamma\bm{P}_{\pi})^{i}$;

\item[(b)] All entries of the matrix $(\bm{I}-\gamma\bm{P}_{\pi})^{-1}$
are non-negative;

\item[(c)] $\|(\bm{I}-\gamma\bm{P}_{\pi})^{-1}\|_{1}\leq1/(1-\gamma)$; 

\item[(d)] $(1-\gamma)(\bm{I}-\gamma\bm{P}_{\pi})^{-1}\bm{1}=\bm{1}$; 

\item[(e)] For any non-negative vectors $\bm{0}\leq\bm{r}_{1}\leq\bm{r}_{2}$ of compatible dimension,
one has $\bm{0}\leq(\bm{I}-\gamma\bm{P}_{\pi})^{-1}\bm{r}_{1}\leq(\bm{I}-\gamma\bm{P}_{\pi})^{-1}\bm{r}_{2}$. 

\end{itemize}
	The above results continue to hold if $\bm{P}_{\pi}$ is replaced by $\bm{P}^{\pi}$. 
\end{lemma}

\end{comment}

\subsection{Connection between the mixing time and the cover time}
\label{sec:connection-tcover-tmix}

Lemma~\ref{lemma:Bernstein-state-occupancy} combined with the definition \eqref{eq:defn-cover-time} immediately reveals the following upper bound on the cover time: 
\begin{align}
	\tcovertime = O\Big( \frac{\tmix}{\mumin} \log |\mathcal{X}| \Big). 
	\label{eq:cover-time-UB-mixing}
\end{align}
In addition, while a general matching converse bound (namely, ${\tmix}/{\mumin}=\widetilde{O}(\tcovertime)$) is not available, we can come up with some special examples for which the bound \eqref{eq:cover-time-UB-mixing} is provably tight.

\begin{example} 
\label{example:cover-time}
Consider a time-homogeneous Markov chain with state space $\mathcal{X}:=\{1,\cdots,|\mathcal{X}|\}$
and probability transition matrix
\begin{align}
\bm{P}=\Big(1-\frac{q(k+1)}{2}\Big)\bm{I}_{|\mathcal{X}|}+\frac{q}{|\mathcal{X}|}\left[\begin{array}{cc}
k\bm{1}_{|\mathcal{X}|}\bm{1}_{|\mathcal{X}|/2}^{\top} & \bm{1}_{|\mathcal{X}|}\bm{1}_{|\mathcal{X}|/2}^{\top}\end{array}\right] \in \mathbb{R}^{|\mathcal{X}| \times |\mathcal{X}|}
\label{eq:construction-P}
\end{align}
for some quantities $q>0$ and $k\geq 1$. 
Suppose $q(k+1)<2$ and $|\mathcal{X}|\geq 3$. Then this chain obeys 
\begin{align}
	\tcovertime\geq \frac{\tmix}{\big(8\log2+4\log\frac{1}{\mumin}\big)\mumin}.
	\label{eq:cover-time-LB}
\end{align}
\end{example}
%
%\begin{remark}
	With the lower bound \eqref{eq:cover-time-LB} in place, 
	we conclude that the upper bound \eqref{eq:cover-time-UB-mixing} is, in general, nearly un-improvable (up to some logarithmic factor). 
	
	\begin{remark}
	We shall take a moment to briefly discuss the key design rationale behind Example~\ref{example:cover-time}. 
		Let us partition the state space into two halves, denoted respectively by $\mathcal{X}_1$ and $\mathcal{X}_2$. From every state $s\in \mathcal{X}$, it is much easier to transition into the first half $\mathcal{X}_1$ rather than the second half $\mathcal{X}_2$. 
		This leads to two properties: (i) the stationary distribution of any state in $\mathcal{X}_2$ is much lower than that of a state in $\mathcal{X}_1$; (ii) the cover time also increases as the stationary distribution w.r.t.~$\mathcal{X}_2$ decreases, given that it becomes more difficult to traverse the second half. As a result, we can guarantee that $\tcovertime$ is proportional to $\mumin$ through this type of designs. 
		On the other hand, the example is also constructed in a way such that all states are ``lazy'', meaning that they are more inclined to stay unchanged rather than moving to a different state. The level of laziness clearly controls how fast the Markov chain mixes, as well as how long it takes to cover all states. This in turn allows one to ensure that $\tcovertime$ is proportional to $\tmix$. More details can be found in the proof below. 	
	\end{remark}

%\end{remark}
%
%\begin{remark}
	
%\end{remark}

\begin{proof} As can be easily verified,  this chain is reversible, whose stationary distribution vector $\bm{\mu}\in\mathbb{R}^{|\mathcal{X}|}$
obeys
\[
\bm{\mu}=\frac{2}{(k+1)|\mathcal{X}|}\left[\begin{array}{c}
k\bm{1}_{|\mathcal{X}|/2}\\
\bm{1}_{|\mathcal{X}|/2}
\end{array}\right].
\]
As a result, the minimum state occupancy probability of the stationary distribution is given by
\begin{equation}
\mumin:=\min_{1\leq x\leq|\mathcal{X}|}\mu_{x}=\frac{2}{(k+1)|\mathcal{X}|}.\label{eq:mumin-q-expression}
\end{equation}
In addition, the reversibility of this chain implies that the matrix
$\bm{P}^{\mathrm{d}}:=\bm{D}^{\frac{1}{2}}\bm{P}\bm{D}^{-\frac{1}{2}}$
with $\bm{D}:=\mathrm{diag}\left[\bm{\mu}\right]$ is symmetric and
has the same set of eigenvalues as $\bm{P}$ \citep{bremaud2013markov}. A little algebra yields
\begin{align*}
\bm{P}^{\mathrm{d}} & =\Big(1-\frac{q(k+1)}{2}\Big)\bm{I}_{|\mathcal{X}|}+\frac{q}{|\mathcal{X}|}\left[\begin{array}{cc}
k\bm{1}_{|\mathcal{X}|/2}\bm{1}_{|\mathcal{X}|/2}^{\top} & \sqrt{k}\bm{1}_{|\mathcal{X}|/2}\bm{1}_{|\mathcal{X}|/2}^{\top}\\
\sqrt{k}\bm{1}_{|\mathcal{X}|/2}\bm{1}_{|\mathcal{X}|/2}^{\top} & \bm{1}_{|\mathcal{X}|/2}\bm{1}_{|\mathcal{X}|/2}^{\top}
\end{array}\right],
\end{align*}
allowing us to determine the eigenvalues $\{\lambda_{i}\}_{1\leq i\leq|\mathcal{X}|}$ as follows
\[
\lambda_{1}=1\qquad\text{and}\qquad\lambda_{i}=1-\frac{q(k+1)}{2} > 0 ~~ (i\geq2).
\]

We are now ready to establish the lower bound on the cover
time. First of all, the well-known
connection between the spectral gap and the mixing time gives \citet[Proposition 3.3]{paulin2015concentration}
\begin{equation}
\tmix\leq\frac{2\log2+\log\frac{1}{\mumin}}{2(1-\lambda_{2})}=\frac{2\log2+\log\frac{1}{\mumin}}{q(k+1)}.\label{eq:tmix-UB-q}
\end{equation}
In addition, let $(x_{0},x_{1},\cdots)$ be the corresponding
Markov chain, and assume that $x_{0}\sim\mu$, where $\mu$ stands
for the stationary distribution. Consider the last state --- denoted by $|\mathcal{X}|$, which enjoys the minimum state occupancy probability $\mumin$. 
For any integer $t>0$ one has
\begin{align*}
\mathbb{P}\left\{ x_{l}\neq|\mathcal{X}|,\text{ }\forall\, 0\leq l\leq t\right\}  
 & \overset{(\mathrm{i})}{=}\mathbb{P}\left\{ x_{0}\neq|\mathcal{X}|\right\} 
 \prod_{l=1}^{t}\mathbb{P}\Big\{ x_{l}\neq|\mathcal{X}| \,\Big| \,  x_{0}\neq|\mathcal{X}|,\cdots,x_{l-1}\neq|\mathcal{X}| \Big\} \\
 & \overset{(\mathrm{ii})}{\geq}\mathbb{P}\left\{ x_{0}\neq|\mathcal{X}|\right\} \prod_{l=1}^{t}\min_{j:j\neq|\mathcal{X}|}\mathbb{P}\big\{ x_{l}\neq|\mathcal{X}| \,\big| \, x_{l-1}=j\big\} \\
 & \overset{(\mathrm{iii})}{=} \Big(1-\frac{2}{(k+1)|\mathcal{X}|}\Big) \left(1-\frac{q}{|\mathcal{X}|}\right)^{t}\\
 & \overset{(\mathrm{iv})}{\geq} \Big(1-\frac{2}{(k+1)|\mathcal{X}|}\Big)\left(1-\frac{2qt}{|\mathcal{X}|}\right),
\end{align*}
where (i) follows from the chain rule, (ii)
relies on the Markovian property, (iii) results from the construction \eqref{eq:construction-P}, and (iv) holds as long
as $\frac{q}{|\mathcal{X}|}t<\frac{1}{2}$. Consequently, if $|\mathcal{X}|\geq3$
and if $t<\frac{|\mathcal{X}|}{8q}$, then one necessarily has 
\begin{align*}
\mathbb{P}\left\{ x_{l}\neq|\mathcal{X}|,~\forall \,0\leq l\leq t\right\}  & \geq\Big(1-\frac{2}{(k+1)|\mathcal{X}|}\Big)\left(1-\frac{2qt}{|\mathcal{X}|}\right)>\frac{1}{2}.
\end{align*}
This taken collectively with the definition of $\tcovertime$ (cf.~\eqref{eq:defn-cover-time}) reveals that
\[
\tcovertime\geq\frac{|\mathcal{X}|}{8q}\geq\frac{\tmix}{\big(8\log2+4\log\frac{1}{\mumin}\big)\mumin},
\]
where the last inequality is a direct consequence of (\ref{eq:mumin-q-expression})
and (\ref{eq:tmix-UB-q}). 
\end{proof}

\section{Proofs of technical lemmas}

\subsection{Proof of Lemma~\ref{lemma:control-beta1}}
\label{sec:proof-lemma-control-beta1}

% To simplify presentation, define
%
% \begin{align}
% 	\bm{\xi}_{t}:=\big(\bm{P}_{t}-\bm{P}\big)\bm{V}^{\star}. 
% 	\label{defn:xi-t}
% \end{align}i%
Fix any state-action pair $(s,a)\in\mathcal{S}\times\mathcal{A}$, and let
us look at $\bm{\beta}_{1,t}(s,a)$, namely, the $(s,a)$-th entry of
\[
	\bm{\beta}_{1,t}= \gamma\sum_{i=1}^{t}\prod_{j=i+1}^{t}\big(\bm{I}-\bm{\Lambda}_{j}\big)\bm{\Lambda}_{i}\big(\bm{P}_{i}-\bm{P}\big)\bm{V}^{\star}.
\]
For convenience of presentation, we abuse the notation
to let $\bm{\Lambda}_{j}(s,a)$ denote the $(s,a)$-th diagonal entry
of the diagonal matrix $\bm{\Lambda}_{j}$, and $\bm{P}_{t}(s,a)$
(resp.~$\bm{P}(s,a)$) the $(s,a)$-th row of $\bm{P}_{t}$ (resp.~$\bm{P}$).
In view of the definition (\ref{eq:defn-beta1-beta3}), we can write
\begin{align}
	\bm{\beta}_{1,t}(s,a) & = \gamma \sum_{i=1}^{t}\prod_{j=i+1}^{t}\big(1-\bm{\Lambda}_{j}(s,a)\big)\bm{\Lambda}_{i}(s,a)\big(\bm{P}_{i}(s,a)-\bm{P}(s,a)\big)\bm{V}^{\star}.
	\label{eq:expression-beta1-sa}
\end{align}

As it turns out, it is convenient to study this expression by
defining
\begin{equation}\label{eq:defn-tk-sa}
	t_{k}(s,a):=\text{the time stamp when }\text{the trajectory visits }(s,a)\text{ for the }k\text{-th time}
\end{equation}
and
\begin{equation}
K_{t}(s,a):=\max\left\{ k\mid t_{k}(s,a) \leq t\right\} ,\label{eq:defn-Kt}
\end{equation}
namely, the total number of times --- during the first $t$ iterations
--- that the sample trajectory visits $(s,a)$. With
these in place, the special form of $\bm{\Lambda}_{j}$ (cf.~\eqref{eq:defn-Lambda-t}) allows us
to rewrite (\ref{eq:expression-beta1-sa}) as
\begin{align}
	\bm{\beta}_{1,t}(s,a) & = \gamma \sum_{k=1}^{K_{t}(s,a)}(1-\eta)^{K_{t}(s,a)-k}\eta\big(\bm{P}_{t_{k}+1}(s,a)-\bm{P}(s,a)\big)\bm{V}^{\star}.\label{eq:expression-beta1-sa-1}
\end{align}
where we suppress the dependency on $(s,a)$ and write $t_k:=t_k(s,a)$
%$K_t: = K_t(s,a)$ 
to streamline notation. 
The main step thus boils down to controlling (\ref{eq:expression-beta1-sa-1}). 

Towards this, we claim that: 
%there exists some universal constant $c>0$ such that 
with probability at least $1-\delta$,
\begin{align}
	& \Bigg|\sum_{k=1}^{K}(1-\eta)^{K-k}\eta\big(\bm{P}_{t_{k}+1}(s,a)-\bm{P}(s,a)\big)\bm{V}^{\star}\Bigg| 
	 \leq  \sqrt{\eta  \log\Big(\frac{|\mathcal{S}||\mathcal{A}|T}{\delta}\Big)} \|\bm{V}^{\star}\|_{\infty} 
	\label{eq:claim-Bernstein-fixed-K}
\end{align}
holds simultaneously for all $(s,a)\in\mathcal{S}\times\mathcal{A}$ and
all $1\leq K\leq T$, provided that $0< \eta  \log\big(\frac{|\mathcal{S}||\mathcal{A}|T}{\delta}\big)<1$.
% where $c_{0}>0$ is some universal constant.
Recognizing the trivial bound $K_{t}(s,a)\leq t\leq T$ (by construction (\ref{eq:defn-Kt})) and substituting the claimed bound \eqref{eq:claim-Bernstein-fixed-K} into the expression \eqref{eq:expression-beta1-sa-1}, we arrive at
\begin{align}
	\forall (s,a)\in \cS\times \cA: \quad 
	| \bm{\beta}_{1,t}(s,a) | \leq \gamma  \sqrt{\eta  \log\Big(\frac{|\mathcal{S}||\mathcal{A}|T}{\delta}\Big)} \|\bm{V}^{\star}\|_{\infty}
	\leq  \sqrt{\eta  \log\Big(\frac{|\mathcal{S}||\mathcal{A}|T}{\delta}\Big)} \|\bm{V}^{\star}\|_{\infty}, 
\end{align}
thus concluding the proof of this lemma. It remains to validate the inequality (\ref{eq:claim-Bernstein-fixed-K}). 

\begin{proof}[Proof of the inequality (\ref{eq:claim-Bernstein-fixed-K})]
%Before proceeding, we introduce some additional notation. Let $\blue{\Var_{\bP}(\bm{V}^{\star})} \in \real^{|\cS||\cA|}$ be a vector whose $(s,a)$-th entry is given by the variance of $\bm{V}^{\star}$ w.r.t.~the transition probability $P_{s,a}(\cdot)$ from state $s$ when action $a$ is taken, namely, 
%\begin{align}
%	\forall (s,a)\in \cS \times \cA, \qquad \big[\Var_{\bP}(\bm{V}^{\star})\big]_{(s,a)}: = \sum_{s'\in \cS} P_{s,a}(s' ) \big( V^{\star}(s') \big)^2 
%	-  \Big( \sum_{s'\in \cS}P_{s,a}(s') V^{\star}(s') \Big)^2.
%	\label{defn:Var-P-V}
%\end{align}

	We first make the observation that: for any {\em fixed} integer $K>0$, the
following vectors
\[
\left\{ \bm{P}_{t_{k}+1}(s,a)\mid1\leq k\leq K\right\} 
\]
	are identically and independently distributed.\footnote{The Markov chain w.r.t.~the sample trajectory should be viewed as being infinitely long, although we only get to observe its first $T$ samples. The random variables $\{t_{k}\}$ are, in truth, independent of the choice of $T$.} To justify this observation,
let us denote by $\mathbb{P}_{s,a}(\cdot)$ the transition probability
from state $s$ when action $a$ is taken. For any $i_{1},\cdots,i_{K}\in\mathcal{S}$, one obtains
\begin{align*}
 & \mathbb{P}\left\{ s_{t_{k}+1}=i_{k}\text{ }(\forall1\leq k\leq K)\right\} =\mathbb{P}\left\{ s_{t_{k}+1}=i_{k}\ (\forall1\leq k\leq K-1)\ \text{and }s_{t_{K}+1}=i_{K}\right\} \\
 & \qquad=\sum_{m>0}\mathbb{P}\left\{ s_{t_{k}+1}=i_{k}\ (\forall1\leq k\leq K-1)\ \text{and }t_{K}=m\text{ and }s_{m+1}=i_{K}\right\} \\
 & \qquad\overset{(\mathrm{i})}{=}\sum_{m>0}\mathbb{P}\left\{ s_{t_{k}+1}=i_{k}\ (\forall1\leq k\leq K-1)\ \text{and }t_{K}=m\right\} \mathbb{P}\left\{ s_{m+1}=i_{K}\mid s_{m}=s,a_{m}=a\right\} \\
 & \qquad=\mathbb{P}_{s,a}(i_{K})\sum_{m>0}\mathbb{P}\left\{ s_{t_{k}+1}=i_{k}\ (\forall1\leq k\leq K-1)\ \text{and }t_{K}=m\right\} \\
 & \qquad=\mathbb{P}_{s,a}(i_{K})\mathbb{P}\left\{ s_{t_{k}+1}=i_{k}\ (\forall1\leq k\leq K-1)\right\} ,
\end{align*}
where (i) holds true from the Markov property as well as the fact
that $t_{K}$ is an iteration in which the trajectory visits state
$s$ and takes action $a$. Invoking the above identity recursively,
we arrive at
\begin{align}
	\label{eq:independence-Markov-chain}
\mathbb{P}\left\{ s_{t_{k}+1}=i_{k}\text{ }(\forall1\leq k\leq K)\right\}  & =\prod_{j=1}^{K}\mathbb{P}_{s,a}(i_{j}),
\end{align}
meaning that the state transitions happening at times $\{t_{1},\cdots,t_{K}\}$
are independent, each following the distribution $\mathbb{P}_{s,a}(\cdot)$.
This clearly demonstrates the independence of $\left\{ \bm{P}_{t_{k}+1}(s,a)\mid1\leq k\leq K\right\} $. 

%\blue{
With the above observation in mind, we resort to the Hoeffding inequality
to bound the quantity of interest (which has zero mean). To begin
with, notice the facts that for all $k \ge 1$,
\begin{equation}
0 \le \bm{P}_{t_{k}+1}(s,a)\bm{V}^{\star} \le \|\bm{V}^{\star}\|_{\infty},\qquad\text{and}\qquad 0 \le \bm{P}(s,a)\bm{V}^{\star} \le \|\bm{V}^{\star}\|_{\infty},
\end{equation}
which gives 
\begin{align*}
\left|(1-\eta)^{K-k}\eta\big(\bm{P}_{t_{k}+1}(s,a)-\bm{P}(s,a)\big)\bm{V}^{\star}\right| \le (1-\eta)^{K-k}\eta\|\bm{V}^{\star}\|_{\infty}.
\end{align*}
As a consequence, invoking the Hoeffding inequality \citep{boucheron2013concentration} implies that
\begin{align}
	& \Bigg|\sum_{k=1}^{K}(1-\eta)^{K-k}\eta\big(\bm{P}_{t_{k}}(s,a)-\bm{P}(s,a)\big)\bm{V}^{\star}\Bigg| 
	 \leq \sqrt{ \frac{1}{2} \sum_{k=1}^{K} \Big( (1-\eta)^{K-k}\eta \|\bm{V}^{\star}\|_{\infty} \Big) ^2\log\Big(\frac{2|\mathcal{S}||\mathcal{A}|T}{\delta}\Big)} \nonumber\\
	& \qquad \leq  \sqrt{  \eta  \log\Big(\frac{|\mathcal{S}||\mathcal{A}|T}{\delta}\Big)} \, \|\bm{V}^{\star}\|_{\infty}
	\label{eq:Bernstein-s-a-k}
\end{align}
with probability exceeding $1-\frac{\delta}{|\mathcal{S}||\mathcal{A}|T}$, where the last line holds since
\begin{align*}
	\sum_{k=1}^{K} \Big( (1-\eta)^{K-k}\eta \Big)^{2} \leq\eta^{2}\sum_{j=0}^{\infty}(1-\eta)^{j}=\frac{\eta^{2}}{1-(1-\eta)} =\eta.
\end{align*}
Taking the union bound over all $(s,a)\in\mathcal{S}\times\mathcal{A}$
and all $1\leq K\leq T$ then reveals that: with probability
at least $1-\delta$, the inequality (\ref{eq:Bernstein-s-a-k}) holds
simultaneously over all $(s,a)\in\mathcal{S}\times\mathcal{A}$ and
all $1\leq K\leq T$. This concludes the proof. \end{proof}

\subsection{Proof of Lemma~\ref{lemma:control-beta3} and Lemma~\ref{lemma:control-beta3-cover-time}}
\label{sec:proof-lemma-control-beta3}

\paragraph{Proof of Lemma~\ref{lemma:control-beta3}.}
Let $\bm{\beta}_{3,t}= \prod_{j=1}^{t}\big(\bm{I}-\bm{\Lambda}_{j}\big)\bm{\Delta}_{0}$. 
Denote by $\bm{\beta}_{3,t}(s,a)$ (resp.~$\bm{\Delta}_{0}(s,a)$)
the $(s,a)$-th entry of $\bm{\beta}_{3,t}$ (resp.~$\bm{\Delta}_{0}$).
From the definition of $\bm{\beta}_{3,t}$, it is easily seen that
\begin{align}
	\label{eq:beta3-sa-bound}
	\left|\bm{\beta}_{3,t}(s,a)\right|=(1-\eta)^{K_t(s,a)}\big|\bm{\Delta}_{0}(s,a)\big|, 
	%\leq(1-\eta)^{N_{s,a}}\big|\bm{\Delta}_{0}(s,a)\big|,
\end{align}
where $K_t(s,a)$ denotes the number of times the
sample trajectory visits $(s,a)$ during the iterations $[1,t]$ (cf.~\eqref{eq:defn-Kt}). By virtue of Lemma~\ref{lemma:Bernstein-state-occupancy} and the union bound, one has, with probability at least $1-\delta$, that 
\begin{align}
	 K_t(s,a) \geq t\mumin / 2
\end{align}
simultaneously over all $(s,a)\in\mathcal{S}\times\mathcal{A}$ and all $t$ obeying $\frac{443\tau_{\mathsf{mix}}}{\mumin}\log\frac{4|\mathcal{S}||\mathcal{A}|T}{\delta}\leq t\leq T$. Substitution into the relation \eqref{eq:beta3-sa-bound} establishes that, with probability greater than $1-\delta$, 
\begin{align}
	 \left|\bm{\beta}_{3}(s,a)\right| \leq (1-\eta)^{\frac{1}{2}t\mumin}\big|\bm{\Delta}_{0}(s,a)\big| . 
\end{align}
holds uniformly over all $(s,a)\in\mathcal{S}\times\mathcal{A}$ and all $t$ obeying $\frac{443\tau_{\mathsf{mix}}}{\mumin}\log\frac{4|\mathcal{S}||\mathcal{A}|T}{\delta}\leq t\leq T$, as claimed.

\paragraph{Proof of Lemma~\ref{lemma:control-beta3-cover-time}.} The proof of this lemma is essentially the same as that of Lemma~\ref{lemma:control-beta3}, except that we use instead the following lower bound on $K_t(s,a)$ (which is an immediate consequence of Lemma~\ref{lem:connection-cover-time})
\begin{align}
	K_t(s,a) \geq \Big\lfloor \frac{t}{\tcoverall} \Big\rfloor \geq \frac{t}{2\tcoverall}
\end{align}
for all $t>\tcoverall$. Therefore, replacing $t\mumin$ with $t/\tcoverall$ in the above analysis, we establish Lemma~\ref{lemma:control-beta3-cover-time}.

\subsection{Proof of Lemma~\ref{lemma:middle-term}} 
\label{Sec:middle-term}

We prove this fact via an inductive argument. 
The base case with $t=0$ is a consequence of the crude bound \eqref{eqn:crude}. Now, assume that the claim holds for all iterations up to $t-1$, and we would like to justify it for the $t$-th iteration as well. 
Towards this, define 
\begin{align}
	\label{defn:ht}
	h(t) :=
	\begin{cases}
	\|\bm{\Delta}_0\|_{\infty}, & \text{if }t\leq \tth,\\
	% (1-\eta)^{\frac{1}{2}t\mumin}\|\bm{\Delta}_0\|_{\infty},\quad & \text{if }\tcover<t\leq\tth,\\
		(1-\gamma) \varepsilon, & \text{if }t>\tth.
	\end{cases}  
\end{align}
Recall that $(1-\eta)^{\frac{1}{2}t\mumin}\leq (1-\gamma)\varepsilon$ for any $t\geq \tth$. 
Therefore, combining the inequality~\eqref{eqn:recursion} with the induction hypotheses indicates that
\begin{align*}
	|\bDel_t| & \leq \gamma \sum_{i=1}^t \prod_{j=i+1}^{t} (\bm{I}-\bLam_j)\bLam_i \one
	\cdot \left(\frac{\tau_1 \|\bm{V}^{\star}\|_{\infty} }{1-\gamma} + {u_{i-1}} + \varepsilon\right) + \tau_1  \|\bm{V}^{\star}\|_{\infty} \one + h(t) \one\\
	& 
	= \gamma \sum_{i=1}^t \prod_{j=i+1}^{t} (\bm{I}-\bLam_j)\bLam_i \one u_{i-1}
	+
	\gamma \sum_{i=1}^t \prod_{j=i+1}^{t} (\bm{I}-\bLam_j)\bLam_i \one \left(\frac{\tau_1 \|\bm{V}^{\star}\|_{\infty} }{1-\gamma} + \varepsilon\right) + \tau_1 \|\bm{V}^{\star}\|_{\infty} \one + h(t)  \one. 
\end{align*}
Taking this together with the inequality~\eqref{eqn:cat} and rearranging terms, we obtain
\begin{align}
	|\bDel_t| 
	& \leq   \gamma \sum_{i=1}^t \prod_{j=i+1}^{t} (\bm{I}-\bLam_j)\bLam_i \one u_{i-1}
	 + \frac{\gamma \tau_1 \|\bm{V}^{\star}\|_{\infty}}{1-\gamma}\one 
	+ \gamma \varepsilon \one + \tau_1 \|\bm{V}^{\star}\|_{\infty} \one + h(t) \one   \nonumber \\
	& = \frac{\tau_1 \|\bm{V}^{\star}\|_{\infty}}{1-\gamma}\one +  \gamma \varepsilon \one + \gamma \sum_{i=1}^t \prod_{j=i+1}^{t} (\bm{I}-\bLam_j)\bLam_i \one u_{i-1}
	 + h(t) \one \nonumber \\
	& = \frac{\tau_1 \|\bm{V}^{\star}\|_{\infty}}{1-\gamma}\one +  \gamma \varepsilon \one + \bm{v}_t + (1-\gamma) \varepsilon \ind\{ t>\tth \}   \one \nonumber\\
	& \leq \frac{\tau_1 \|\bm{V}^{\star}\|_{\infty}}{1-\gamma}\one +  \varepsilon \one + \bm{v}_t , 
\end{align}
where we have used the definition of $\bm{v}_t$ in \eqref{eq:defn-ut-vt}. 
This taken collectively with the definition $u_t = \|\bm{v}_t\|_{\infty}$ establishes that
\[
	\|\bDel_t\|_{\infty} \leq \frac{\tau_1 \|\bm{V}^{\star}\|_{\infty}}{1-\gamma} +  \varepsilon  + u_t
\]
as claimed. This concludes the proof.

%Consider the case $t < t_{0}$ where $u_{t}$ is defined as in expression~\eqref{eq:defn-ut-vt}. 
% We thus reach  
% \begin{align*}
% 	|\bDel_t| &\leq \frac{\tau_1}{1-\gamma}\one + \bm v_t + \gamma \varepsilon \one,
% \end{align*}
% which clearly leads to inequality~\eqref{eqn:bear} for $i = t$. 
% On the other hand, when $u_{t}$ is defined as in expression~\eqref{eqn:tlarge}, one has 
%\begin{align*}
% 	|\bDel_t| &\leq \frac{\tau_1}{1-\gamma}\one + \bm v_t + \gamma \varepsilon \one + 
% 	(1-\eta)^{\lfloor \frac{t}{\tcover} \rfloor} \linf{\bDel_0}\one\\
% %
% 	& \leq \frac{\tau_1}{1-\gamma}\one + \bm v_t + \varepsilon \one,
% \end{align*}
% which again leads to inequality~\eqref{eqn:bear} for $i = t$.
% Here, with $t\geq t_{0}$, we make use of the observation that $(1-\eta)^{\lfloor \frac{t}{\tcover} \rfloor} \linf{\bDel_0} \leq (1-\gamma)\varepsilon$ where we remind the readers that $\linf{\bDel_0} \leq \frac{1}{1 - \gamma}.$ 
% Putting things together, we finish the induction step for $i = t$.

\subsection{Proof of Lemma~\ref{lemma:daban}}
\label{Sec:daban}

We shall prove this result by induction over the index $k$. 
To start with, consider the base case where $k=0$ and $ t < \tth +\tcover$. By definition, it is straightforward to see that $u_{0}\leq \| \bDel_0 \|_{\infty}/(1-\gamma) = w_0$. In fact, repeating our argument for the crude bound (see Section~\ref{sec:recursion}) immediately reveals that
\begin{align}
	\forall t\geq 0: \qquad u_t \leq \frac{ \| \bDel_0 \|_{\infty} }{1-\gamma} = w_0 ,
\end{align}
thus indicating that the inequality~\eqref{eqn:lele} holds for the base case. 
In what follows, we assume that the inequality~\eqref{eqn:lele} holds up to $k-1$, and would like to extend it to the case with all $t$ obeying $\big\lfloor\frac{t-\tth}{\tcover}\big\rfloor = k$.

%Let us focus on the case when $t = \tth + k \tcover$; the case with $t = \tth + k \tcover + j$ ($1\leq j< \tcover$) follows from an analogous argument and is omitted for brevity. 
%Recall that in the expression~\eqref{eq:defn-ut-vt}, we define $u_{\tth + k \tcover} = \linf{\bm v_{\tth + k \tcover}}$.

Consider any $0\leq j< \tcover$. In view of the definition of $\bm{v}_t$ (cf.~\eqref{eq:defn-ut-vt})  as well as our induction hypotheses, one can arrange terms to derive
%
%\begin{align}
%\label{eqn:into-epoch}
%\notag	\bm v_{\tth + k \tcover} 
%	&= \gamma \sum_{i=1}^{\tth + k \tcover} \prod_{j=i+1}^{{\tth + k \tcover}} (\Ind-\bLam_j)\bLam_i \one u_{i-1} \\
%%
%\notag	&= \gamma 
%	\sum_{s=0}^{k-1} \Bigg\{\sum_{i:\,  \max\big\{ \lfloor \frac{i - 1 - \tth}{\tcover} \rfloor, 0 \big\}  = s} \prod_{j=i+1}^{{\tth + k \tcover}} (\Ind-\bLam_j)\bLam_i \one  u_{i-1}\Bigg\} 		
%	\\
%%
%	&\leq \gamma 
%	\sum_{s=0}^{k-1} \Bigg\{\sum_{i:\,  \max\big\{ \lfloor \frac{i - 1 - \tth}{\tcover} \rfloor, 0 \big\}  = s} \prod_{j=i+1}^{{\tth + k \tcover}} (\Ind-\bLam_j)\bLam_i \one  \Bigg\} w_{s}, 
%\end{align}
%
\begin{align}
\label{eqn:into-epoch}
\notag	\bm v_{\tth + k \tcover + j} 
	&= \gamma \sum_{i=1}^{\tth + k \tcover + j} \, \prod_{n=i+1}^{{\tth + k \tcover + j}} (\Ind-\bLam_n)\bLam_i \one u_{i-1} \\
\notag	&= \gamma 
	\sum_{s=0}^{k-1} \Bigg\{\sum_{i:\,  \max\big\{ \lfloor \frac{i - j - 1 - \tth}{\tcover} \rfloor, 0 \big\}  = s} \prod_{n=i+1}^{{\tth + k \tcover + j}} (\Ind-\bLam_n)\bLam_i \one  u_{i-1}\Bigg\} 		
	\\
	&\leq \gamma 
	\sum_{s=0}^{k-1} \Bigg\{\sum_{i:\,  \max\big\{ \lfloor \frac{i - j - 1 - \tth}{\tcover} \rfloor, 0 \big\}  = s} \prod_{n=i+1}^{{\tth + k \tcover + j}} (\Ind-\bLam_n)\bLam_i \one  \Bigg\} w_{s}, 
\end{align}
where the last inequality follows from our induction hypotheses, the non-negativity of 
$(\Ind-\bLam_j)\bLam_i \one$, and the fact that $w_{s}$ is non-increasing.

Given any state-action pair $(s,a) \in \cirS\times \A$, let us look at the $(s,a)$-th entry 
of $\bm v_{\tth + k \tcover + j}$ --- denoted by $\bm{v}_{\tth + k \tcover + j}(s,a)$, towards which it is convenient to pause and introduce some notation.
Recall that $N_{i}^n(s,a)$ has been used to denote the number of visits to the state-action pair $(s,a)$ between iteration $i$ and iteration $n$ (including $i$ and $n$). To help study the behavior  in each timeframe, we introduce the following quantities
\begin{align}
	L_{h}^{k-1} \defn N_{i}^{n}(s,a) \qquad \text{with } 
	 i = \tth + h \tcover + j + 1, ~n = \tth + k \tcover + j
\end{align}
for every $h \leq k-1$. 
%in words, $L_{h}^{k-1}$ stands for the total number of visits to $(s,a)$ between the $h$-th frame and the $(k-1)$-th frame \blue{with a bias $j$}. 
Lemma~\ref{lemma:Bernstein-state-occupancy} tells us that, with probability at least $1-2\delta$,  
\begin{align}
\label{eqn:amazing-times}
	L_{h}^{k-1} \geq (k-h) \muepo \qquad  \text{with }\muepo = \frac{1}{2}\mumin\tcover,
\end{align}
which holds uniformly over all state-action pairs $(s,a)$. 
 Armed with this set of notation, 
 it is straightforward to use the expression~\eqref{eqn:into-epoch} to verify that 
\begin{align}
 \bm{v}_{\tth + k \tcover + j}(s,a)  
	& \leq \gamma\sum_{h=0}^{k-1}\eta\left\{ (1-\eta)^{L_{h}^{k-1}-1}+(1-\eta)^{L_{h}^{k-1}-2}+\cdots+(1-\eta)^{L_{h+1}^{k-1}}\right\} w_{h} \nonumber \\
 & =\gamma\sum_{h=0}^{k-1}\left((1-\eta)^{L_{h+1}^{k-1}}-(1-\eta)^{L_{h}^{k-1}}\right)w_{h} =:\gamma\sum_{h=0}^{k-1}\left(\alpha_{h+1}-\alpha_{h}\right)w_{h},  \label{eqn:allegro}
\end{align}
where we denote $\alpha_{h} \defn (1-\eta)^{L_{h}^{k-1}}$ for any $h\leq k-1$ and 
$\alpha_{k} \defn 1.$

% In the case when $k=1$, invoking the expression~\eqref{eqn:allegro} gives
% %
% \begin{align}
% v_{\tth + k \tcover}(s,a) 
% 	\leq \gamma \alpha_{1} w_{1} = \gamma \frac{\|\bDel_0\|_{\infty}}{1-\gamma} \leq (1-\rho)\frac{\linf{\bDel_0}}{1-\gamma},
% \end{align}
% %
% where the last line follows since $\gamma \leq 1 - \rho.$ Whence, we see that $v_{\tth + k \tcover}(s,a) \leq w_{k}$ holds true for $k=1.$
% We now move on to the case with $k\geq 2.$ 
A little algebra further leads to
\begin{align}
\label{eqn:adagio}
	\gamma\sum_{h=0}^{k-1}\left(\alpha_{h+1}-\alpha_{h}\right)w_{h} & = \gamma( \alpha_k w_{k-1} - \alpha_{0}w_{0} )+ \gamma 
	\sum_{h=1}^{k-1}\alpha_{h}\left(w_{h-1}-w_{h}\right).
\end{align}
Thus, in order to control the quantity $\bm{v}_{\tth + k \tcover + j}(s,a)$, it suffices to control the right-hand side of 
\eqref{eqn:adagio}, for which
we start by bounding the last term. Plugging in the definitions of $w_h$ and $\alpha_h$ yields 
\begin{align*}
	\frac{1-\gamma}{\|\bDel_0\|_{\infty}} \sum_{h=1}^{k-1}\alpha_{h}\left(w_{h-1}-w_{h}\right) 
& =\sum_{h=1}^{k-1}(1-\eta)^{L_{h}^{k-1}}(1-\rho)^{h-1}\rho
\leq \rho\sum_{h=1}^{k-1}(1-\eta)^{\left(k-h\right)\muepo}(1-\rho)^{h-1},
\end{align*}
where the last inequality results from the fact~\eqref{eqn:amazing-times}.
Additionally, direct calculation yields
\begin{align}
\notag \rho\sum_{h=1}^{k-1}(1-\eta)^{\left(k-h\right)\muepo}(1-\rho)^{h-1}
	& =\rho(1-\eta)^{(k-1)\muepo}\sum_{h=1}^{k-1}\Big(\frac{1-\rho}{(1-\eta)^{\muepo}}\Big)^{h-1} \nonumber\\
  & =\rho(1-\eta)^{(k-1)\muepo}\frac{1-\big(\frac{1-\rho}{(1-\eta)^{\muepo}}\big)^{k-1}}{1-\frac{1-\rho}{(1-\eta)^{\muepo}}}\nonumber\\
 % & =\rho(1-\eta)^{(k-1)\muepo}\frac{1-\Big(\frac{1-\rho}{(1-\eta)^{\muepo}}\Big)^{k-2}}{\frac{(1-\eta)^{\muepo}-(1-\rho)}{(1-\eta)^{\muepo}}}\\
\notag & =\rho(1-\eta)^{\muepo}\frac{(1-\rho)^{k-1}-(1-\eta)^{(k-1)\muepo}}{(1-\rho)-(1-\eta)^{\muepo}} \nonumber\\
 & \leq\rho(1-\eta)^{\muepo}\frac{(1-\rho)^{k-1}}{(1-\rho)-(1-\eta)^{\muepo}}, \label{eqn:schertzo}
\end{align}
where the last inequality makes use of the fact that 
\begin{align}
(1-\rho)-(1-\eta)^{\muepo} & =1-(1-\gamma)(1-(1-\eta)^{\muepo})-(1-\eta)^{\muepo} \nonumber\\
% =\gamma+(1-\eta)^{\muepo}-\gamma(1-\eta)^{\muepo}-(1-\eta)^{\muepo}\\
& =\gamma\left\{ 1-(1-\eta)^{\muepo}\right\} =\frac{\gamma}{1-\gamma}\rho \geq 0 .
	\label{eq:rho-eta-comparison}
\end{align}
Combining the inequalities~\eqref{eqn:allegro}, \eqref{eqn:adagio} and \eqref{eqn:schertzo} and using the fact $\alpha_0w_0\geq 0$ give 
\begin{align}
	\bm{v}_{\tth + k \tcover + j}(s,a) & \leq \gamma\sum_{h=1}^{k-1}\alpha_{h}\left(w_{h-1}-w_{h}\right)+ 
\gamma \alpha_k w_{k-1} \nonumber\\
 & \leq
\frac{\linf{\bDel_0}}{1-\gamma} \left\{ 
 \gamma\rho(1-\eta)^{\muepo}\frac{(1-\rho)^{k-1}}{(1-\rho)-(1-\eta)^{\muepo}}
 % + \gamma(1-\eta)^{\muepo}\Big(1-\rho\Big)^{k-2} 
	+ \gamma(1-\rho)^{k-1} \right\} . \label{eq:vth-bound1}
 % & =
 % \frac{\linf{\bDel_0}}{1-\gamma} \left\{ 
 % \gamma(1-\eta)^{\muepo}\left\{ \frac{\rho(1-\eta)^{\muepo}}{(1-\rho)-(1-\eta)^{\muepo}}+1\right\} \Big(1-\rho\Big)^{k-2} + \gamma(1-\rho)^{k-1}
 % \right\}.
\end{align}

We are now ready to justify that $\bm{v}_{\tth + k \tcover + j}(s,a) \leq w_{k}$. 
Note that the observation \eqref{eq:rho-eta-comparison} implies
\begin{align*}
\gamma\frac{\rho(1-\eta)^{\muepo}}{(1-\rho)-(1-\eta)^{\muepo}}=\gamma\frac{\rho(1-\eta)^{\muepo}}{\frac{\gamma}{1-\gamma}\rho}=(1-\gamma)(1-\eta)^{\muepo}.	
\end{align*}
% And hence
% \begin{align}
% \gamma(1-\eta)^{\muepo}\left\{ \frac{\rho(1-\eta)^{\muepo}}{(1-\rho)-(1-\eta)^{\muepo}}+1\right\}  
% & =(1-\gamma)(1-\eta)^{2\muepo}+\gamma(1-\eta)^{\muepo}
% % =(1-\eta)^{\muepo}\left\{ (1-\gamma)(1-\eta)^{\muepo}+\gamma\right\} \\
% % & =(1-\eta)^{\muepo}\left\{ 1-\gamma-\rho+\gamma\right\} \\
% =(1-\eta)^{\muepo}(1-\rho).
% \end{align}
This combined with the bound \eqref{eq:vth-bound1} yields   
\begin{align}
\label{eqn:doupi}
	\notag	\bm{v}_{\tth + k \tcover + j}(s,a) &\leq \frac{\linf{\bDel_0}}{1-\gamma} \big\{ (1-\gamma)(1-\eta)^{\muepo} (1-\rho)^{k-1} +  \gamma(1-\rho)^{k-1} \big\}\\
	\notag &\leq \frac{\linf{\bDel_0}}{1-\gamma} \big(\gamma + (1-\gamma) (1-\eta)^{\muepo}\big) (1-\rho)^{k-1} \\
	&= (1-\rho)^{k} \frac{\linf{\bDel_0}}{1-\gamma} = w_k, 
\end{align}
where  the last line follows from the definition of $\rho$ (cf.~\eqref{eqn:rho}).
Since the above inequality holds for all state-action pair $(s,a)$, we conclude that
\begin{align}
	u_{\tth + k \tcover + j} = \linf{\bm v_{\tth + k \tcover + j}}  \leq w_k.
\end{align}
%

%We have thus finished the proof for the case when $t = \tth + k \tcover$. 
%As mentioned before, the case with $t = \tth + k \tcover + j$ ($j=1,\ldots, \tcover-1$) can be justified using the same argument. 
As a consequence, we have established the inequality~\eqref{eqn:lele} for all $t$ obeying  $\big\lfloor\frac{t-\tth}{\tcover}\big\rfloor= k$, which together with the induction argument completes the proof of this lemma.

\subsection{Proof of Lemma \ref{lemma:bound-h0t}}
\label{sec:proof-lemma-bound-h0t}

Recalling that $\bm{0} \leq \sum_{i=1}^t \prod_{j=i+1}^{t} ( \bm{I} -\bLam_j)\bLam_i \one \leq \one$ (cf.~\eqref{eqn:cat}), we obtain
\begin{align}
	\| \bm{h}_{0,t}\|_{\infty} 
	&\leq \gamma \Big\|\sum_{i=1}^{t}\prod_{j=i+1}^{t}\big(\bm{I}-\bm{\Lambda}_{j}\big)\bm{\Lambda}_{i}\Big\|_{1}  
	 \big\|(\widetilde{\bm{P}}-\bm{P}\big)\overline{\bm{V}}\big\|_{\infty}
	\leq \gamma \big\|(\widetilde{\bm{P}}-\bm{P}\big)\overline{\bm{V}}\big\|_{\infty}.
	\label{eq:H0t-simple-bound}
\end{align}
As a result, it remains  to upper bound $\big\|(\widetilde{\bm{P}}-\bm{P}\big)\overline{\bm{V}}\big\|_{\infty}$.

Suppose that $\widetilde{\bm{P}}$ is constructed
using $N$ consecutive sample transitions. Without loss of generality,
assume that these $N$ sample transitions are the transitions between
the following $N+1$ samples
\[
(s_{0},a_{0}),(s_{1},a_{1}),(s_{2},a_{2}),\cdots,(s_{N},a_{N}).
\]
Then the $(s,a)$-th row of $\widetilde{\bm{P}}$ --- denoted by
$\widetilde{\bm{P}}(s,a)$ --- is given by
\begin{align}
	\widetilde{\bm{P}}(s,a)=\frac{1}{K_{N}(s,a)}\sum_{i=0}^{N-1} \bm{P}_{i+1}(s,a) \overline{\bm{V}}\ind\{(s_{i},a_{i})=(s,a)\}=\frac{1}{K_{N}(s,a)}\sum_{i=1}^{K_{N}(s,a)} \bm{P}_{t_{i}+1}(s,a) \overline{\bm{V}}	, 
\end{align}
where $\bm{P}_i$ is defined in \eqref{eq:defn-Pt}, and $\bm{P}_i(s,a)$ denotes its $(s,a)$-th row. Here, $K_N(s,a)$ denotes the total number of visits to $(s,a)$ during the first $N$ time instances (cf. \eqref{eq:defn-Kt}), and  
$t_k:= t_k(s,a)$ denotes the time stamp when the trajectory visits ($s,a$) for the $k$-th time (cf. \eqref{eq:defn-tk-sa}).

In view of our derivation for \eqref{eq:independence-Markov-chain}, the state transitions happening at time $t_1,t_2,\cdots,t_k$ are independent for any given integer $k>0$. This together with the Hoeffding inequality implies that
\begin{equation}
	\mathbb{P}\left\{ \frac{1}{k} \left|\sum_{i=1}^{k} \big(\bm{P}_{t_{i}+1}(s,a)-\bm{P}(s,a) \big)\overline{\bm{V}}\right|\geq\tau \right\} \leq2\exp\left\{ -\frac{k\tau^{2}}{2\|\overline{\bm{V}}\|_{\infty}^{2}}\right\} .
	\label{eq:Hoeffding-Ptilde}
\end{equation}
Consequently, with probability at least $1-\frac{\delta}{|\mathcal{S}||\mathcal{A}|}$
one has
\[
	\left| \frac{1}{k} \sum_{i=1}^{k} \big(\bm{P}_{t_{i}+1}(s,a)-\bm{P}(s,a)\big)\overline{\bm{V}}\right|\leq\sqrt{\frac{2\log\big( \frac{2N|\mathcal{S}||\mathcal{A}|}{\delta} \big)}{k}}\big\|\overline{\bm{V}}\big\|_{\infty},\qquad1\leq k\leq N.
\]
Recognizing the simple bound $K_{N}(s,a) \leq N$, the above inequality holds for each state-action pair $(s,a)$ when $k$ is replaced by $K_{N}(s,a)$. Conditioning on these $K_{N}(s,a)$, applying the union bound over all $(s,a)\in \cS\times \cA$, we obtain 
\begin{align}
	\big\|(\widetilde{\bm{P}}-\bm{P})\overline{\bm{V}}\big\|_{\infty} 
	\leq
	\max_{(s,a)\in \cS\times \cA} \sqrt{\frac{2\log\big( \frac{2N|\mathcal{S}||\mathcal{A}|}{\delta} \big)}{K_N(s,a)}}\big\|\overline{\bm{V}}\big\|_{\infty}
	\label{eq:Hoeffding-Ptilde2}
\end{align}
with probability at least $1-\delta$.

	In addition, for any $N \geq \tcov$, Lemma~\ref{lemma:Bernstein-state-occupancy} guarantees that with probability $1-2\delta$, each state-action pair $(s,a)$ is visited at least $N\mumin/2$ times, namely, $K_N(s,a) \geq \frac{1}{2} N \mumin $ for all $(s,a)$. This combined with \eqref{eq:Hoeffding-Ptilde2} yields
\begin{align}
	\big\|(\widetilde{\bm{P}}-\bm{P})\overline{\bm{V}}\big\|_{\infty} 
	&\leq
	\sqrt{\frac{4\log\big( \frac{2N|\mathcal{S}||\mathcal{A}|}{\delta} \big)}{N\mumin}}\big\|\overline{\bm{V}}\big\|_{\infty} \nonumber \\
	& \leq \sqrt{\frac{4\log\big( \frac{2N|\mathcal{S}||\mathcal{A}|}{\delta} \big)}{N\mumin}}\left( \big\|\overline{\bm{V}} - \bm{V}^{\star} \big\|_{\infty}  + \big\|  \bm{V}^{\star} \big\|_{\infty}  \right) \nonumber \\
& \leq \sqrt{\frac{4\log\big( \frac{2N|\mathcal{S}||\mathcal{A}|}{\delta} \big)}{N\mumin}}  \big\|\overline{\bm{V}} - \bm{V}^{\star} \big\|_{\infty}  + \frac{1}{1-\gamma}\sqrt{\frac{4\log\big( \frac{2N|\mathcal{S}||\mathcal{A}|}{\delta} \big)}{N\mumin}}\label{eq:Hoeffding-Ptilde2}
\end{align}
with probability at least $1-3\delta$, where the second inequality follows from the triangle inequality, and the last inequality follows from $\big\|  \bm{V}^{\star} \big\|_{\infty}\leq \frac{1}{1-\gamma}$. Putting this together with \eqref{eq:H0t-simple-bound} concludes the proof.

%	event 
% \begin{align*}
% 	\mathcal{B} \defn \{\forall (s,a) \in \cS\times \cA: ~N(s,a) \geq N\mumin/2 \},
% \end{align*}
% happens with probability at least $1-2\delta.$
% Conditional on the number of visits to each state being $N(s,a)$, Hoeffding's inequality guarantees that for any fixed $\delta \in (0,1)$, one has with probability at least $1 - \delta$,
% \begin{align*}
% 	\Big\|(\widetilde{\bm{P}}-\bm{P}\big)\overline{\bm{V}}\Big\|_{\infty} \leq 
% 	\max_{(s,a) \in \cS \times \cA} ~\sqrt{\frac{\log \frac{|\cS||\cA|}{\delta}}{N(s,a)}} \linf{\Vbar}.
% \end{align*}
% Moreover, conditional on event $\mathcal{B}$ which happens with probability $1-2\delta$, we can further ensure that 
% \begin{align*}
% 	\Big\|(\widetilde{\bm{P}}-\bm{P}\big)\overline{\bm{V}}\Big\|_{\infty} \leq 
% 	\sqrt{\frac{2\log \frac{|\cS||\cA|}{\delta}}{N\mumin}} \linf{\Vbar}
% \end{align*}
% Combining the above arguments with $\linf{\Vbar}\leq 1/(1-\gamma)$ further gives 
% % \begin{align}
% 	\linf{\bm{h}_{0,t}} \leq \gamma \sqrt{\frac{\log \frac{|\cS||\cA|}{\delta}}{(1-\gamma)^2 N\mumin}}.
% \end{align}
%

\subsection{Proof of Lemma~\ref{lem:connection-cover-time}}
\label{sec:proof-lemma-connection-cover-time}

%\paragraph{Proof of Lemma~\ref{lem:connection-cover-time}.}

For notational convenience, set $t_{l}:=\tcovertime l$, and define
\[
\mathcal{H}_{l}:=\Big\{\exists(s,a)\in\mathcal{S\times\mathcal{A}} ~\text{that is not visited within }\big(t_{l},t_{l+1}\big]\,\Big\}
\]
for any integer $l\geq0$. In view of the definition of $\tcovertime$, 
we see that for any given $(s',a')\in\mathcal{S}\times\mathcal{A}$, 
\begin{equation}
\mathbb{P}\left\{ \mathcal{H}_{l}\mid(s_{t_{l}},a_{t_{l}})=(s',a')\right\} \leq\frac{1}{2} .\label{eq:Hl-prob-UB}
\end{equation}
Consequently,
for any integer $L>0$, one can invoke the Markovian property to obtain
\begin{align*}
 & \mathbb{P}\left\{ \mathcal{H}_{1}\cap\cdots\cap\mathcal{H}_{L}\right\} =\mathbb{P}\left\{ \mathcal{H}_{1}\cap\cdots\cap\mathcal{H}_{L-1}\right\} \mathbb{P}\left\{ \mathcal{H}_{L}\mid\mathcal{H}_{1}\cap\cdots\cap\mathcal{H}_{L-1}\right\} \\
 & \quad=\mathbb{P}\left\{ \mathcal{H}_{1}\cap\cdots\cap\mathcal{H}_{L-1}\right\} \sum_{s',a'}\mathbb{P}\left\{ \mathcal{H}_{L}\mid(s_{t_{l}},a_{t_{l}})=(s',a')\right\} \mathbb{P}\left\{ (s_{t_{l}},a_{t_{l}})=(s',a')\mid\mathcal{H}_{1}\cap\cdots\cap\mathcal{H}_{L-1}\right\} \\
 & \quad\leq\frac{1}{2}\mathbb{P}\left\{ \mathcal{H}_{1}\cap\cdots\cap\mathcal{H}_{L-1}\right\} \sum_{s',a'}\mathbb{P}\left\{ (s_{t_{l}},a_{t_{l}})=(s',a')\mid\mathcal{H}_{1}\cap\cdots\cap\mathcal{H}_{L-1}\right\} \\
 & \quad=\frac{1}{2}\mathbb{P}\left\{ \mathcal{H}_{1}\cap\cdots\cap\mathcal{H}_{L-1}\right\} ,
\end{align*}
where the inequality follows from (\ref{eq:Hl-prob-UB}). Repeating
this derivation recursively, we deduce that
\[
\mathbb{P}\left\{ \mathcal{H}_{1}\cap\cdots\cap\mathcal{H}_{L}\right\} \leq\frac{1}{2^{L}}.
\]
This tells us that
\begin{align*}
\mathbb{P}\left\{ \exists(s,a)\in\mathcal{S\times\mathcal{A}}\text{ that is not visited between }(0,\tcoverall]\right\}  & \leq\mathbb{P}\left\{ \mathcal{H}_{1}\cap\cdots\cap\mathcal{H}_{\log_{2}\frac{T}{\delta}}\right\} \leq\frac{1}{2^{\log_{2}\frac{T}{\delta}}}=\frac{\delta}{T},
\end{align*}
which in turn establishes the advertised result by applying the union bound.

%%%%%%%%%%%%%%%%%%%%%%%%%%%%%%%%%%%%%%%%%%%%%%%%%%%%%%%%%%%%%%%%%%%%%%%%%%%%%%%%%%%%%%%%%%%%%%
 
\bibliographystyle{apalike}
\bibliography{bibfileRL}

\end{document}